\documentclass[sigconf,nonacm]{acmart}

\usepackage{balance} %
\usepackage{graphicx} %
\usepackage{amsmath}
\usepackage{amsthm}
\usepackage{xcolor}
\usepackage{comment}
\usepackage{xargs}
\usepackage{tikz}
\usepackage{cancel}
\usepackage{caption}
\usepackage{bm}
\usepackage[capitalise]{cleveref}
\usepackage{nicefrac}
\usepackage{thm-restate}

\usepackage{algorithm}
\usepackage{algpseudocode}
\usepackage{todonotes}
\usepackage{enumitem}

\usepackage{tikz}
\usetikzlibrary{patterns,arrows,arrows.meta,calc,shapes,shadows,decorations.pathmorphing,decorations.pathreplacing,automata,shapes.multipart,positioning,shapes.geometric,fit,circuits,trees,shapes.gates.logic.US,fit,backgrounds,decorations.text,topaths}

\Crefname{equation}{Equation}{Equations}
\crefname{equation}{Equation}{Equations}
\Crefname{figure}{Figure}{Figures}
\crefname{figure}{Figure}{Figures}
\Crefname{remark}{Remark}{Remarks}
\crefname{remark}{Remark}{Remarks}
\Crefname{section}{Section}{Sections}
\crefname{section}{Section}{Sections}
\Crefname{subsection}{Section}{Sections}
\crefname{subsection}{Section}{Sections}
\Crefname{appendix}{Appendix}{Appendices}
\crefname{appendix}{Appendix}{Appendices}
\Crefname{example}{Example}{Examples}
\crefname{example}{Example}{Examples}
\Crefname{definition}{Definition}{Definitions}
\crefname{definition}{Definition}{Definitions}
\Crefname{lemma}{Lemma}{Lemmas}
\crefname{lemma}{Lemma}{Lemmas}
\Crefname{theorem}{Theorem}{Theorems}
\crefname{theorem}{Theorem}{Theorems}
\Crefname{proposition}{Proposition}{Propositions}
\crefname{proposition}{Proposition}{Propositions}

\makeatletter
\gdef\@copyrightpermission{
  \begin{minipage}{0.2\columnwidth}
   \href{https://creativecommons.org/licenses/by/4.0/}{\includegraphics[width=0.90\textwidth]{by}}
  \end{minipage}\hfill
  \begin{minipage}{0.8\columnwidth}
   \href{https://creativecommons.org/licenses/by/4.0/}{This work is licensed under a Creative Commons Attribution International 4.0 License.}
  \end{minipage}
  \vspace{5pt}
}
\makeatother

\copyrightyear{2027}
\acmYear{2027}
\acmDOI{}
\acmISBN{}

\title[On the Complexity Gap Between POMDPs and Multi-Environment POMDPs]{On the Complexity Gap Between POMDPs and Multi-Environment POMDPs}

\author{Eline M. Bovy}
\authornote{Shared first authorship, ordered alphabetically.}
\affiliation{
  \institution{Radboud University}
  \city{Nijmegen}
  \country{The Netherlands}}

\author{Caleb Probine}
\authornotemark[1]
\affiliation{
  \institution{The University of Texas at Austin}
  \city{Austin}
  \country{USA}}

\author{Marnix Suilen}
\affiliation{
  \institution{University of Antwerp
  }
  \city{Antwerp}
  \country{Belgium}}

\author{Ufuk Topcu}
\affiliation{
  \institution{The University of Texas at Austin}
  \city{Austin}
  \country{USA}}

\author{Nils Jansen}
\affiliation{
  \institution{Ruhr-University Bochum}
  \city{Bochum}
  \country{Germany}}
\affiliation{
  \institution{Radboud University}
  \city{Nijmegen}
  \country{The Netherlands}}

\begin{abstract}
\emph{Multi-environment} POMDPs (ME-POMDPs) generalize POMDPs by encoding a finite set of possible environments, capturing model uncertainty.
The finite-horizon robust reward maximization problem for a ME-POMDP asks for a policy that maximizes reward in the worst environment, and recent work has shown this problem to be $\PSPACE$-complete.
We expand the complexity landscape of ME-POMDPs by identifying which subclasses retain this $\PSPACE$-completeness.
In particular, we take \emph{subclasses} of POMDPs that are tractable in the single-environment case, and investigate the computational complexity of these classes when extended to the multi-environment setting.
We first prove the $\mathsf{NP}$-completeness of multi-environment \emph{blind MDPs} and \emph{deterministic POMDPs} using a certain closure-under-mixture condition.
We then prove $\mathsf{PSPACE}$-completeness for a range of subclasses that do not satisfy that closure condition, yielding new complexity results for multi-observation POMDPs. 
Finally, we show \emph{existential theory of the reals} ($\ETR$)-completeness for numerous ME-POMDP problems under time-dependent Markovian policies or finite-state controllers that are $\NP$-complete for standard POMDPs.
These results reveal a subtle relationship between POMDP and ME-POMDP complexity that depends heavily on the subclasses involved.

\end{abstract}


\newcommand{\BibTeX}{\rm B\kern-.05em{\sc i\kern-.025em b}\kern-.08em\TeX}

\newcommand{\RR}{\ensuremath{\mathbb{R}}}
\newcommand{\NN}{\ensuremath{\mathbb{N}}}

\newcommand{\EE}{\ensuremath{\mathbb{E}}}
\newcommand{\calM}{\mathcal{M}}

\newcommand{\calH}{\mathcal{H}}

\newcommand{\calI}{\mathcal{I}}
\newcommand{\calL}{\mathcal{L}}

\newcommand{\calC}{\mathcal{C}}
\newcommand{\calN}{\mathcal{N}}
\newcommand{\calA}{\mathcal{A}}
\newcommand{\calZ}{\mathcal{Z}}
\newcommand{\calS}{\mathcal{S}}
\newcommand{\calP}{\mathcal{P}}
\newcommand{\calU}{\mathcal{U}}

\newcommand{\supp}{\mathsf{supp}}

\newcommand{\polDet}{\Pi_{\mathsf{det}}}
\newcommand{\polDetObs}{\Pi_{\mathsf{det,obs}}}
\newcommand{\polMR}{\Pi_{\mathsf{MR}}}  %

\newcommand{\vecat}[2]{#1[#2]}
\newcommand{\dist}[1]{\Delta(#1)}
\newcommand{\unit}[1]{\mathbf{e}_{#1}}

\newcommand{\bit}{\mathsf{b}}

\newcommand{\ie}{\emph{i.e.}}

\newcommand{\wlogg}{\emph{w.l.o.g.}}

\newcommand{\NP}{\mathsf{NP}}
\newcommand{\PSPACE}{\mathsf{PSPACE}}
\newcommand{\Ptime}{\mathsf{P}}
\newcommand{\ETR}{\exists \RR}

\newtheorem{remark}{Remark}

\newtheorem{problem}{Problem}

\newenvironment{proofsketch}
  {\begin{proof}[Proof Sketch]}
  {\end{proof}}

\definecolor{color-deter-trns}{HTML}{FFC107}
\definecolor{color-deter-obs}{HTML}{E26BBA}

\definecolor{color-normal}{HTML}{0CA5F3}
\definecolor{color-deter}{HTML}{2F17C7}
\definecolor{color-memoryless}{HTML}{FF2727}

\begin{document}

\maketitle

\section{Introduction}
Partially observable Markov decision processes (POMDPs) are a canonical model for sequential decision-making problems under partial information~\cite{DBLP:journals/ai/KaelblingLC98}.
POMDPs capture problems in which a decision-maker must reason about an unknown state of the world, with applications ranging from robotics to healthcare~\cite{DBLP:journals/arcras/Kurniawati22,DBLP:journals/artmed/HauskrechtF00}.

POMDP planning algorithms, however, rely on the restrictive assumption that a single exact model of the environment exists. 
Multi-environment POMDPs (ME-POMDPs) remove this assumption and instead assume a finite set of possible environment models~\cite{DBLP:conf/nips/BovyPSTJ25}.
A ME-POMDP comprises a discrete set of POMDPs with shared state, observation, and action spaces.
The robust reward-maximization problem for ME-POMDPs then asks for a policy that maximizes expected reward in the \emph{worst-case} environment, chosen by an adversary.
Like POMDPs, ME-POMDPs admit value-iteration-based~\cite{DBLP:conf/nips/BovyPSTJ25} and gradient-ascent-based~\cite{DBLP:conf/ijcai/GaleslootA0J025} algorithms for reward maximization, as well as reinforcement learning algorithms for reach-reward objectives \cite{DBLP:conf/ifaamas/HudakGTKJC26}.
Meanwhile, in terms of complexity, finite-horizon reward maximization is  $\PSPACE$-complete in ME-POMDPs~\cite{DBLP:journals/corr/abs-2605-07537}, just as in standard POMDPs~\cite{DBLP:journals/mor/PapadimitriouT87}, while the discounted infinite-horizon setting is already undecidable for the latter~\cite{DBLP:journals/ai/MadaniHC03}.

Significant effort has been poured into finding more tractable subclasses of POMDPs. %
These classes rely on additional assumptions to reduce the computational complexity of finite-horizon reward maximization or to achieve decidability for many objectives, including infinite-horizon discounted rewards and logic objectives such as reachability and parity~\citep{asadi2026revealing, DBLP:conf/aaai/BellyFG00V25}.
For finite-horizon rewards on single-environment models, only complexity results for general POMDPs, fully observable MDPs, and blind MDPs are known~\cite{DBLP:journals/mor/PapadimitriouT87}.
In the multi-environment setting, results are only known for multi-environment MDPs~\cite{DBLP:conf/fsttcs/RaskinS14} and general ME-POMDPs~\cite{DBLP:journals/corr/abs-2605-07537}, as \Cref{fig:complexity_overview_before} illustrates, leaving a large gap between single-environment subclasses of POMDPs and their extension to multiple environments.

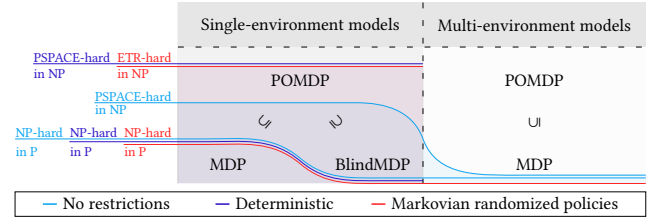
\begin{figure}[t]
    \centering
    \resizebox{\columnwidth}{!}{\begin{tikzpicture}[state/.append style={shape = circle}, >=stealth, bobbel/.style={minimum size=1mm,inner sep=0pt,fill=black,circle}]
    \fill[lightgray!20] (0.6,1.4) rectangle (5,-1.85);
    \fill[lightgray!5] (5,1.4) rectangle (9,-1.85);
    \draw[fill=color-deter,opacity=0.05] (0.6,-1.85) to[out=90,in=-90] (0.6,0.3) to (5,0.3) to[out=-90,in=90] (5,-1.85);
    \draw[fill=color-memoryless,opacity=0.05] (0.6,-1.85) to[out=90,in=-90] (0.6,0.25) to[out=0,in=180] (5,0.25) to[out=-90,in=90] (5,-1.85);
    \draw[fill=color-normal,opacity=0.05] (0.6,-1.85) to[out=90,in=-90] (0.6,-0.4) to[out=0,in=180] (3.8,-0.4) to[out=0,in=110] (5,-1.05) to[out=-70,in=180] (6.2,-1.7) to[out=0,in=180] (9,-1.7) to[out=-90,in=90] (9,-1.85);
    \node[] (pomdp1) at (2.8,0) {POMDP};
    \node[] (mdp1) at ($(pomdp1) + (-1.3,-1.5)$) {MDP};
    \node[] (blind_mdp1) at ($(pomdp1) + (1.3,-1.5)$) {BlindMDP};
    \draw[opacity=0] (mdp1) to node[sloped,opacity=1] {$\subseteq$} (pomdp1);
    \draw[opacity=0] (blind_mdp1) to node[sloped,opacity=1,xscale=-1] {$\subseteq$} (pomdp1);
    \node[] (pomdp2) at (7,0) {POMDP};
    \node[] (mdp2) at ($(pomdp2) + (0,-1.5)$) {MDP};
    \draw[opacity=0] (mdp2) to node[sloped,opacity=1] {$\subseteq$} (pomdp2);
    \fill[lightgray!40] (0.6,1.4) rectangle (9,0.6);
    \node[] (models1) at ($(pomdp1) + (0,0.97)$) {Single-environment models};
    \node[] (models2) at ($(pomdp2) + (0,1)$) {Multi-environment models};
    \draw[-,loosely dashed] (0.6,0.6) to (9,0.6);
    \draw[-,loosely dashed] (5,1.4) to (5,-1.85);
    \draw[color-deter] (-2,0.3) to (5,0.3);
    \node[color-deter, draw=none,anchor=west] at (-2.1,0.42) {\footnotesize{PSPACE-hard}};
    \node[color-deter, draw=none,anchor=west] at (-2.1,0.13) {\footnotesize{in NP}};
    \draw[color-deter] (-1.35,-1.1) to[out=0,in=180] (1.1,-1.1) to[out=0,in=140] (2.8,-1.45) to[out=-40,in=180] (4.3,-1.8) to[out=0,in=180] (5,-1.8);
    \node[color-deter, draw=none,anchor=west] at (-1.45,-0.93) {\footnotesize{NP-hard}};
    \node[color-deter, draw=none,anchor=west] at (-1.45,-1.27) {\footnotesize{in P}};
    \draw[color-normal] (-0.9,-0.4) to[out=0,in=180] (3.8,-0.4) to[out=0,in=110] (5,-1.05) to[out=-70,in=180] (6.8,-1.7) to[out=0,in=180] (9,-1.7);
    \node[color-normal, draw=none,anchor=west] at (-1,-0.28) {\footnotesize{PSPACE-hard}};
    \node[color-normal, draw=none,anchor=west] at (-1,-0.52) {\footnotesize{in NP}};
    \draw[color-normal] (-2.325,-1.05) to[out=0,in=180] (1.1,-1.05) to[out=0,in=140] (2.8,-1.4) to[out=-40,in=180] (4.3,-1.75) to[out=0,in=180] (9,-1.75);
    \node[color-normal, draw=none,anchor=west] at (-2.425,-0.93) {\footnotesize{NP-hard}};
    \node[color-normal, draw=none,anchor=west] at (-2.425,-1.27) {\footnotesize{in P}};
    \draw[color-memoryless] (-0.5,0.25) to[out=0,in=180] (5,0.25);
    \node[color-memoryless, draw=none,anchor=west] at (-0.6,0.42) {\footnotesize{ETR-hard}};
    \node[color-memoryless, draw=none,anchor=west] at (-0.6,0.13) {\footnotesize{in NP}};
    \draw[color-memoryless] (-0.375,-1.15) to[out=0,in=180] (1.1,-1.15) to[out=0,in=140] (2.8,-1.5) to[out=-40,in=180] (4.3,-1.85) to[out=0,in=180] (9,-1.85);
    \node[color-memoryless, draw=none,anchor=west] at (-0.475,-0.93) {\footnotesize{NP-hard}};
    \node[color-memoryless, draw=none,anchor=west] at (-0.475,-1.27) {\footnotesize{in P}};
    \coordinate (rec_nw) at (-2.325,-2);
    \draw[black, very thin,fill = white] (rec_nw) rectangle +(11.325,-0.44);
    \draw[color-normal,thick] ($(rec_nw) + (0.5,-0.22)$) to +(0.25,0);
    \node[draw=none,anchor=west] at ($(rec_nw) + (0.75,-0.2)$) {No restrictions};
    \draw[color-deter,thick] ($(rec_nw) + (3.6,-0.22)$) to +(0.25,0);
    \node[draw=none,anchor=west] at ($(rec_nw) + (3.85,-0.2)$) {Deterministic};
    \draw[color-memoryless,thick] ($(rec_nw) + (6.4,-0.22)$) to +(0.25,0);
    \node[draw=none,anchor=west] at ($(rec_nw) + (6.65,-0.233)$) {Markovian randomized policies};
\end{tikzpicture}}
    \caption{
    \emph{Known complexity results for finite-horizon ME-POMDPs. } 
    Dark blue and red lines correspond to problems with deterministic dynamics and Markovian policies, respectively, while bright blue lines correspond to problems with no dynamics or policy restrictions.
    An annotation ``$\mathsf{X}$-hard'' atop a line encodes that problems above that line are hard for the complexity class $\mathsf{X}$. 
    An annotation ``in $\mathsf{X}$'' below a line encodes that problems below the line are decidable in $\mathsf{X}$. 
    }
    \label{fig:complexity_overview_before}
\end{figure}

\begin{figure*}[t]
    \centering
    \resizebox{\textwidth}{!}{\input{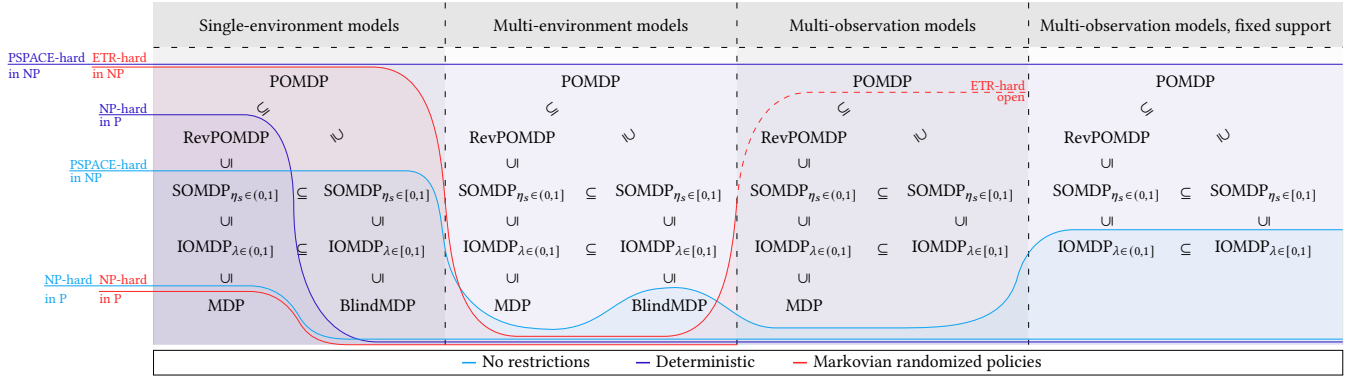}}
    \caption{\emph{Overview of ME-POMDP complexity including results we present.}
    Dark blue and red lines correspond to problems with deterministic dynamics and Markovian policies, respectively, while bright blue lines correspond to problems with no dynamics or policy restrictions.
    An annotation ``$\mathsf{X}$-hard'' atop a line encodes that problems above that line are hard for the complexity class $\mathsf{X}$. 
    An annotation ``in $\mathsf{X}$'' below a line encodes that problems below the line are decidable in $\mathsf{X}$. 
    }
    \label{fig:complexity_overview}
\end{figure*}

\subsection*{Contributions}
We investigate the complexity of finite-horizon reward maximization in multi-environment POMDPs across a large set of subclasses.
We consider these subclasses in full generality without further restrictions, and also examine special cases with deterministic transition and observation dynamics, as well as restricted policy classes such as finite-memory policies. 
Our results establish a significantly expanded complexity landscape for ME-POMDPs in \Cref{fig:complexity_overview} compared with \Cref{fig:complexity_overview_before}.
Our specific contributions are as follows.

\begin{enumerate}
    \item We show that POMDPs where all states are hidden through the same observation (Blind MDPs) and POMDPs with deterministic transition and observation functions (deterministic POMDPs) remain $\NP$-complete when extended to the multi-environment setting.
    These results rely on both subclasses respecting a certain \emph{closure under mixture} condition: 
    Allowing the adversary to randomly choose an environment from a fixed set of POMDPs in one of these subclasses is equivalent to the agent acting in a new POMDP in that subclass.
    \item We analyze a range of multi-environment models that lack the closure property above, yielding hardness results for a range of \emph{multi-observation} POMDP problems.
    We show that simple observation uncertainty, in particular uncertainty about state labelings or the probability that the true state is revealed, leads to $\PSPACE$-completeness, as in general ME-POMDPs.
    We also obtain novel complexity results for many single-environment POMDP subclasses.
    \item We characterize the complexity of ME-POMDPs with restricted policies.
    We show that numerous problems involving time-dependent Markovian policies or finite-state controllers that are $\NP$-complete for POMDPs become \emph{existential theory of the reals} ($\ETR$)-complete
    for ME-POMDPs.
\end{enumerate}
\paragraph{Outline}
After related work and preliminaries, \Cref{sec:single:env:models} details all involved subclasses. 
\Cref{sec:complexity_stays_prior_universal} presents the complexity results for ME-POMDPs using hidden-mixture reasoning.
Hardness results for ME-POMDPs built from subclasses that are not closed under hidden mixtures are in \Cref{sec:jumps_subclasses}.
Results on restricted policy classes for ME-POMDPs are in \Cref{sec_jumps}.

\section{Related Work}\label{sec_related_work}

A range of works provide algorithmic and complexity results for Multi-Environment MDPs (ME-MDPs).
In particular, for the robust setting, in which we assume no distribution on the environments, previous work characterizes the complexity of various qualitative problems \cite{chatterjee2025Value,DBLP:conf/icalp/BordaisR26,DBLP:conf/fsttcs/RaskinS14,DBLP:conf/concur/SuilenVJ24,DBLP:conf/tacas/VegtJJ23}, and provides algorithmic approaches for the reward-maximization problem \cite{DBLP:journals/cor/AhluwaliaSD21,DBLP:journals/iiset/SteimleAKD21}.
For finite-horizon reward maximization, the problem is $\PSPACE$-complete for history-dependent policies \cite{le2007robust,DBLP:journals/corr/abs-2605-07537} and $\NP$-hard when policies are restricted to be time-dependent Markovian \cite{le2007robust,DBLP:journals/orl/Li26}.
While we study ME-POMDPs, which strictly generalize ME-MDPs, we also tighten existing ME-MDP results by showing that the reward-maximization problem with time-dependent Markovian policies is $\ETR$-complete, where $\ETR$ lies between $\NP$ and $\PSPACE$. %
Single-environment POMDP complexity results provide trivial lower bounds on ME-POMDP complexity.
The $\PSPACE$-hardness of POMDPs \cite{DBLP:journals/mor/PapadimitriouT87}, along with the $\NP$-hardness of subclasses such as Blind MDPs \cite{DBLP:journals/mor/PapadimitriouT87} and deterministic POMDPs \cite{littman1996algorithms}, thus also holds for ME-POMDPs.
The closest related work is \cite{DBLP:journals/corr/abs-2605-07537}, which shows that finite-horizon reward maximization in ME-POMDPs is $\PSPACE$-complete, but it considers only general POMDPs, not subclasses.

The interaction between the agent and adversary in a ME-POMDP can be encoded as a one-sided zero-sum partially observable stochastic game (POSG) \cite{DBLP:conf/nips/BovyPSTJ25}.
As ME-POMDPs have a highly restricted adversary structure, complexity lower bounds for POSGs do not immediately transfer to ME-POMDPs.
While upper bounds for POSGs naturally transfer to ME-POMDPs, existing results already separate the two, with finite-horizon reward-maximization being $\PSPACE$-complete for deterministic POSGs\footnote{This fact follows from the proof of \cite{DBLP:journals/mor/PapadimitriouT87}.} but $\NP$-complete for deterministic ME-POMDPs, as we show.
Like POSGs, \emph{extensive-form games} also involve partial information and have well-known complexity results \cite{koller1992complexity}.
However, they are strictly more general than ME-POMDPs in the structure of the adversary's actions, and naive encodings of ME-POMDPs as imperfect-information games lead to an exponential blow-up in representation size.

\section{Preliminaries}\label{sec_prelims}

The sets of real numbers and natural numbers are  $\RR$ and $\NN$, respectively.
We denote the set of integers $\{1,\ldots,n\}$ as $[n]$.
We use $\dist{B}$ to denote the set of distributions on elements of a set $B$.
For a distribution $q \in \dist{X}$, $\supp(q)$ is the set of elements $x \in X$ such that $q(x) > 0$.
For $x \in X$, $\delta_x \in \dist{X}$ is the Dirac distribution which places all probability mass on $x$.
For a vector $\mathbf{u}$ we use $\vecat{\mathbf{u}}{i}$ or $\mathbf{u}(i)$ to denote the $i^{\text{th}}$ element, and $\mathbf{u}^\top$ denotes its transpose.

We assume basic familiarity with standard complexity classes $\mathsf{P}, \NP$, and $\PSPACE$, as well as the classical $\NP$-complete problem 3SAT, its $\PSPACE$-complete variant QSAT~\cite{DBLP:books/daglib/0072413}, and their \emph{conjunctive normal form} (CNF).
We also consider the complexity class $\ETR$.
The first-order theory of the reals over $(\mathbb{R}, 0, 1, +, \cdot, <)$ is the logic used to express whether a polynomial system of (in)equalities has a solution in $\mathbb{R}$. 
The \emph{existential} theory of the reals is the fragment that asks whether $\exists x_1, \dots, x_n. \phi(x_1,\dots,x_n)$ has a solution for a given quantifier-free formula $\phi$.
The truth problem for this fragment defines the complexity class $\ETR$, with $\NP \subseteq \ETR \subseteq \PSPACE$~\cite{DBLP:journals/corr/abs-2407-18006}.

\subsection{POMDPs}

A partially observable Markov decision process (POMDP) models an environment in which an agent takes actions to maximize reward under partial information on the true state of the world. 
\begin{definition}[POMDP]
    A partially observable Markov decision process (POMDP) is a tuple $(S,A,Z,T,O,R,b,H)$, where 
    $S$ is a finite set of states,
    $A$ is a finite set of actions,
    $Z$ is a finite set of observations,
    $T\colon S \times A \rightarrow \dist{S} $ is a transition function,
    $O\colon  S \rightarrow \dist{Z}$ is an observation function,
    $R\colon  S\times A \rightarrow \RR$ is a reward function,
    $b \in \dist{S}$ is an initial state distribution, and
    $H \in \mathbb{N}$ is the horizon.
\end{definition}
We write $T(s,a,s') = T(s,a)(s')$ and similarly, $O(s,z) = O(s)(z)$.
The state $s_0$ at time $0$ is drawn according to $b$.
For every $t \geq 0$, the observation $z_t$ is drawn from $O(s_t)$. 
The agent then observes $z_t$, takes an action $a_t$, and the next state is drawn from $T(s_t,a_t)$.

An agent acts in a POMDP using a policy that maps observation histories to distributions over actions:
$\pi\colon\cup_{t=0}^{H-1} Z(AZ)^{t} \rightarrow \dist{A}$.
We denote by $\Pi^{Z,A,H}$ the set of all such history-dependent randomized policies for a given observation space, action space, and horizon. 
When the POMDP is clear from context, we will simply write $\Pi$. 
A policy $\pi$ is deterministic if, for all histories $h$ in $\cup_{t=0}^{H-1} Z(AZ)^{t}$, $\pi(h)$ is a Dirac distribution, and we use $\polDet$ to denote the set of deterministic policies.
A \emph{mixed policy} is a distribution over deterministic policies, and the agent randomly draws one such policy before playing.
In the finite-horizon setting, any mixed policy $\pi'$ has an equivalent behavioral policy $\pi$ that induces the same path distribution in all POMDPs~\cite{kuhn1953extensive}.

The finite-horizon expected reward of a policy $\pi$ in a POMDP $\calM$ is $V_\calM^\pi = \EE_\calM^\pi\left[\sum_{t=0}^{H-1} R(s_t,a_t)\right]$, where $s_t$ and $a_t$ are random variables representing the state and action respectively at time $t$, for $t$ in the set $\{0,1,\ldots, H-1\}$, and we use use $\EE_\calM^\pi$ to denote the expectation under the measure induced by the policy $\pi$ and the POMDP $\calM$.

\subsection{Multi-Environment POMDPs}

A multi-environment POMDP (ME-POMDP) comprises a finite set of POMDPs with common state, action, and observation spaces, and a common horizon, but potentially different dynamics.
\begin{definition}[ME-POMDP]
    A multi-environment POMDP (ME-POMDP) is a tuple $(S,A,Z,(T_i)_{i \in [n]}, (O_i)_{i\in [n]}, (R_i)_{i\in [n]}, (b_i)_{i\in [n]}, H)$.
    Each index $i \in [n]$ defines a POMDP $\calM_i = (S,A,Z,T_i,O_i,R_i,b_i,H)$.
\end{definition}
We emphasize that we define ME-POMDPs \emph{without} a distribution over environments; thus, ME-POMDPs are not a subclass of POMDPs.
We define the set of history-dependent policies in the same manner as we do for POMDPs. 
We evaluate policies in ME-POMDPs in a \emph{worst-case} fashion.
For a policy $\pi$ in a ME-POMDP $(\calM_i)_{i\in[n]}$, we define its value as
\begin{equation}\label{eq:ME-POMDP:value}
    V_{(\calM_i)_{i\in [n]}}^\pi = \min_{i \in [n]} V^\pi_{\calM_i} =  \min_{i \in [n]} \mathbb{E}^\pi_{\calM_i}\left[ \sum_{t=0}^{H-1} R(s_t,a_t)\right].
\end{equation}

A \emph{multi-observation POMDP} (MO-POMDP) is a ME-POMDP in which the environments differ only in their observation functions.
That is, for all $i,j$ in $[n]$, we have $T_i = T_j$, $R_i = R_j$ and $b_i = b_j$. 
A MO-POMDP additionally satisfies the \emph{fixed support} condition if, for all states $s \in S$ and environments $i,j \in [n]$, the set of possible observations is fixed, i.e., $\supp(O_i(s)) = \supp(O_j(s))$.

\subsection{Decision Problems and Encodings}

We study the complexity of the following decision problem. 
\begin{problem}[Value problem]\label{prob:value:problem}
Given a ME-POMDP $(\calM_i)_{i\in [n]}$, and a rational number $B$, the 
value problem asks whether there exists a policy $\pi$ such that $V_{(\calM_i)_{i\in [n]}}^\pi \geq B$.
\end{problem}
\noindent
The value problem for restricted policy classes is defined analogously.
Problem~\ref{prob:value:problem} considers history-dependent randomized policies $\pi \in \Pi$, unless noted otherwise.
The value problem for single-environment models follows from \Cref{eq:ME-POMDP:value} with $n=1$.

We encode each ME-POMDP following the standard conventions for single-environment POMDPs~\cite{DBLP:journals/mor/PapadimitriouT87}. 
All transition, observation, and reward functions, as well as the initial distributions, are given as tables, where each entry is a rational number encoded as a ratio of two integers given in binary. 
The horizon is encoded in unary.

\section{Single-Environment Models}\label{sec:single:env:models}
In this section, we introduce POMDP subclasses we consider and the complexities of their reward-maximization problems. 
\Cref{fig:model_overview} summarizes the subclasses we consider.
We give complete proofs for all novel single-environment complexity results 
in~\Cref{sec:single_environment_full_proofs}.

\subsection{Subclasses of POMDPs}\label{sec_subclasses}
Recall a POMDP is a tuple $(S,A,Z,T,O,R,b,H)$.
We consider the following subclasses of POMDPs, ordered by increasing restriction.

\paragraph{Revealing POMDPs (RevPOMDPs)}
In a RevPOMDP, each state has a non-zero probability of revealing the true state through a unique observation \cite{DBLP:conf/aaai/BellyFG00V25,asadi2026revealing}.
Formally, a RevPOMDP is a POMDP where $Z$ contains a subset $\{\top_s \mid s \in S\}$\footnote{The subset of $Z$ can be any $|S|$ observations as long as the observation function maps a unique state to each observation.} of observations such that 
$    \forall s\in S, O(s,\top_s) > 0 \land \forall s'\in S, s' \neq s \implies O(s',\top_s) = 0$.
This definition encodes a variant of the strongly revealing POMDPs of \cite{DBLP:conf/aaai/BellyFG00V25} or revealing POMDPs of \cite{asadi2026revealing} with state-based observations.
RevPOMDPs are decidable for almost-sure, limit-sure, and quantitative reachability and parity objectives~\cite{DBLP:conf/aaai/BellyFG00V25,asadi2026revealing}, in contrast to general POMDPs where most objectives are undecidable~\cite{DBLP:conf/mfcs/ChatterjeeDH10,DBLP:journals/ai/MadaniHC03}.

\paragraph{Semi-Observable MDPs (SOMDPs)}
In a SOMDP, each state has a revealing observation and is otherwise unobservable \cite{DBLP:conf/iros/BasichPZ22}.
The revelation probability may vary across states, and is encoded by a likelihood function $\eta \colon S \to [0,1]$. 
Formally, a SOMDP is a POMDP where the observation set comprises revealing observations for each state and a blind observation, i.e., $Z = \{\top_s \mid s \in S\} \cup \{\bot\}$. The observation function then must satisfy
$    O(s) = \{\top_s \mapsto \eta(s), \bot \mapsto 1- \eta(s)\}$
for some function $\eta \colon S \to [0,1]$.
This definition corresponds to a variant of SOMDPs with state-based observations \cite{DBLP:conf/iros/BasichPZ22}.
In SOMDPs, one can track the action-observation history by tracking the last observed state and all actions taken since then \cite{DBLP:conf/iros/BasichPZ22}.

\paragraph{Intermittently-Observable MDPs (IOMDPs)}
In an IOMDP, each state has a revealing observation and is otherwise unobservable \cite{DBLP:journals/tac/ChenL26}, as in a SOMDP.
However, IOMDPs additionally require the revelation probability to be constant across states.
Formally, an IOMDP is a POMDP where $Z = \{\top_s \mid s \in S\} \cup \{\bot\}$, and the observation function satisfies
$    O(s) = \{\top_s \mapsto \lambda, \bot \mapsto 1- \lambda\}$
for some $\lambda \in [0,1]$. 
In an IOMDP, action-observation histories are tracked by the last observed state and all actions since then.

\paragraph{Blind MDPs (BlindMDPs)}
In a BlindMDP, there is only a single observation  \cite{DBLP:journals/mor/PapadimitriouT87,DBLP:journals/ai/MadaniHC03,DBLP:journals/jacm/MundhenkGLA00}.
Formally, it is a POMDP where $Z = \{\bot\}$, and for each state $s$, $O(s,\bot) = 1$.

\paragraph{MDPs}
MDPs are a fully observable subclass of POMDPs, in which each state has a single unique observation \cite{DBLP:books/wi/Puterman94}.
Formally, an MDP is a POMDP where $Z = \{\top_s \mid s \in S\}$ consists of a revealing observation for each state $\top_s$, also called a \emph{state labeling}, and the observation function is deterministic with $O(s,\top_s) = 1$.

\paragraph{Partial orders on subclasses}
In \cref{fig:model_overview}, we arrange the above POMDP subclasses by the inclusion relation.
Note that the subset inclusion of Blind MDPs depends on whether the three middle models, \ie, IOMDPs, SOMDPs, and RevPOMDPs, allow revelation probabilities of $0$.
Also note that allowing RevPOMDPs to include zero revelation probabilities recovers the class of all POMDPs.
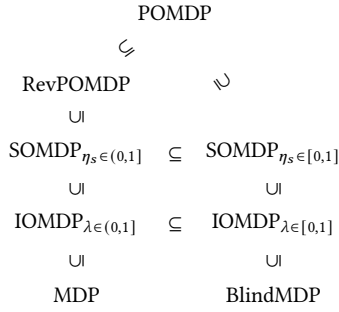
\begin{figure}[t]
    \centering
    \resizebox{0.55\columnwidth}{!}{\begin{tikzpicture}[font=\normalsize,state/.append style={shape = circle}, >=stealth, bobbel/.style={minimum size=1mm,inner sep=0pt,fill=black,circle}]
    \node[] (pomdp1) at (0,0) {POMDP};
    \node[] (rev_pomdp1) at ($(pomdp1) + (-1.4,-1)$) {RevPOMDP};
    \node[] (semi_mdp1a) at ($(rev_pomdp1) + (0,-1)$) {SOMDP$_{\eta_s \in (0,1]}$};
    \node[] (inter_mdp1a) at ($(semi_mdp1a) + (0,-1)$) {IOMDP$_{\lambda \in (0,1]}$};
    \node[] (semi_mdp1b) at ($(rev_pomdp1) + (2.8,-1)$) {SOMDP$_{\eta_s \in [0,1]}$};
    \node[] (inter_mdp1b) at ($(semi_mdp1b) + (0,-1)$) {IOMDP$_{\lambda \in [0,1]}$};
    \node[] (mdp1) at ($(inter_mdp1a) + (0,-1)$) {MDP};
    \node[] (blind_mdp1) at ($(inter_mdp1b) + (0,-1)$) {BlindMDP};
    \draw[opacity=0] (rev_pomdp1) to node[sloped,opacity=1] {$\subseteq$} (pomdp1);
    \draw[opacity=0] (semi_mdp1a) to node[sloped,opacity=1] {$\subseteq$} (rev_pomdp1);
    \draw[opacity=0] (inter_mdp1a) to node[sloped,opacity=1] {$\subseteq$} (semi_mdp1a);
    \draw[opacity=0] (mdp1) to node[sloped,opacity=1] {$\subseteq$} (inter_mdp1a);
    \draw[opacity=0] (semi_mdp1b) to node[sloped,opacity=1,xscale=-1] {$\subseteq$} (pomdp1);
    \draw[opacity=0] (inter_mdp1b) to node[sloped,opacity=1] {$\subseteq$} (semi_mdp1b);
    \draw[opacity=0] (inter_mdp1a) to node[sloped,opacity=1] {$\subseteq$} (inter_mdp1b);
    \draw[opacity=0] (semi_mdp1a) to node[sloped,opacity=1] {$\subseteq$} (semi_mdp1b);
    \draw[opacity=0] (blind_mdp1) to node[sloped,opacity=1] {$\subseteq$} (inter_mdp1b);
\end{tikzpicture}}
    \caption{Overview of the models considered in this paper.}
    \label{fig:model_overview}
\end{figure}

\paragraph{Deterministic dynamics}
For each subclass introduced above, we also consider a variant with deterministic transition and observation functions.
Note that, when the observation function is deterministic, RevPOMDPs, SOMDPs with $\eta_s \in (0,1]$, and IOMDPs with $\lambda \in (0,1]$ all collapse to MDPs as each revelation probability becomes $1$.
Also, IOMDPs with $\lambda \in [0,1]$ collapse to either MDPs or Blind MDPs, as the revelation probability is constant.

\paragraph{Subclasses not considered}
We focus on the subclasses above as they prove sufficient to uncover interesting complexity behavior.
However, we also note that there are numerous other interesting subclasses of POMDPs which we do not characterize, such as posterior-deterministic POMDPs \cite{DBLP:journals/corr/abs-2602-07473}, quasi-deterministic POMDPs \cite{besse2009quasi} or observable POMDPs \cite{golowich2022planning}. 
We leave exploration of multi-environment versions of these subclasses to future work.

\subsection{Complexity of Single-Environment Models}\label{sec_base_results}

We discuss the complexity of finite-horizon reward maximization in single-environment models.
The first column of \cref{fig:complexity_overview} summarizes the results.
First, recall that this problem for MDPs without any restrictions is in $\mathsf{P}$ \cite{DBLP:journals/mor/PapadimitriouT87}, and for Blind MDPs it is $\NP$-complete \cite{DBLP:journals/mor/PapadimitriouT87,DBLP:journals/jacm/MundhenkGLA00}.
We now show that for all models between IOMDPs and SOMDPs the problem is $\NP$-complete.
\begin{restatable}[]{proposition}{propIOMDPsBase}\label{prop:IOMDPs_base}
    The value problem for IOMDPs with $\lambda \in (0,1]$ is $\NP$-hard.
\end{restatable}
\begin{proofsketch}
We reduce 3SAT to the value problem for IOMDPs.
For each 3SAT instance, one can build a BlindMDP with non-positive reward such that the 3SAT instance is satisfiable if and only if the BlindMDP has zero expected reward \cite{DBLP:journals/jacm/MundhenkGLA00}. 
To prove $\NP$-hardness for IOMDPs, we use the same transition structure, but allow the agent to see the true state with probability $\nicefrac{1}{2}$.
With positive probability, the agent never observes the state and, in effect, plays a sequence of actions in a BlindMDP, and we can argue that this sequence of actions also defines a satisfying assignment for the original 3SAT instance.
\end{proofsketch}

Proposition~\ref{prop:IOMDPs_base} implies that finite-horizon reward maximization for all models in the class $\textrm{IOMDP}_{\lambda \in (0,1]}$ is $\NP$-hard. 
We next provide an $\NP$ upper bound for all models below $\textrm{SOMDP}_{\eta_s \in [0,1]}$.

\begin{restatable}[]{proposition}{propSOMDPsBase}\label{prop:SOMDPs_base}
    The value problem for SOMDPs with $\eta_s \in [0,1]$ is in $\NP$.
\end{restatable}
\begin{proofsketch}
    We show that the optimal policy in a SOMDP can be guessed and evaluated in polynomial time.
    This policy is belief-based and deterministic, tracking the last observed state, the time of the last observation, and the current time.
    The belief in an SOMDP collapses every time a state is revealed.
    Until the next state revelation, all observations are blind, so the decision-maker has only the actions and number of time steps that have passed as available information.
    Since the actions are chosen deterministically, tracking the last revelation time and the current time step suffices to track the belief.
\end{proofsketch}

We complete the complexity results for single-environment models without deterministic dynamics by proving that reward maximization in RevPOMDPs is $\PSPACE$-hard. 
\begin{restatable}[]{proposition}{propRevPOMDPsBase}\label{prop:revPOMDPs_base}
    The value problem for revealing POMDPs is $\PSPACE$-hard.
\end{restatable}
\begin{proofsketch}
We reduce from QSAT. %
For each QSAT instance, one can build a POMDP with non-positive reward such that the QSAT instance is satisfiable if and only if the POMDP has zero expected reward \cite{DBLP:journals/jacm/MundhenkGLA00}. 
To prove $\PSPACE$-hardness for RevPOMDPs, we use the same transition structure, but allow the agent to see the true state with probability $\nicefrac{1}{2}$.
With positive probability, the agent never observes the state and, in effect, plays a sequence of actions in a POMDP, and we can argue that this sequence of actions also defines a satisfying assignment for the original QSAT instance.    
\end{proofsketch}

\paragraph{Deterministic Single-Environment Models}
Next, we discuss subclasses with deterministic dynamics.
Deterministic MDPs are a subset of MDPs, and are thus decidable in $\Ptime$.
Deterministic RevPOMDPs collapse to MDPs because they must have a positive probability of revealing each state. 
Therefore, every model that is a subset of RevPOMDPs lies in $\Ptime$ for deterministic dynamics.
For BlindMDPs, the reduction used to prove $\NP$-hardness in \cite{DBLP:journals/jacm/MundhenkGLA00} uses deterministic transitions. As such, BlindMDPs, and any superset model, remain $\NP$-hard in the deterministic case.
Finally, the general deterministic POMDP (DetPOMDP) problem lies in
$\NP$ \cite{littman1996algorithms,DBLP:conf/uai/Bonet09}.

\subsection{Complexity of Policy Restrictions}
\label{sec:restricted_policy_single_env}

Restricting policy classes also reduces POMDP complexity, and we next discuss the complexity of the associated decision problems.
Under a time-dependent Markovian policy, an agent chooses an action using only their last observation and the current time, i.e.,
$\pi \colon \{0,\ldots, H-1\} \times Z \rightarrow \dist{A}$.
We write $\polMR$ for the set of time-dependent Markovian randomized policies. 
A policy is \emph{stationary} if it only depends on the observation, and not the time step. 

The value problem under time-dependent Markovian policies is $\NP$-complete for all POMDP subclasses besides MDPs.
By~\cite{DBLP:journals/jacm/MundhenkGLA00}, the value problem for time-dependent Markovian policies lies in $\NP$ when restricted to deterministic policies, and these deterministic policies are optimal among time-dependent Markovian policies~\cite{DBLP:conf/nips/BagnellKNS03}.
Meanwhile, $\NP$-hardness for Blind MDPs follows by observing that history-dependent and time-dependent Markovian policies coincide in this subclass.
For IOMDPs with $\lambda \in (0,1]$, we can also prove $\NP$-hardness using similar arguments to those used in proving $\NP$-hardness for Blind MDPs.
\begin{restatable}[]{proposition}{propIOMDPsMemoryless}\label{prop:memoryless_IOMDP}
    The value problem under time-dependent Markovian policies for IOMDPs is $\NP$-hard, even when $\lambda \in (0,1]$.
\end{restatable}
\begin{proofsketch}
    For both time-dependent Markovian and history-dependent policies, the agent's policy on the event that it never observes the state is a sequence of actions.
    Analogously to the proof of \Cref{prop:IOMDPs_base}, we can establish a correspondence between these action sequences and satisfying assignments of a 3SAT instance, which we in turn use to establish $\NP$-hardness.
\end{proofsketch}

While not shown in \Cref{fig:complexity_overview}, we also explore finite-state controllers \cite{DBLP:conf/nips/PoupartB03,DBLP:conf/uai/MeuleauKKC99} for POMDPs and their multi-environment counterparts.
A generalization of a time-dependent Markovian policy is a time-dependent FSC, in which an agent can use a set of memory states to store information about past actions and observations.
\begin{definition}[Time-dependent FSC]
    For a (ME-)POMDP with observations $Z$, actions $A$, and horizon $H$, a time-dependent FSC is a tuple
    $(Q, \iota, (\gamma_t)_{t \in \{0,\ldots, H-1\}}, (\eta_t)_{t \in \{0,\ldots, H-1\}})$,
    where 
    $Q$ is a set of memory states,
    $\iota \in \dist{Q}$ is an initial memory-state distribution,
    $\gamma_t\colon Q \times Z \rightarrow \dist{A}$ is a time-dependent action mapping, and  %
    $\eta_t \colon Q \times Z \times A \rightarrow \dist{Q}$ is a time-dependent memory update mapping.
\end{definition}
A time-dependent FSC is equivalent to a standard time-independent FSC with memory nodes $Q' = Q \times \{0,\ldots, H-1\}$ where one always updates $(q,t)$ to $(q',t+1)$.
The value problem for FSCs mirrors Problem~\ref{prob:value:problem}, but takes an additional integer $|Q|$ as input to encode the memory size the agent may use.
\begin{problem}[FSC value problem]\label{prob:fsc:value:problem}
        Given a ME-POMDP $(\calM_i)_{i\in [n]}$, a rational number $B$, and a memory bound $|Q|$ polynomial in the ME-POMDP size, the FSC value problem asks whether there exists a time-dependent FSC $\pi$ such that $V_{(\calM_i)_{i\in [n]}}^\pi \geq B$.
\end{problem}

\begin{restatable}[]{lemma}{lemPOMDPsFSCs}\label{fac:fsc_rep_fac}
    For any POMDP $\calM$, and fixed $|Q|$, there exists a deterministic optimal time-dependent FSC for $ \calM$.  
    Consequently, the time-dependent FSC value problem for POMDPs is $\NP$-complete. 
\end{restatable}
\begin{proofsketch}
    We can construct a product POMDP $\calM'$ with state space $S \times Q$~\cite{junges2018finite}.
    Time-dependent FSCs in the original POMDP then correspond to time-dependent Markovian policies in this product, and \emph{deterministic} time-dependent Markovian policies are already optimal \cite{DBLP:conf/nips/BagnellKNS03}. 
    Thus, we can simply guess and evaluate such policies in $\calM'$ to decide the FSC value problem.
    $\NP$-hardness follows from the time-dependent Markovian case, which we can encode by restricting $|Q|$ to be one.
\end{proofsketch}

\section{Upper Bounds via Hidden Mixtures}
\label{sec:complexity_stays_prior_universal}
The complexity results for subclasses of POMDPs naturally raise the question of whether they extend to the multi-environment setting.
We next provide a general method for characterizing when this extension is possible by characterizing subclass behavior under hidden mixtures.
Full proofs reside in 
\Cref{sec:proofs_complexity_stays_prior_universal}.

Many positive POMDP subclass complexity results follow due to optimal policy structure.
Indeed, if one can guess a policy quickly, then a natural $\NP$ algorithm exists for determining the value problem.
BlindMDPs and DetPOMDPs exhibit this structure. 
Indeed, any deterministic policy in a BlindMDP is simply an action sequence, while optimal policies in DetPOMDPs have concise tree representations \cite{littman1996algorithms}.

We show that BlindMDPs and DetPOMDPs share another favorable property: they are \emph{closed under hidden mixtures}, and this property lets us transfer complexity results to ME-POMDPs.
Informally, a subclass $\calC$ is closed under hidden mixtures if allowing the adversary to randomly choose one of $n$ POMDPs from $\calC$ yields a new POMDP in $\calC$.
As we show in this section, ME-POMDPs built from subclasses that (1) satisfy this hidden-mixture condition, and (2) admit structured optimal policies, also admit efficient algorithms.

We first abstract the notion of a policy class.
A policy class defines a set of policies for each POMDP depending on its observation, state, and action space and horizon.
Let $\calS, \calA,$ and $\calZ$ be sets encoding all possible state spaces, action spaces, and observation spaces.\footnote{We remark that while the definition of the set of all state sets seems unrigorous, we give a more formal justification in 
Appendix~\ref{sec:policy_class_soas}.}
\begin{definition}[Policy class]
    A policy class is a family of policy sets $\Pi_{\calC} = (\Pi_\calC^{S,Z,A,H})_{S \in \calS, A \in \calA, Z \in \calZ, H \in \mathbb{N}}$ such that
    \begin{equation}
        \Pi_\calC^{S,Z,A,H} \subset \Pi^{Z,A,H} \quad \forall Z \in \calZ, A \in \calA, H \in \mathbb{N}.
    \end{equation}
\end{definition}
While state-space-dependence may seem unnecessary, it handles optimal policies for DetPOMDPs, whose representation depends on $|S|$.
We next need to tie optimal policy classes for POMDP subclasses to their multi-environment counterparts.
We bridge this gap by introducing \emph{hidden-mixture} POMDPs.

A \emph{hidden mixture of a finite set of POMDPs} is a POMDP in which the environment is randomly drawn before the agent executes their policy.
Fix a set $(\calM_i)_{i\in[n]}$ of POMDPs and a distribution $\mathbf{p} \in \dist{[n]}$.
In the resulting hidden mixture POMDP, the adversary uses $\mathbf{p}$ to draw an environment, and the dynamics unfold according to environment $i$. 
We use mixtures to define new POMDP subclasses.
\begin{definition}[Class of hidden mixture POMDPs]
    Given a subclass $\calC$ of POMDPs, $\mathsf{Mix}(\calC)$ is the set of POMDPs formed by mixing over a finite number of POMDPs from $\calC$.
\end{definition}
Hidden mixture POMDPs are structurally distinct from ME-POMDPs, as they assume a distribution over environments and, in the terminology of \cite{DBLP:conf/icalp/BordaisR26}, have \emph{prior semantics}.
Hidden mixture POMDPs nonetheless serve as a powerful tool for reasoning about ME-POMDP complexity, as our next result shows.

The optimal policy in a ME-POMDP built from a subclass $\calC$ randomizes over a small set of policies that are optimal in $\mathsf{Mix}(\calC)$.
\begin{restatable}[]{proposition}{propPriorUniversal}\label{prop:prior_universal}
    Fix a class $\calC$ of POMDPs, and suppose $\Pi_\calC$ is a policy class that contains a deterministic optimal policy for every POMDP in $\mathsf{Mix}(\calC)$.
    If $(\calM_i)_{i \in [n]}$ is a ME-POMDP with state, action and observation sets $S,A,Z$ such that each POMDP lies in $\calC$, then there exists an optimal policy that mixes over $n$ elements of $\Pi_{\calC}^{(S\times [n]), Z,A , H}$.
\end{restatable}
\begin{proofsketch}
    Let $(\calM_i)_{i\in [n]}$ be a ME-POMDP with POMDPs in $\calC$ where each POMDP has state, action and observation spaces $S$, $A$, and $Z$. The value satisfies the following chain of equalities.
    \begin{multline*}
        \max_{\pi \in \Pi} \min_{i \in [n]} V_{\calM_i}^\pi = \min_{\mathbf{p} \in \dist{[n]}} \max_{\pi \in \polDet^{Z,A, H}} \sum_{i \in [n]} \vecat{\mathbf{p}}{i} V^\pi_{\calM_i}
        \\ = \min_{\mathbf{p} \in \dist{[n]}} \max_{\pi \in \Pi_{\calC}^{S\times [n], Z, A, H}} \sum_{i \in [n]} \vecat{\mathbf{p}}{i} V^\pi_{\calM_i}
        = \max_{\pi \in \mathsf{Mix}\left( \Pi_{\calC}^{S\times [n], Z,A, H}\right)} \min_{i \in [n]} V_{\calM_i}^\pi.
    \end{multline*}
    The first and last equalities apply the min-max theorem for zero-sum matrix games~\cite{basar1998dynamic}, which holds because $\polDet$ is finite and the number of environments is finite. 
    The middle inequality uses the optimality of policies in $\Pi_\calC$ for mixture POMDPs, along with the fact that each mixture POMDP has state space $S \times [n]$.
    To deduce that the optimal policy mixes over at most $n$ policies in $\Pi_\calC$, we can then either use a linear programming result or appeal to Carath\'eodory's theorem, as in \cite{DBLP:journals/corr/abs-2605-07537}.
\end{proofsketch}

Using \Cref{prop:prior_universal}, we can show that the value problem for both deterministic ME-POMDPs and ME-BlindMDPs lies in $\NP$.
Critically, these classes satisfy $\calC = \mathsf{Mix}(\calC)$. That is, they are \emph{closed} under mixtures. 
Thus, by \Cref{prop:prior_universal}, we can solve the \emph{ME}-POMDP by guessing $n$ policies for the original subclass and verifying with a linear program (LP) that we can mix over these policies.

\begin{restatable}[]{corollary}{corDetMEPOMDP}\label{cor:det_MEPOMDP}
    The value problem for ME-DetPOMDPs is in $\NP$.
\end{restatable}
\begin{proofsketch}
    When we mix over a set of DetPOMDPs $(\calM_i)_{i \in [n]}$, we get another DetPOMDP $\tilde{\calM}$ with state space $S \times [n]$. 
    By \cite{littman1996algorithms}, this POMDP $\tilde{\calM}$ will have an optimal policy with representation-size polynomial in $n|S|$, $|A|$, $|Z|$ and $H$. 
    These polynomially sized policies form the set $\Pi_\calC$ of policies for the class $\calC$ of deterministic POMDPs.
    Thus, to determine if a ME-POMDP with deterministic dynamics has a policy with value at least $B$, we can guess $n$ policies from $\Pi_{\calC}$ and solve an LP, all in polynomial time. 
\end{proofsketch}

\begin{restatable}[]{corollary}{corBlindMEPOMDP}
    \label{cor:blind_mepompd_in_NP}
    The value problem for ME-BlindMDPs is in $\NP$.
\end{restatable}
\begin{proofsketch}
    Any optimal policy for a BlindMDP comprises a sequence of actions.
    Also, mixing over a set of BlindMDPs $(\calM_i)_{i \in [n]}$ results in another BlindMDP $\tilde{\calM}$.
    Thus, by the same argument as \Cref{cor:det_MEPOMDP}, we can guess and check a randomized policy in polynomial time.
\end{proofsketch}

The above corollaries settle the complexity of all deterministic ME-POMDP subclasses we consider, and we additionally obtain our first instance of an $\NP$-complete ME-POMDP problem with stochastic dynamics.
However, the above mixture reasoning fails for subclasses such as SOMDPs, as they are not closed under mixtures.
Similarly, \Cref{prop:prior_universal} does not apply to restricted policy classes, as a mixture over deterministic time-dependent Markovian policies may not be a time-dependent Markovian policy.

\section{Multi-Observation POMDPs}

\label{sec:jumps_subclasses}

We next present hardness results for subclasses where the mixture reasoning of \cref{sec:complexity_stays_prior_universal} fails.
We identify subclasses that reside in $\Ptime$ or $\NP$ in the single-environment case, but become $\PSPACE$-complete when we examine the relevant \emph{multi-observation} problem.
In identifying these subclasses, we expand the range of $\PSPACE$-hard multi-environment problems beyond known examples such as ME-MDPs~\cite{le2007robust}.
Full proofs reside in 
\Cref{sec:proofs_jumps_subclasses}.

First, we observe that multi-observation MDPs (MO-MDPs) are $\PSPACE$-hard despite the significant structure of their observation functions.
A MO-MDP comprises a set of MDPs with the same dynamics but potentially different state labelings.
We demonstrate that shuffling this state labeling is sufficient to simulate partial observability and thus recover $\PSPACE$-hardness.
\begin{restatable}[]{proposition}{propMOMDPsBase}\label{prop:MOMDPs_base}
    The value problem for MO-MDPs is $\PSPACE$-hard.
\end{restatable}
\begin{proofsketch}
    In the classical proof of $\PSPACE$-hardness for POMDPs~\cite{DBLP:journals/mor/PapadimitriouT87}, 
    one transforms a QSAT instance to a POMDP in which nature randomly selects one of $m$ clauses in the conjunctive normal form (CNF) formula at the initial time step, and then uses randomness to assign values to universal variables \cite{DBLP:journals/mor/PapadimitriouT87}. 
    The agent cannot observe the chosen clause and must therefore use a clause-independent policy to choose existential variables so the formula remains true for any universal variable assignment.
    We simulate this partial observability in MO-MDPs using multiple state labelings.
    We first design a reference environment with the same transition structure as the above POMDP, where states corresponding to clause $l$ emit observations in some set $Z_l$, for $l \in [m]$.
    We then define $m - 1$ new environments, where, in environment $k$, states corresponding to clause $l$ emit observations in the set $Z_{l + k}$.
    We can then extract a clause-independent assignment strategy by examining the agents' behavior on observations in the set $Z_1$, since for any clause $l$, there exists an environment in which the corresponding states emit observations in $Z_1$.
\end{proofsketch}
This shuffling trick fails if we assume the observation function's support is fixed between environments. 
However, as we next demonstrate, fixed-support models can still exhibit $\PSPACE$-hardness.

Even with fixed supports, multi-observation semi-observable MDPs remain $\PSPACE$-complete.
We again prove this result by adapting the proof of \cite{DBLP:journals/mor/PapadimitriouT87}, by letting the agent and the environment alternately choose binary variables in a CNF formula.
We design an environment for each clause, and we design environment $i$ so that satisfying clause $i$ reveals information that helps the agent reach an accepting state.

\begin{restatable}[]{proposition}{propMOSOMDPs}\label{prop:MOSOMDPs}
    The value problem for MO-SOMDPs with fixed support is $\PSPACE$-hard.
\end{restatable}
\begin{proofsketch}
    We design an environment corresponding to each clause in a CNF formula with $n$ variables.
    In particular, fix a clause $i$. 
    The POMDP for clause $i$ has $2n$ stages: odd stages correspond to variable assignment, and even stages correspond to a guessing game in which the agent can reach an accepting state.
    At variable assignment stages, either the environment or the agent chooses a variable assignment, as in \cite{DBLP:journals/mor/PapadimitriouT87}.
    At the guessing game stage, an unbiased coin is flipped, which the agent must identify to reach an accepting state.
    If the last variable assignment satisfied clause $i$, the environment reveals the state, and thus the coin value, with high probability.
    Conversely, if the last variable assignment did not set the clause to true, the state remains hidden with high probability.
    In this MO-POMDP, policies that eventually reveal the coin's value with high probability correspond to strategies that satisfy the given QSAT instance. Conversely, given a QSAT solution, we can build a POMDP policy that reveals the information with sufficiently high probability.
    This multi-environment problem is also a MO-SOMDP with fixed support. 
    At the variable assignment stages, we can reveal the true state in all environments.
    In the guessing game states, two observations are possible in all environments: the blind observation and the state's unique observation. 
\end{proofsketch}

Unlike SOMDPs, however, IOMDPs have sufficient structure to avoid $\PSPACE$-hardness. 
In an IOMDP, every state has an equal revelation probability.
This structure in turn gives MO-IOMDPs certain monotonicity properties \cite{DBLP:journals/tac/ChenL26},
which allow us to solve the MO-IOMDP by identifying and solving the worst environment, i.e., the environment with the lowest revelation probability.

\begin{restatable}[]{proposition}{propMOIOMDPs}\label{prop:IOMDPs_in_NP}
    The value problem for MO-IOMDPs with fixed support is in $\NP$.
\end{restatable}
\begin{proofsketch}
    Increasing the revelation probability in an MO-IOMDP cannot decrease the reward.
    In particular, fix a transition function, and let $\pi_\lambda$ be the optimal policy for a given revelation probability $\lambda$. 
    By Corollary~1 in \cite{DBLP:journals/tac/ChenL26}, increasing the revelation probability to $\lambda'$ increases the expected reward for $\pi_\lambda$.
    Thus, to solve the MO-IOMDP problem, we merely need to guess and check an optimal policy for an IOMDP, which we can do efficiently by following the reasoning used to prove \Cref{prop:SOMDPs_base}.
\end{proofsketch}
This last result provides an instance of a MO-POMDP problem with stochastic dynamics and a non-trivial observation function that avoids $\PSPACE$-hardness. However, the associated monotonicity property is a stringent condition on the MO-POMDP.

\section{Restricted Policy Classes}\label{sec_jumps}

We next demonstrate that a range of restricted policy class problems that are $\NP$-complete for POMDPs are $\ETR$-complete when we move to the ME-POMDP setting. Full proofs reside in 
\Cref{sec:restricted_pol_full_proofs}.

\label{sec:restricted_pol_classes}
\subsection{Markovian Policies}

The difficulty of finding time-dependent Markovian policies for ME-POMDPs lies in reasoning about decoupled randomness.
When finding time-dependent Markovian policies in POMDPs, we can assume the policy is deterministic, i.e., in $\polDet$, and guess a policy in polynomial time.
However, optimal ME-POMDP policies generally require randomness.
One might try to solve the ME-POMDP problem by reasoning about mixtures over deterministic Markovian policies, as in \Cref{sec:complexity_stays_prior_universal}.
However, a mixture of deterministic Markovian policies may not be a Markovian policy, i.e., 
   $ \polMR \subsetneq \dist{\polDet \cap \polMR}$.
In particular, the class of time-dependent Markovian policies requires that the randomness be decoupled across time steps, creating problems that are known across game theory and distributed systems; see, e.g. 
\cite{DBLP:conf/lics/BriceHT26,DBLP:conf/wine/HansenHMS08,borgs2008myth,DBLP:conf/aaai/Celli018}.

For all the subclasses we consider, the time-dependent Markovian policy problem for ME-POMDPs is $\ETR$-hard. 
In a ME-BlindMDP, the agent's strategy comprises a sequence of action distributions, one for each time step.
We can establish hardness results by viewing each time step as a player on a team.
That is, finding a time-dependent Markovian strategy in a ME-BlindMDP means finding a set of mixed strategies for a team to use against an adversary, where team members can not correlate their randomness.
Finding optimal team strategies in such games is $\ETR$-hard.

\begin{restatable}[]{proposition}{propMEBlindPOMDPsMemoryless}\label{prop:mepomdp_blind_timedep}
    The time-dependent Markovian policy problem for deterministic ME-BlindMDPs is $\ETR$-hard.
\end{restatable}
\begin{proofsketch}
    We reduce $2$-vs.-$1$ team matrix games with two players, known to be $\ETR$-hard~\cite{DBLP:conf/lics/BriceHT26,DBLP:journals/mst/BerthelsenH22,DBLP:journals/mst/Hansen19}, to ME-BlindMDPs.
\end{proofsketch}
In the proof for ME-BlindMDPs, we store the agent's first action in the state, which prevents the result from generalizing to ME-MDPs because the agent can then use this information at the next time step, even with a Markovian policy.
However, we can resolve this issue using a state-relabeling construction.

\begin{restatable}[]{proposition}{propDetMEMDPsMemoryless}
    \label{prop:memoryless_MEMDP}
    The time-dependent Markovian policy problem for deterministic multi-environment MDPs is $\ETR$-hard.
\end{restatable}
\begin{proofsketch}
    Similar to the proof of \Cref{prop:mepomdp_blind_timedep}, we encode team matrix games in ME-MDPs.
    However, we add additional environments that differ in state-space labeling, preventing the agent from using state observations at the second time step to create correlations between actions at different time steps.
\end{proofsketch}
\Cref{prop:memoryless_MEMDP,prop:mepomdp_blind_timedep} establish $\ETR$-hardness for all the multi-environment problems we consider.
We can also prove $\ETR$-hardness for the multi-observation problem by changing the observation function to control the agent's actions and simulate an adversary.
\begin{restatable}[]{proposition}{propMOPOMDPsMemoryless}
    \label{prop:mopomdp_single_reward}
    The time-dependent Markovian policy problem for multi-observation POMDPs is $\ETR$-hard.
\end{restatable}
\begin{proofsketch}
    The proof begins with the same reduction as in the proof of \Cref{prop:mepomdp_blind_timedep}. However, the environment can no longer modify the transition or reward functions. 
    To address this difficulty, we add a component to the problem in which certain observations are revealed \emph{conditional} on the agent taking a certain action in each environment. 
    Adding this component then allows the adversary to control the environment dynamics by incentivizing the agent to take certain actions.
\end{proofsketch}

We can also easily express the time-dependent Markovian policy problem in $\ETR$.
Indeed, we parameterize the policy with $|Z||A|H$ variables and encode the dynamics as bilinear constraints linking $\mathbb{P}[s_{t+1} = s]$ to $(\mathbb{P}[s_t = s, a_t = a])_{a\in A, s \in S}$. 
\begin{restatable}[]{lemma}{lemMEPOMDPsMemoryless}\label{fac:ETR_inclusion_mem}
    The time-dependent Markovian policy problem for ME-POMDPs is in $\ETR$.
\end{restatable}
\begin{proofsketch}
    This result follows from standard encodings of robust POMDP problems as non-linear constrained optimization problems, for example as in \cite{DBLP:conf/aaai/Cubuktepe0JMST21}. 
    Critically, because the ME-POMDP comprises a finite set of POMDPs, the $\ETR$-encoding is finite.
\end{proofsketch}

\subsection{Policies with Finite Memory}

Next, we show a counter-intuitive relationship between memory and complexity in ME-POMDPs. 
For any memory size, the value problem for time-dependent FSCs in POMDPs lies in $\NP$. 
In contrast, we show that ME-POMDPs exhibit a cutoff phenomenon, in which the complexity of the FSC problem depends on the ratio of memory size to the number of environments.
We adapt the FSC value problem (Problem~\ref{prob:fsc:value:problem}) with a fixed rational parameter $\alpha$.
\begin{problem}[$\alpha$-FSC value problem]\label{problem:alpha:fsc:value:problem}
        Given a ME-POMDP $(\calM_i)_{i\in [n]}$, a rational number $B$, and a memory bound $|Q|$ polynomial in the ME-POMDP size that satisfies $|Q| \geq \alpha n$, the $\alpha$-FSC value problem asks whether there exists a time-dependent FSC $\pi$ such that $V_{(\calM_i)_{i\in [n]}}^\pi \geq B$.
\end{problem}
Setting $\alpha = 1$ gives the set of all FSC-problems where the memory bound is at least the number of environments.
For certain subclasses of ME-POMDPs, increasing $\alpha$ past a certain threshold drops the complexity of the $\alpha$-parameterized problem from $\ETR$-complete to $\NP$-complete, as we show next.

\paragraph{Blind MDPs}
For ME-BlindMDPs, FSC synthesis is $\NP$-complete when we restrict the FSC size to exceed the number of environments, as we recover the problem with no policy constraints.
\begin{restatable}[]{lemma}{lemMEBlindPOMDPsFSCsAlpha}\label{lem:alpha_one_FSC_blind}
    For $\alpha = 1$, the $\alpha$-FSC value problem for ME-BlindMDPs is $\NP$-complete.
\end{restatable}
\begin{proof}
    Let $n$ be the number of environments in the ME-BlindMDP. 
    By \Cref{prop:prior_universal}, the ME-BlindMDP has an optimal policy that randomizes over $n$ action sequences of length $H$. 
    We can encode this policy in an FSC with $|Q|$ states, since by assumption $|Q| \geq n$, so we can use the memory to index the sequence and use $\gamma_t$ to encode the $n$ action sequences.
    Thus, solving the memory-constrained ME-BlindMDP problem is equivalent to solving the ME-BlindMDP problem without policy constraints.
    The latter problem is in $\NP$ by \Cref{cor:blind_mepompd_in_NP}, and $\NP$-hard as Blind MDPs are $\NP$-hard \cite{DBLP:journals/mor/PapadimitriouT87}.
\end{proof}
\Cref{lem:alpha_one_FSC_blind} raises the problem of identifying the $\alpha$ value at which complexity drops from $\ETR$ to $\NP$.
As we show below, $\alpha = 1$ is tight, in the sense that the FSC policy existence problem is $\ETR$-complete for any smaller value of $\alpha$.
This result follows from the fact that, when $\alpha < 1$, we can encode $\ETR$-hard low-non-negative-rank feasibility problems in the time-dependent FSC problem.
\begin{restatable}[]{theorem}{thmMEBlindPOMDPsFSCs}
    \label{thm:alpha_FSC_complexity}
    For any fixed rational $\alpha \in (0,1)$, the $\alpha$-FSC value problem for ME-BlindMDPs is $\ETR$-complete.
\end{restatable}
\begin{proofsketch}
    We prove this result by linking the existence of time-dependent finite-state controllers to solving certain non-negative-rank-constrained feasibility problems.
    We first show a one-to-one correspondence between time-dependent FSC controllers with $|Q|$ states and inter-time-step joint action distributions with non-negative rank less than $|Q| $ \cite{DBLP:journals/siamjo/Vavasis09}.
    We can then build a ME-POMDP that has a non-negative value if and only if a joint action distribution exists satisfying a set of linear constraints, where each action distribution corresponds to a matrix with low non-negative rank. 
    Finally, we adapt existing results on the complexity of non-negative rank-constrained feasibility \cite{DBLP:journals/corr/abs-2603-19976} to show that deciding whether such a rank-constrained distribution exists is $\ETR$-hard.
    Critically, in adapting this result, we must carefully control the relationship between the memory size $|Q|$ and the number $n$ of environments in the instance we build.
\end{proofsketch}
The above results demonstrate a tradeoff between computation at planning and deployment time, whereby we can increase the memory size to reduce the complexity of the planning problem.

\paragraph{General POMDPs}
For general POMDPs, the $\alpha$-FSC value problem is $\ETR$-hard even for $\alpha=1$, and thus generally harder than for BlindMDPs.
In effect, introducing non-trivial observation functions allows one to construct ME-POMDPs where the observations consume all the agent's memory, so the agent cannot couple its actions across time steps at all.
\begin{restatable}[]{proposition}{propMEPOMDPsFSCsAlpha}
    \label{prop:nFSC_general_MEPOMDP}
     For $\alpha = 1$, the $\alpha$-FSC value problem for ME-POMDPs is $\ETR$-complete.
\end{restatable}
\begin{proofsketch}
    We use the same techniques as in \Cref{prop:mepomdp_blind_timedep}. 
    However, we add a randomly drawn state component to the environment, and we set the rewards so agent must remember this state component to avoid negative reward. 
    This component exhausts the agent's memory so that it must draw actions at other timesteps independently, allowing us to simulate the 2-vs.-1 matrix game.
\end{proofsketch}
It remains an open problem whether there are lower bounds on $|Q|$ in terms of $n$, or other ME-POMDP parameters, such that the $\alpha$-FSC value problem becomes $\NP$-complete.

\section{Conclusion}

We established a range of new complexity results for finite-horizon ME-POMDPs across subclasses of POMDPs.
Our results show that some POMDP subclasses remain tractable when extended to the multi-environment setting. 
Future work will extend complexity characterizations to more subclasses and examine infinite-horizon and logic-based objectives.%

\begin{acks}
Nils Jansen was supported by the ERC Starting Grant DEUCE (101077178).
Marnix Suilen was supported by Research Foundation Flanders (FWO) Junior Postdoctoral Fellowship grant 12ALH27N (IDEAS in POMDPs).
Caleb Probine and Ufuk Topcu were funded in part by the Air Force Office of Scientific Research (AFOSR) under grant number FA9550-22-1-0403, and by the Office of Naval Research (ONR) under grant number N00014-24-1-2797.

\end{acks}

\bibliographystyle{ACM-Reference-Format} 
\bibliography{refs}

\clearpage\newpage
\onecolumn
\appendix
\crefalias{section}{appendix} 
\crefalias{subsection}{appendix} 
\crefalias{subsubsection}{appendix}

\section{Statement on the use of Artificial Intelligence}

\subsection{Expanding literature searches}
We used large language models to supplement search keywords for literature review. In particular, when searching for literature in areas where search terms were unclear, we used generative models to suggest keywords and search strategies that could be useful for finding papers. We then used these keywords in Google Scholar searches to see whether they yielded new papers that added context to our related work. Examples of prompts provided to ChatGPT are as follows.
\begin{enumerate}
    \item `` What are terms for optimization subject to constraints that take the form of ``A must have a non-negative factorization with dimension k'' ''
    \item `` Consider the following setup for a partially observable stochastic game 1. P1 has full observability 2. P1 moves only at the first time step 3. P2 then plays the rest of the game with some partial observability. Such games are naturally a subclass of extensive-form games, partially observable stochastic games (POSGs), and games on graphs. What theoretical complexity results exist for such games, and what might they have been called? ''
\end{enumerate}

LLM-generated text from keyword generation does not directly appear in the paper. Instead, we used the generated keywords in Google Scholar searches to expand the set of papers we examined when contextualizing our work. When generative models suggested papers, we first sourced them via Google Scholar and read them to identify the results they claimed.

\subsection{Identifying folklore results}

We also used language models to identify canonical references for folklore results. For example, it is well-known that many different formalisms for POMDPs that differ slightly are actually equivalent. For example, it is well-known that one can transform a POMDP with stochastic observations into one with deterministic observations. Literature on such folklore results, however, is difficult to find via standard search techniques. While one can easily prove these results, we used ChatGPT to try to identify proofs of such folklore results in the literature so they could be referenced.  An example of a prompt used for this purpose when developing the paper is as follows.
\begin{itemize}
    \item ``Can you find a POMDP paper that gives the polytime reduction from POMDPs with state-action-based observations to state-based observation functions''
\end{itemize}
When ChatGPT suggested such papers, we again confirmed their existence and read them to determine what results they actually contained.

The particular models used to generate the responses include \textsc{GPT-5.6-Luna}, \textsc{GPT-5.5}, and \textsc{GPT-5.5-Mini}.

\section{Additional preliminaries}

For a vector $\mathbf{u}$ we use $\vecat{\mathbf{u}}{i}$ or $\mathbf{u}(i)$ to denote the $i^{\text{th}}$ element, and $\mathbf{u}^\top$ denotes its transpose.
The entry-wise $1$ norm of a matrix $P \in \mathbb{R}^{m \times n}$ is $\lVert P \rVert_{1,1} = \sum_{i \in [m], j \in [n]} |\vecat{P}{i,j}|$, and similarly for a vector
 $\lVert \mathbf{u} \rVert_1 = \sum_{i \in [n]} |\vecat{\mathbf{u}}{i}|$.
We use $\unit{i} \in \RR^n$ to denote the standard basis vector with a one at entry $i$ and zeros elsewhere, and $\mathbf{1}$ is the all-ones vector.
For a pair of matrices $P,Q \in \RR^{m,n}$, $\langle P, Q \rangle$ is the element-wise product $\sum_{i\in [m], j \in [n]} \vecat{P}{i,j}\vecat{Q}{i,j}$.

\paragraph{Paths of length $H$}
Given a horizon $H$, we evaluate the reward on paths of the shape $\rho = s_0,a_0,\dots,s_{H-1},a_{H-1}$.
Although the final state-action pair $(s_{H-1},a_{H-1})$ still induces a distribution on the successor state $s_H$, we ignore this state, as it does not contribute anything to the finite-horizon reward $\sum_{t=0}^{H-1} R(s_t,a_t)$.

\begin{definition}[Belief update]
    Given a belief $b$, an action $a$, and observation $z$, we compute the successor belief $b'$ using the belief update function $b_{upd}\colon \Delta(S) \times A \times Z \to \Delta(S)$, which we define as follows.
    \begin{align*}
    b'(s') &= b_{upd}(b,a,z)(s'),\\
    &= \frac{O(s',z)\sum_{s\in S}b(s)\cdot T(s,a,s')}{\sum_{s'\in S}O(s',z)\sum_{s\in S}b(s)\cdot T(s,a,s')}.
    \end{align*}
\end{definition}

\paragraph{Belief update at $t=0$}
Although we are given an initial belief, we also receive an observation of the initial state.
Upon receiving this observation, we perform a belief update without action:
\begin{align*}
     b_{upd}(b_0,\cdot,z_0)(s) &= \frac{O(s,z_0)b_0(s)}{\sum_{s\in S}O(s,z_0)b_0(s)}.
\end{align*}

\paragraph{Additional policy types}
We next list some additional types of policies that we use in proofs in this appendix.
\begin{definition}[Belief-based policies]
    A time-dependent belief-based policy maps tuples of distributions over states and time to distributions over actions. Formally, a belief-based policy is a function
    \[
        \pi \colon \dist{S} \times \{0,\dots,H-1\} \to \dist{A}.
    \]
\end{definition}
\begin{definition}[Observation-based policies]
    An observation-based policy maps sequences of observations to distributions over actions, i.e.
    \[
        \pi \colon \bigcup_{t=1}^H Z^t \to \dist{A}.
    \]
\end{definition}

\subsection{Supporting lemmas}

We will repeatedly need to evaluate time-dependent Markovian policies as a subroutine in the algorithms we provide, and as such, we provide the following lemma on the evaluation of deterministic time-dependent Markovian policies in POMDPs.
In particular, for an arbitrary POMDP $\calM$, we can evaluate a deterministic time-dependent Markovian policy $\pi$ in polynomial time, and furthermore, the representation of the expected reward of this policy is polynomial in the size of the POMDP.
We remark that existing work details polynomial-time procedures for policy evaluation in POMDPs when probabilities are represented with binary fractions, and observation functions are deterministic \cite{DBLP:journals/jacm/MundhenkGLA00}. 
For completeness, we provide the following lemma, which handles stochastic observation functions and explicitly bounds the reward's representation size.
\begin{lemma}
    \label{lem:dtd_pol_eval_rep_and_time}
    There exists an algorithm that takes a POMDP $\calM = (S,A,Z,T,O,R,b,H)$ and a deterministic time-dependent Markovian policy $(\pi_t)_{t\in\{0,\ldots,H-1\}}$ and evaluates the expected reward $\mathbb{E}_{\calM}^\pi\left[\sum_{t=0}^{H-1}R(s_t,a_t)\right]$ in polynomial time in the representation size $\calN$ of the POMDP $\calM$. 

    Furthermore, the bit lengths of the integers required to represent $\mathbb{E}_{\calM}^\pi\left[\sum_{t=0}^{H-1}R(s_t,a_t)\right]$ are bounded by a polynomial in $\calN$. 
\end{lemma}
\begin{proof}
    Before proceeding, we argue that we can process any POMDP instance so that all rationals share a common denominator. We also argue that this procedure runs in polynomial time in the POMDP instance size.
    In particular, we can argue in a similar manner to the proof of Lemma $2$ in \cite{DBLP:journals/corr/abs-2605-07537}, and following the notation in \cite{DBLP:journals/corr/abs-2605-07537}, we use $\calN$ to denote the representation size of the ME-POMDP problem. 
    
    We enumerate all rational numbers $(q_k)_{k\in [K]}$ in the POMDP's representation,
    \begin{equation}
        \underbrace{\nicefrac{x_1}{y_1}}_{q_1},\ldots, \underbrace{\nicefrac{x_K}{y_K}}_{q_K},
    \end{equation}
    where $K \leq \mathcal{N}$ is the number of rational numbers in the POMDP's encoding which are used to represent transition functions, observation functions, rewards and initial state distributions. 
    Clearly, we can rewrite the rational number $q_l$ as
    \begin{equation}
        q_l = \frac{x_l\prod_{l' \in [K]: l' \neq l} y_{l'}}{\prod_{l' \in [K]} y_{l'}} = \frac{X_l}{C},
    \end{equation}
    where $X_l \coloneq x_l\prod_{l' \in [K]: l' \neq l} y_{l'}$, and $C \coloneq \prod_{l' \in [K]} y_{l'}$.
    Furthermore, if the POMDP has an $\calN$-bit representation, then we have $|x_l| \leq 2^\calN$ and $|y_l| \leq 2^\calN$ for all $l \in [K]$, which implies that $|X_l| \leq 2^{\calN^2}$ and $|C| \leq 2^{\calN^2}$.
    Computing each $X_l$ along with the constant $C$ comprises multiplying $K \leq \calN$ integers each with bit length bounded above by $\calN$.
    Thus, we can compute this common-denominator representation in polynomial time in the POMDP size.

    To evaluate the memoryless policy $(\pi_t)_{t \in \{0,\ldots,H-1\}}$, we first evaluate the following expressions which encode the probability of being at each state at each time. 
    \begin{align*}
        p_0(s) = b(s) \quad & \forall s \in S \\ 
        p_{t+1}(s) = \sum_{s' \in S} p_t(s')   \sum_{z \in Z}T(s',\pi_t(z))(s) O(s')(z) \quad & \forall s \in S, t \in \{0,\ldots, H-2\}.
    \end{align*}
    We can then argue, by induction, that the probabilities $(p_t(s))_{s\in S}$ have $C^{2t+1}$ as a denominator.
    Trivially, we can represent the probability $p_0(s)$ as a rational with denominator $C$ for each $s \in S$. 
    Then, assume that $p_t(s) = \frac{w_{s}^{(t)}}{C^{2t+1}}$ for all $s \in S$
    , where $w_{s}^{(t)} \in \{0,\ldots, C^{2t+1}\}$ for each $s$.
    Then, for each $s$, we have
    \begin{equation}
        p_{t+1}(s) = \sum_{s' \in S} \frac{w_{s'}^{(t)}}{C^{2t+1}} \sum_{z \in Z} \frac{u_{s'\pi_t(z)s}}{C} \frac{v_{s'z}}{C}
        = \frac{w_{s}^{(t+1)}}{C^{2(t+1) + 1}}
    \end{equation}
    for some $w_{s}^{(t+1)} \in \{0,\ldots, C^{2(t+1) + 1}\}$, where $T(s',a,s) = u_{s'as}/C$ for each $s',s \in S$, and $a\in A$, and $O(s',z) = v_{s'z}/C$ for each $s' \in S,z \in Z$.

    We can also evaluate $(p_{t}(s))_{t \in \{0,\ldots,H_1\}}$ in polynomial time. 
    First, note that $C^{2(t+1) + 1} \leq 2^{\calN^2 (2H + 1)} \leq 2^{3\calN^3}$, which implies that each integer $w_{s}^{(t)}$ has at most $3\calN^3$ bits.
    To compute each of the $|S|(H-1)$ values $w_{s}^{(t+1)}$, we need to compute $\sum_{s' \in S, z \in Z} w_{s'}^{(t)} u_{s'\pi_t(a)s} v_{s' z}$, which can clearly be done in polynomial time, since each $w_{s'}^{(t)}$ has at most $3\calN^3$ bits, and each of 
    the $u$ and $v$ values has at most $\calN^2$ bits.
    
    The expected reward for the memoryless policy is
    \begin{equation}
        \mathbb{E}_{\calM}^{\pi}\left[\sum_{t=0}^{H-1} R(s_t,a_t)\right] = \sum_{t = 0}^{H-1} \sum_{s \in S, z \in Z} p_{t}(s) O(s) (z) R(s,\pi(z)) = \sum_{t = 0}^{H-1} \sum_{s \in S, z \in Z} 
        \frac{C^{2(H-t)} w_{s}^{(t)}}{C^{2H+1}} \frac{v_{sz}}{C} \frac{\kappa_{s\pi_t(z)}}{C},
    \end{equation}
    where we assume $R(s,a) = \nicefrac{\kappa(s,a)}{C}$ for each $s\in S$ and $a \in A$.
    Again, we can compute the sum in polynomial time, as all of the integers in the above expression have representation length bounded by a fixed polynomial in $\calN$.

    We can also observe that the expected reward is equal to $\nicefrac{\xi}{(C^{2H+3})}$ for some integer $\xi$, which we now bound.
    In particular, since the expected reward is bounded above by $H\max_{s,a} |R(s,a)|$, we have
    \begin{equation}
        \left| \frac{\xi}{(C^{2H+3})} \right| \leq H \max_{s \in S,a \in A} |R(s,a)| \leq H \frac{\max_{l \in K}|X_l|}{C},
    \end{equation}
    which implies that 
    \begin{equation*}
        \xi \leq H C^{2H + 2} \max_{l \in [K]} |X_l| \leq H 2^{\calN^2 (2H + 2) } 2^{\calN^2} \leq H 2^{5H \calN^2} \leq 2^{6\calN^3} .
    \end{equation*}
    Thus, from the bounds on $C$ and $\xi$, we deduce the polynomial bit-length condition for the policy's expected reward.
\end{proof}

Next, we discuss a correspondence between FSCs on POMDPs and Markovian policies on a product POMDP.
This correspondence is well known, and mirrors the transformation from \emph{time-independent} FSC-synthesis to memoryless policy synthesis in a product POMDP \cite{junges2018finite}. 
We give a proof here for completeness in handling the time-dependent details.

\begin{lemma}\label{lem:fsc_product_pomdp}
    Fix a POMDP $\calM = (S,A,Z,T,O,R,b,H)$, and a memory size $|Q|$. 
    In deterministic polynomial time, we can construct an extended POMDP $\calM'$ such that there exists a correspondence between
    \begin{enumerate}
        \item the set of time-dependent FSCs with $|Q|$ states in the POMDP $\calM$, and
        \item the set of time-dependent Markovian policies in an expanded POMDP $\calM'$,
    \end{enumerate}
    such that for every time-dependent FSC with $|Q|$ states in the POMDP $\calM$, there exists a time-dependent Markovian policy in an expanded POMDP $\calM'$that achieves the same reward, and vice versa.
    Furthermore, this policy translation procedure can be done in polynomial time in the policy's representation size.

    Furthermore, if the time-dependent Markovian policy $\pi'$ in the expanded POMDP $\calM'$ is {deterministic}, then there exists a deterministic time-dependent FSC with $|Q|$ states in $\calM$ with the same reward.
\end{lemma}

\begin{proof}
    
    Define a POMDP $\calM'$ with
    \begin{equation}
        \calM' = (\underbrace{\{s_\bot\} \cup S \times Q}_{S'}, \underbrace{A \times Q}_{A'}, \underbrace{Z \times Q}_{Z'}, T', O', R', b', H+1),
    \end{equation}
    where we denote the augmented state space as $S'$, the augmented action space as $A'$ and the augmented observation space as $Z'$.
    The state $s_\bot$ is a dummy state, included to allow a random initial FSC state.
    The transition function encodes the original POMDP's transition function and the memory update.
    \begin{align*}
        & T'(s_\bot, (a,q)) = b \times \delta_q & \forall (a,q) \in A \times Q \\
        & T'((s,q), (a,q')) = T(s,a) \times \delta_{q'} & \forall (a,q') \in A \times Q, (s,q) \in S\times Q.
    \end{align*}
    The observation function reveals the memory state and encodes the original observation function.
    \begin{align*}
        & O'(s_\bot) = \delta_{(z^{(*)},q^{(*)})} \\
        & O'((s,q)) = O(s) \times \delta_q \quad \forall s \in S, q \in Q.
    \end{align*}
    In this observation function, $z^{(*)}$ and $q^{(*)}$ are an arbitrary observation-memory-state pair.
    The reward function encodes the original reward function outside the dummy state, where the reward is zero.
    \begin{align*}
        & R'(s_\bot, (a,q)) = 0 & \forall (a,q) \in A \times Q \\
        & R'((s,q), (a,q')) = R(s,a) & \forall (a,q') \in A \times Q, (s,q) \in S\times Q.
    \end{align*}
    Finally, the initial state distribution places all mass on the dummy state, i.e., $b' = \delta_{s_\bot}$.

    Clearly, we can construct $\calM'$ from $\calM$ in polynomial time.
    
    Any FSC $(Q, \iota, (\gamma_t)_{t \in \{0,\ldots, H-1\}}, (\eta_t)_{t \in \{0,\ldots, H-1\}})$ for the original POMDP corresponds to a time-dependent Markovian policy in $\calM'$ which attains the same reward.
    Indeed, for each $t \in \{0,\ldots, H-1\}$, let $\xi_t : Q \times Z \rightarrow \dist{A \times Q}$ denote the function such that
    \begin{equation}
        \xi_t((q,z))((a,q')) = \gamma_t(q,z)(a) \eta_t(q,z,a)(q') \quad \forall q,q' \in Q, z \in Z, a \in A.
    \end{equation}
    That is, $\xi_t$ encodes the joint distribution on actions and successor memory states given an original memory state and an observation. 

    Clearly, we can implement $\xi_t$ as a time-dependent Markovian policy $\pi'$ in the expanded POMDP $\calM'$ by setting
    \begin{align*}
        & \pi'_0((z^{(*)},q^{(*)}))((a^{(*)},q)) = \iota(q)  \quad & \forall q \in Q, \\
        & \pi'_{t}((z,q))((a,q')) = \xi_{t-1}((q,z))((a,q')) & \forall q,q' \in Q, z \in Z, a \in A, t \in \{1,\ldots,H\},
    \end{align*}
    where $a^{(*)} \in A$ is an arbitrary fixed action.
    Now, if $h = q_0, s_0, z_0, a_0, \ldots, q_{H-1}, s_{H-1}, z_{H-1}, a_{H-1} \in (Q \times S \times Z \times A)^{H}$ is some trajectory in the original POMDP, then the set of trajectories
    \begin{equation}
        \calH_h := \{s_\bot, (z^{(*)},q^{(*)}), (a^{(*)},q_0), (s_0,q_0), (z_0, q_0), (a_0,q_1),\ldots, (s_{H-1},q_{H-1}), (z_{H-1}, q_{H-1}), (a_{H-1}, q_{H}) | q_{H+1} \in Q\} \subset (S' \times Z' \times A')^{H+1}
    \end{equation}
    satisfies $\mathbb{P}^\pi_\calM[h] = \mathbb{P}^{\pi'}_{\calM'}[\calH_h]$.
    As the two POMDPs give the same reward for trajectories in $h$ and $\calH_h$ respectively, we can deduce that $\pi'$ obtains the same reward in $\calM'$ as $\pi$ does in $\calM$. 

    Conversely, let $\pi'$ be some time-dependent Markovian policy in $\calM'$. We first set $\iota$ such that $\iota(q) = \sum_{a \in A}\pi'_0((z^{(*)}, q^{(*)}))((q,a))$ for all $q \in Q$. 
    Then, we set, for each $t \in \{0,\ldots, H-1\}$, $q \in Q$, $z \in Z$, and $a \in A$,
    \begin{equation}
        \gamma_t(q,z)(a) = \sum_{q' \in Q} \pi'_{t+1}((z,q))((q',a)).
    \end{equation}
    Next, for each $t \in \{0,\ldots, H-1\}$,$q,q' \in Q$, $z \in Z$, and $a \in A$, we set
    \begin{equation}
        \eta_t(q,z,a)(q') = 
        \begin{cases}
            \frac{\pi'_{t+1}((z,q))((q',a))}{\gamma_t(q,z)(a)} & \text{ if $\gamma_t(q,z)(a) > 0$} \\
            \delta_{q^{(*)}} & \text{ otherwise .}
        \end{cases}
    \end{equation}
    In the above expression $q^{(*)}$ is some arbitrary fixed memory-state in $Q$.
    Clearly, we have 
    \begin{equation}
        \pi'_{t+1}((z,q))((q',a)) = \eta_t(q,z,a)(q')\gamma_t(q,z)(a) \quad \forall q,q' \in Q, a \in A, z \in Z, t \in \{0,\ldots, H-1\},
    \end{equation}
    and thus, by the same logic as above, we can deduce that if $h = q_0, s_0, z_0, a_0, \ldots, q_H, s_H, z_H, a_H \in (Q \times S \times Z \times A)^{H}$ is some trajectory in the \emph{original} POMDP $\calM$, then the following set of trajectories in the expanded POMDP $\calM'$ has the same probability and reward under $\pi'$.
    \begin{equation}
        \calH_h := \{s_\bot, (z^{(*)},q^{(*)}), (\tilde{a},q_0), (s_0,q_0), (z_0, q_0), (a_0,q_1),\ldots, (s_H,q_H), (z_H, q_H), (a_H, q_{H+1}) | q_{H+1} \in Q, \tilde{a} \in A\} \subset (S' \times Z' \times A')^{H+1}.
    \end{equation}
    We can then deduce that the policy $\pi$ defined by $(\iota, (\gamma_t)_{t \in \{0,\ldots,H-1\}}, (\eta_t)_{t \in \{0,\ldots,H-1\}})$ attains the same reward as $\pi'$.

    We have now established the correspondence between the two policy classes, and we remark that each direction of the translation between the two policy classes can be done in polynomial time in the size of the policy.

    Furthermore, examining the proof, we see that if a time-dependent Markovian policy $\pi'$ in $\calM'$ is deterministic, we can construct an FSC-based policy in $\calM$ with the same reward that is also deterministic.
\end{proof}

\clearpage\newpage
\section{Figures}
\begin{figure}[h]
    \scalebox{0.78}{\begin{tikzpicture}[state/.append style={shape = circle}, >=stealth, bobbel/.style={minimum size=1mm,inner sep=0pt,fill=black,circle}]
    \fill[lightgray!20] (-0.2,1.4) rectangle (5,-4.7);
    \fill[lightgray!5] (5,1.4) rectangle (10.2,-4.7);
    \fill[lightgray!20] (10.2,1.4) rectangle (15.4,-4.7);
    \fill[lightgray!5] (15.4,1.4) rectangle (21,-4.7);
    \draw[fill=color-normal,opacity=0.05] (-0.2,-4.7) to[out=90,in=-90] (-0.2,-1.6) to[out=0,in=180] (4.3,-1.6) to[out=0,in=100] (5,-3) to[out=-80,in=175] (6.5,-4.4) to[out=-5,in=190] (7.2,-4.4) to[out=10,in=185] (8.8,-3.7) to[out=5,in=170] (9.3,-3.7) to[out=-10,in=180] (10.9,-4.4) to[out=0,in=180] (13,-4.4) to[out=0,in=-105] (15.2,-3.5) to[out=75,in=180] (16.2,-2.65) to[out=0,in=180] (21,-2.65) to[out=-90,in=90] (21,-4.7);
    \draw[fill = color-normal,opacity=0.05] (-0.2,-4.7) to[out=90,in=-90] (-0.2,-3.65) to[out=0,in=180] (1.5,-3.65) to[out=0,in=120] (2.4,-4.2) to[out=-60,in=180] (3.3,-4.6) to[out=0,in=180] (21,-4.6) to[out=-90,in=90] (21,-4.7);
    \node[] (pomdp1) at (2.4,0) {POMDP};
    \node[] (rev_pomdp1) at ($(pomdp1) + (-1.3,-1)$) {RevPOMDP};
    \node[] (semi_mdp1a) at ($(rev_pomdp1) + (0,-1)$) {SOMDP$_{\eta_s \in (0,1]}$};
    \node[] (inter_mdp1a) at ($(semi_mdp1a) + (0,-1)$) {IOMDP$_{\lambda \in (0,1]}$};
    \node[] (semi_mdp1b) at ($(rev_pomdp1) + (2.7,-1)$) {SOMDP$_{\eta_s \in [0,1]}$};
    \node[] (inter_mdp1b) at ($(semi_mdp1b) + (0,-1)$) {IOMDP$_{\lambda \in [0,1]}$};
    \node[] (mdp1) at ($(inter_mdp1a) + (0,-1)$) {MDP};
    \node[] (blind_mdp1) at ($(inter_mdp1b) + (0,-1)$) {BlindMDP};
    \draw[opacity=0] (rev_pomdp1) to node[sloped,opacity=1] {$\subseteq$} (pomdp1);
    \draw[opacity=0] (semi_mdp1a) to node[sloped,opacity=1] {$\subseteq$} (rev_pomdp1);
    \draw[opacity=0] (inter_mdp1a) to node[sloped,opacity=1] {$\subseteq$} (semi_mdp1a);
    \draw[opacity=0] (mdp1) to node[sloped,opacity=1] {$\subseteq$} (inter_mdp1a);
    \draw[opacity=0] (semi_mdp1b) to node[sloped,opacity=1,xscale=-1] {$\subseteq$} (pomdp1);
    \draw[opacity=0] (inter_mdp1b) to node[sloped,opacity=1] {$\subseteq$} (semi_mdp1b);
    \draw[opacity=0] (inter_mdp1a) to node[sloped,opacity=1] {$\subseteq$} (inter_mdp1b);
    \draw[opacity=0] (semi_mdp1a) to node[sloped,opacity=1] {$\subseteq$} (semi_mdp1b);
    \draw[opacity=0] (blind_mdp1) to node[sloped,opacity=1] {$\subseteq$} (inter_mdp1b);
    \node[] (pomdp2) at (7.6,0) {POMDP};
    \node[] (rev_pomdp2) at ($(pomdp2) + (-1.4,-1)$) {RevPOMDP};
    \node[] (semi_mdp2a) at ($(rev_pomdp2) + (0,-1)$) {SOMDP$_{\eta_s \in (0,1]}$};
    \node[] (inter_mdp2a) at ($(semi_mdp2a) + (0,-1)$) {IOMDP$_{\lambda \in (0,1]}$};
    \node[] (semi_mdp2b) at ($(rev_pomdp2) + (2.8,-1)$) {SOMDP$_{\eta_s \in [0,1]}$};
    \node[] (inter_mdp2b) at ($(semi_mdp2b) + (0,-1)$) {IOMDP$_{\lambda \in [0,1]}$};
    \node[] (mdp2) at ($(inter_mdp2a) + (0,-1)$) {MDP};
    \node[] (blind_mdp2) at ($(inter_mdp2b) + (0,-1)$) {BlindMDP};
    \draw[opacity=0] (rev_pomdp2) to node[sloped,opacity=1] {$\subseteq$} (pomdp2);
    \draw[opacity=0] (semi_mdp2a) to node[sloped,opacity=1] {$\subseteq$} (rev_pomdp2);
    \draw[opacity=0] (inter_mdp2a) to node[sloped,opacity=1] {$\subseteq$} (semi_mdp2a);
    \draw[opacity=0] (mdp2) to node[sloped,opacity=1] {$\subseteq$} (inter_mdp2a);
    \draw[opacity=0] (semi_mdp2b) to node[sloped,opacity=1,xscale=-1] {$\subseteq$} (pomdp2);
    \draw[opacity=0] (inter_mdp2b) to node[sloped,opacity=1] {$\subseteq$} (semi_mdp2b);
    \draw[opacity=0] (inter_mdp2a) to node[sloped,opacity=1] {$\subseteq$} (inter_mdp2b);
    \draw[opacity=0] (semi_mdp2a) to node[sloped,opacity=1] {$\subseteq$} (semi_mdp2b);
    \draw[opacity=0] (blind_mdp2) to node[sloped,opacity=1] {$\subseteq$} (inter_mdp2b);
    \node[] (pomdp3) at (12.8,0) {POMDP};
    \node[] (rev_pomdp3) at ($(pomdp3) + (-1.4,-1)$) {RevPOMDP};
    \node[] (semi_mdp3a) at ($(rev_pomdp3) + (0,-1)$) {SOMDP$_{\eta_s \in (0,1]}$};
    \node[] (inter_mdp3a) at ($(semi_mdp3a) + (0,-1)$) {IOMDP$_{\lambda \in (0,1]}$};
    \node[] (semi_mdp3b) at ($(rev_pomdp3) + (2.8,-1)$) {SOMDP$_{\eta_s \in [0,1]}$};
    \node[] (inter_mdp3b) at ($(semi_mdp3b) + (0,-1)$) {IOMDP$_{\lambda \in [0,1]}$};
    \node[] (mdp3) at ($(inter_mdp3a) + (0,-1)$) {MDP};
    \draw[opacity=0] (rev_pomdp3) to node[sloped,opacity=1] {$\subseteq$} (pomdp3);
    \draw[opacity=0] (semi_mdp3a) to node[sloped,opacity=1] {$\subseteq$} (rev_pomdp3);
    \draw[opacity=0] (inter_mdp3a) to node[sloped,opacity=1] {$\subseteq$} (semi_mdp3a);
    \draw[opacity=0] (mdp3) to node[sloped,opacity=1] {$\subseteq$} (inter_mdp3a);
    \draw[opacity=0] (semi_mdp3b) to node[sloped,opacity=1,xscale=-1] {$\subseteq$} (pomdp3);
    \draw[opacity=0] (inter_mdp3b) to node[sloped,opacity=1] {$\subseteq$} (semi_mdp3b);
    \draw[opacity=0] (inter_mdp3a) to node[sloped,opacity=1] {$\subseteq$} (inter_mdp3b);
    \draw[opacity=0] (semi_mdp3a) to node[sloped,opacity=1] {$\subseteq$} (semi_mdp3b);
    \node[] (pomdp4) at (18.2,0) {POMDP};
    \node[] (rev_pomdp4) at ($(pomdp4) + (-1.4,-1)$) {RevPOMDP};
    \node[] (semi_mdp4a) at ($(rev_pomdp4) + (0,-1)$) {SOMDP$_{\eta_s \in (0,1]}$};
    \node[] (inter_mdp4a) at ($(semi_mdp4a) + (0,-1)$) {IOMDP$_{\lambda \in (0,1]}$};
    \node[] (semi_mdp4b) at ($(rev_pomdp4) + (2.8,-1)$) {SOMDP$_{\eta_s \in [0,1]}$};
    \node[] (inter_mdp4b) at ($(semi_mdp4b) + (0,-1)$) {IOMDP$_{\lambda \in [0,1]}$};
    \draw[opacity=0] (rev_pomdp4) to node[sloped,opacity=1] {$\subseteq$} (pomdp4);
    \draw[opacity=0] (semi_mdp4a) to node[sloped,opacity=1] {$\subseteq$} (rev_pomdp4);
    \draw[opacity=0] (inter_mdp4a) to node[sloped,opacity=1] {$\subseteq$} (semi_mdp4a);
    \draw[opacity=0] (semi_mdp4b) to node[sloped,opacity=1,xscale=-1] {$\subseteq$} (pomdp4);
    \draw[opacity=0] (inter_mdp4b) to node[sloped,opacity=1] {$\subseteq$} (semi_mdp4b);
    \draw[opacity=0] (inter_mdp4a) to node[sloped,opacity=1] {$\subseteq$} (inter_mdp4b);
    \draw[opacity=0] (semi_mdp4a) to node[sloped,opacity=1] {$\subseteq$} (semi_mdp4b);
    \fill[lightgray!40] (-0.2,1.4) rectangle (21,0.6);
    \node[] (models1) at ($(pomdp1) + (0,0.97)$) {Single-environment models};
    \node[] (models2) at ($(pomdp2) + (0,1)$) {Multi-environment models};
    \node[] (models3) at ($(pomdp3) + (0,1)$) {Multi-observation models};
    \node[] (models3) at ($(pomdp4) + (0,0.97)$) {Multi-observation models, fixed support};
    \draw[-,loosely dashed] (-0.2,0.6) to (21,0.6);
    \draw[-,loosely dashed] (5,1.4) to (5,-4.7);
    \draw[-,loosely dashed] (10.2,1.4) to (10.2,-4.7);
    \draw[-,loosely dashed] (15.4,1.4) to (15.4,-4.7);
    \draw[color-normal] (-1.7,-1.6) to[out=0,in=180] (4.3,-1.6) to[out=0,in=100] (5,-3) to[out=-80,in=175] (6.5,-4.4) to[out=-5,in=190] (7.2,-4.4) to[out=10,in=185] (8.8,-3.7) to[out=5,in=170] (9.3,-3.7) to[out=-10,in=180] (10.9,-4.4) to[out=0,in=180] (13,-4.4) to[out=0,in=-105] (15.2,-3.5) to[out=75,in=180] (16.2,-2.65) to[out=0,in=180] (21,-2.65);
    \node[color-normal, draw=none,anchor=west] at (-1.8,-1.48) {\footnotesize{PSPACE-hard}};
    \node[color-normal, draw=none,anchor=west] at (-1.8,-1.72) {\footnotesize{in NP}};
    \draw[color-normal] (-1.175,-3.65) to[out=0,in=180] (1.5,-3.65) to[out=0,in=120] (2.4,-4.2) to[out=-60,in=180] (3.3,-4.6) to[out=0,in=180] (21,-4.6);
    \node[color-normal, draw=none,anchor=west] at (-1.275,-3.53) {\footnotesize{NP-hard}};
    \node[color-normal, draw=none,anchor=west] at (-1.275,-3.77) {\footnotesize{in P}};
\end{tikzpicture}}
    \caption{Overview of the complexity results no restrictions}
    \label{fig:complexity_overview_no_res}
\end{figure}
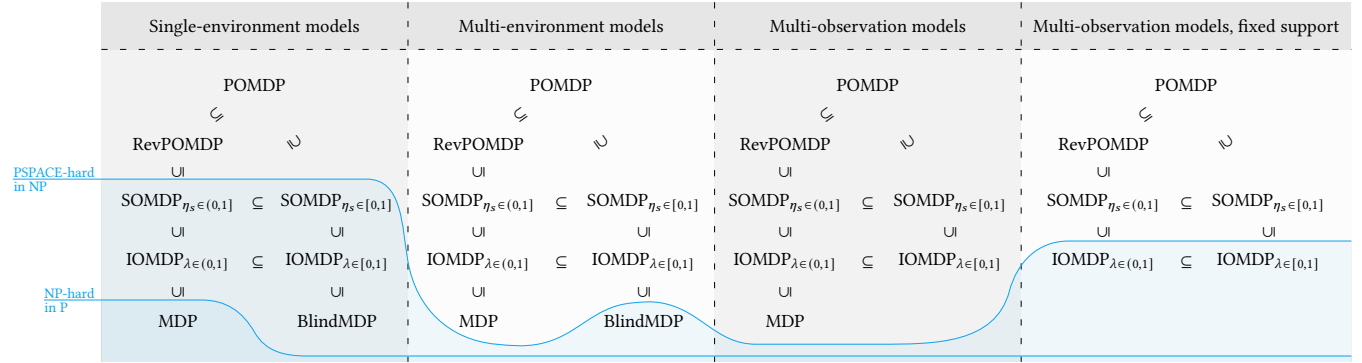

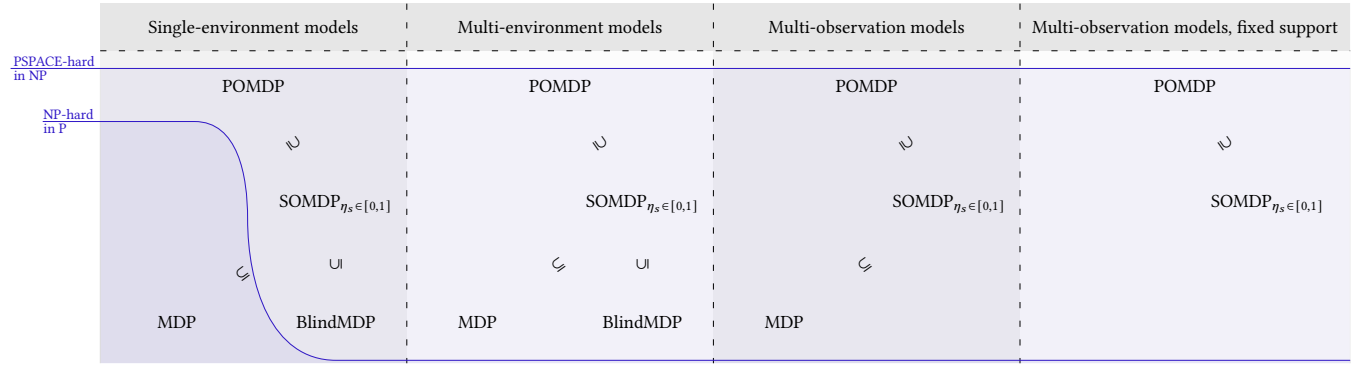
\begin{figure}[h]
    \scalebox{0.78}{\begin{tikzpicture}[state/.append style={shape = circle}, >=stealth, bobbel/.style={minimum size=1mm,inner sep=0pt,fill=black,circle}]
    \fill[lightgray!20] (-0.2,1.4) rectangle (5,-4.7);
    \fill[lightgray!5] (5,1.4) rectangle (10.2,-4.7);
    \fill[lightgray!20] (10.2,1.4) rectangle (15.4,-4.7);
    \fill[lightgray!5] (15.4,1.4) rectangle (21,-4.7);
    \draw[fill=color-deter,opacity=0.05] (-0.2,-4.7) to[out=90,in=-90] (-0.2,0.3) to (21,0.3) to[out=-90,in=90] (21,-4.7);
    \draw[fill=color-deter,opacity=0.05] (-0.2,-4.7) to[out=90,in=-90] (-0.2,-0.6) to[out=0,in=180] (1.4,-0.6) to[out=0,in=90] (2.3,-2.25) to[out=-90,in=180] (3.8,-4.65) to[out=0,in=180] (21,-4.65) to[out=-90,in=90] (21,-4.7);
    \node[] (pomdp1) at (2.4,0) {POMDP};
    \node[opacity=0] (rev_pomdp1) at ($(pomdp1) + (-1.3,-1)$) {RevPOMDP};
    \node[opacity=0] (semi_mdp1a) at ($(rev_pomdp1) + (0,-1)$) {SOMDP$_{\eta_s \in (0,1]}$};
    \node[opacity=0] (inter_mdp1a) at ($(semi_mdp1a) + (0,-1)$) {IOMDP$_{\lambda \in (0,1]}$};
    \node[] (semi_mdp1b) at ($(rev_pomdp1) + (2.7,-1)$) {SOMDP$_{\eta_s \in [0,1]}$};
    \node[opacity=0] (inter_mdp1b) at ($(semi_mdp1b) + (0,-1)$) {IOMDP$_{\lambda \in [0,1]}$};
    \node[] (mdp1) at ($(inter_mdp1a) + (0,-1)$) {MDP};
    \node[] (blind_mdp1) at ($(inter_mdp1b) + (0,-1)$) {BlindMDP};
    \draw[opacity=0] (mdp1) to node[sloped,opacity=1,pos=0.4] {$\subseteq$} (semi_mdp1b);
    \draw[opacity=0] (semi_mdp1b) to node[sloped,opacity=1,xscale=-1] {$\subseteq$} (pomdp1);
    \draw[opacity=0] (blind_mdp1) to node[sloped,opacity=1] {$\subseteq$} (semi_mdp1b);
    \node[] (pomdp2) at (7.6,0) {POMDP};
    \node[opacity=0] (rev_pomdp2) at ($(pomdp2) + (-1.4,-1)$) {RevPOMDP};
    \node[opacity=0] (semi_mdp2a) at ($(rev_pomdp2) + (0,-1)$) {SOMDP$_{\eta_s \in (0,1]}$};
    \node[opacity=0] (inter_mdp2a) at ($(semi_mdp2a) + (0,-1)$) {IOMDP$_{\lambda \in (0,1]}$};
    \node[] (semi_mdp2b) at ($(rev_pomdp2) + (2.8,-1)$) {SOMDP$_{\eta_s \in [0,1]}$};
    \node[opacity=0] (inter_mdp2b) at ($(semi_mdp2b) + (0,-1)$) {IOMDP$_{\lambda \in [0,1]}$};
    \node[] (mdp2) at ($(inter_mdp2a) + (0,-1)$) {MDP};
    \node[] (blind_mdp2) at ($(inter_mdp2b) + (0,-1)$) {BlindMDP};
    \draw[opacity=0] (mdp2) to node[sloped,opacity=1] {$\subseteq$} (semi_mdp2b);
    \draw[opacity=0] (semi_mdp2b) to node[sloped,opacity=1,xscale=-1] {$\subseteq$} (pomdp2);
    \draw[opacity=0] (blind_mdp2) to node[sloped,opacity=1] {$\subseteq$} (semi_mdp2b);
    \node[] (pomdp3) at (12.8,0) {POMDP};
    \node[opacity=0] (rev_pomdp3) at ($(pomdp3) + (-1.4,-1)$) {RevPOMDP};
    \node[opacity=0] (semi_mdp3a) at ($(rev_pomdp3) + (0,-1)$) {SOMDP$_{\eta_s \in (0,1]}$};
    \node[opacity=0] (inter_mdp3a) at ($(semi_mdp3a) + (0,-1)$) {IOMDP$_{\lambda \in (0,1]}$};
    \node[] (semi_mdp3b) at ($(rev_pomdp3) + (2.8,-1)$) {SOMDP$_{\eta_s \in [0,1]}$};
    \node[opacity=0] (inter_mdp3b) at ($(semi_mdp3b) + (0,-1)$) {IOMDP$_{\lambda \in [0,1]}$};
    \node[] (mdp3) at ($(inter_mdp3a) + (0,-1)$) {MDP};
    \draw[opacity=0] (mdp3) to node[sloped,opacity=1] {$\subseteq$} (semi_mdp3b);
    \draw[opacity=0] (semi_mdp3b) to node[sloped,opacity=1,xscale=-1] {$\subseteq$} (pomdp3);
    \node[] (pomdp4) at (18.2,0) {POMDP};
    \node[opacity=0] (rev_pomdp4) at ($(pomdp4) + (-1.4,-1)$) {RevPOMDP};
    \node[opacity=0] (semi_mdp4a) at ($(rev_pomdp4) + (0,-1)$) {SOMDP$_{\eta_s \in (0,1]}$};
    \node[opacity=0] (inter_mdp4a) at ($(semi_mdp4a) + (0,-1)$) {IOMDP$_{\lambda \in (0,1]}$};
    \node[] (semi_mdp4b) at ($(rev_pomdp4) + (2.8,-1)$) {SOMDP$_{\eta_s \in [0,1]}$};
    \node[opacity=0] (inter_mdp4b) at ($(semi_mdp4b) + (0,-1)$) {IOMDP$_{\lambda \in [0,1]}$};
    \draw[opacity=0] (semi_mdp4b) to node[sloped,opacity=1,xscale=-1] {$\subseteq$} (pomdp4);
    \fill[lightgray!40] (-0.2,1.4) rectangle (21,0.6);
    \node[] (models1) at ($(pomdp1) + (0,0.97)$) {Single-environment models};
    \node[] (models2) at ($(pomdp2) + (0,1)$) {Multi-environment models};
    \node[] (models3) at ($(pomdp3) + (0,1)$) {Multi-observation models};
    \node[] (models3) at ($(pomdp4) + (0,0.97)$) {Multi-observation models, fixed support};
    \draw[-,loosely dashed] (-0.2,0.6) to (21,0.6);
    \draw[-,loosely dashed] (5,1.4) to (5,-4.7);
    \draw[-,loosely dashed] (10.2,1.4) to (10.2,-4.7);
    \draw[-,loosely dashed] (15.4,1.4) to (15.4,-4.7);
    \draw[color-deter] (-1.72,0.3) to (21,0.3);
    \node[color-deter, draw=none,anchor=west] at (-1.78,0.42) {\footnotesize{PSPACE-hard}};
    \node[color-deter, draw=none,anchor=west] at (-1.78,0.18) {\footnotesize{in NP}};
    \draw[color-deter] (-1.175,-0.6) to[out=0,in=180] (1.4,-0.6) to[out=0,in=90] (2.3,-2.25) to[out=-90,in=180] (3.8,-4.65) to[out=0,in=180] (21,-4.65);
    \node[color-deter, draw=none,anchor=west] at (-1.275,-0.48) {\footnotesize{NP-hard}};
    \node[color-deter, draw=none,anchor=west] at (-1.275,-0.72) {\footnotesize{in P}};
\end{tikzpicture}}
    \caption{Overview of the complexity results determinsitic models}
    \label{fig:complexity_overview_deter}
\end{figure}

\begin{figure}[h]
    \strut\hfill
    \scalebox{0.78}{\begin{tikzpicture}[state/.append style={shape = circle}, >=stealth, bobbel/.style={minimum size=1mm,inner sep=0pt,fill=black,circle}]
    \fill[lightgray!20] (-0.2,1.4) rectangle (5,-4.7);
    \fill[lightgray!5] (5,1.4) rectangle (10.2,-4.7);
    \fill[lightgray!20] (10.2,1.4) rectangle (15.4,-4.7);
    \fill[lightgray!5] (15.4,1.4) rectangle (21,-4.7);
    \draw[fill=color-memoryless,opacity=0.05] (-0.2,-4.7) to[out=90,in=-90] (-0.2,0.25) to[out=0,in=180] (3.7,0.25) to[out=0,in=95] (5,-2.15) to[out=-85,in=180] (6.3,-4.55) to[out=0,in=180] (8.9,-4.55) to[out=0,in=-95] (10.2,-2.15) to[out=-90,in=90] (10.2,-4.7);
    \draw[fill=color-memoryless,opacity=0.05] (-0.2,-4.7) to[out=90,in=-90]  (-0.2,-3.75) to[out=0,in=180] (1.5,-3.75) to[out=0,in=120] (2.4,-4.3) to[out=-60,in=180] (3.3,-4.7) to[out=0,in=180] (10.2,-4.7) to[out=0,in=180] (10.2,-4.7);
    \node[] (pomdp1) at (2.4,0) {POMDP};
    \node[] (rev_pomdp1) at ($(pomdp1) + (-1.3,-1)$) {RevPOMDP};
    \node[] (semi_mdp1a) at ($(rev_pomdp1) + (0,-1)$) {SOMDP$_{\eta_s \in (0,1]}$};
    \node[] (inter_mdp1a) at ($(semi_mdp1a) + (0,-1)$) {IOMDP$_{\lambda \in (0,1]}$};
    \node[] (semi_mdp1b) at ($(rev_pomdp1) + (2.7,-1)$) {SOMDP$_{\eta_s \in [0,1]}$};
    \node[] (inter_mdp1b) at ($(semi_mdp1b) + (0,-1)$) {IOMDP$_{\lambda \in [0,1]}$};
    \node[] (mdp1) at ($(inter_mdp1a) + (0,-1)$) {MDP};
    \node[] (blind_mdp1) at ($(inter_mdp1b) + (0,-1)$) {BlindMDP};
    \draw[opacity=0] (rev_pomdp1) to node[sloped,opacity=1] {$\subseteq$} (pomdp1);
    \draw[opacity=0] (semi_mdp1a) to node[sloped,opacity=1] {$\subseteq$} (rev_pomdp1);
    \draw[opacity=0] (inter_mdp1a) to node[sloped,opacity=1] {$\subseteq$} (semi_mdp1a);
    \draw[opacity=0] (mdp1) to node[sloped,opacity=1] {$\subseteq$} (inter_mdp1a);
    \draw[opacity=0] (semi_mdp1b) to node[sloped,opacity=1,xscale=-1] {$\subseteq$} (pomdp1);
    \draw[opacity=0] (inter_mdp1b) to node[sloped,opacity=1] {$\subseteq$} (semi_mdp1b);
    \draw[opacity=0] (inter_mdp1a) to node[sloped,opacity=1] {$\subseteq$} (inter_mdp1b);
    \draw[opacity=0] (semi_mdp1a) to node[sloped,opacity=1] {$\subseteq$} (semi_mdp1b);
    \draw[opacity=0] (blind_mdp1) to node[sloped,opacity=1] {$\subseteq$} (inter_mdp1b);
    \node[] (pomdp2) at (7.6,0) {POMDP};
    \node[] (rev_pomdp2) at ($(pomdp2) + (-1.4,-1)$) {RevPOMDP};
    \node[] (semi_mdp2a) at ($(rev_pomdp2) + (0,-1)$) {SOMDP$_{\eta_s \in (0,1]}$};
    \node[] (inter_mdp2a) at ($(semi_mdp2a) + (0,-1)$) {IOMDP$_{\lambda \in (0,1]}$};
    \node[] (semi_mdp2b) at ($(rev_pomdp2) + (2.8,-1)$) {SOMDP$_{\eta_s \in [0,1]}$};
    \node[] (inter_mdp2b) at ($(semi_mdp2b) + (0,-1)$) {IOMDP$_{\lambda \in [0,1]}$};
    \node[] (mdp2) at ($(inter_mdp2a) + (0,-1)$) {MDP};
    \node[] (blind_mdp2) at ($(inter_mdp2b) + (0,-1)$) {BlindMDP};
    \draw[opacity=0] (rev_pomdp2) to node[sloped,opacity=1] {$\subseteq$} (pomdp2);
    \draw[opacity=0] (semi_mdp2a) to node[sloped,opacity=1] {$\subseteq$} (rev_pomdp2);
    \draw[opacity=0] (inter_mdp2a) to node[sloped,opacity=1] {$\subseteq$} (semi_mdp2a);
    \draw[opacity=0] (mdp2) to node[sloped,opacity=1] {$\subseteq$} (inter_mdp2a);
    \draw[opacity=0] (semi_mdp2b) to node[sloped,opacity=1,xscale=-1] {$\subseteq$} (pomdp2);
    \draw[opacity=0] (inter_mdp2b) to node[sloped,opacity=1] {$\subseteq$} (semi_mdp2b);
    \draw[opacity=0] (inter_mdp2a) to node[sloped,opacity=1] {$\subseteq$} (inter_mdp2b);
    \draw[opacity=0] (semi_mdp2a) to node[sloped,opacity=1] {$\subseteq$} (semi_mdp2b);
    \draw[opacity=0] (blind_mdp2) to node[sloped,opacity=1] {$\subseteq$} (inter_mdp2b);
    \node[] (pomdp3) at (12.8,0) {POMDP};
    \node[] (rev_pomdp3) at ($(pomdp3) + (-1.4,-1)$) {RevPOMDP};
    \node[] (semi_mdp3a) at ($(rev_pomdp3) + (0,-1)$) {SOMDP$_{\eta_s \in (0,1]}$};
    \node[] (inter_mdp3a) at ($(semi_mdp3a) + (0,-1)$) {IOMDP$_{\lambda \in (0,1]}$};
    \node[] (semi_mdp3b) at ($(rev_pomdp3) + (2.8,-1)$) {SOMDP$_{\eta_s \in [0,1]}$};
    \node[] (inter_mdp3b) at ($(semi_mdp3b) + (0,-1)$) {IOMDP$_{\lambda \in [0,1]}$};
    \node[] (mdp3) at ($(inter_mdp3a) + (0,-1)$) {MDP};
    \draw[opacity=0] (rev_pomdp3) to node[sloped,opacity=1] {$\subseteq$} (pomdp3);
    \draw[opacity=0] (semi_mdp3a) to node[sloped,opacity=1] {$\subseteq$} (rev_pomdp3);
    \draw[opacity=0] (inter_mdp3a) to node[sloped,opacity=1] {$\subseteq$} (semi_mdp3a);
    \draw[opacity=0] (mdp3) to node[sloped,opacity=1] {$\subseteq$} (inter_mdp3a);
    \draw[opacity=0] (semi_mdp3b) to node[sloped,opacity=1,xscale=-1] {$\subseteq$} (pomdp3);
    \draw[opacity=0] (inter_mdp3b) to node[sloped,opacity=1] {$\subseteq$} (semi_mdp3b);
    \draw[opacity=0] (inter_mdp3a) to node[sloped,opacity=1] {$\subseteq$} (inter_mdp3b);
    \draw[opacity=0] (semi_mdp3a) to node[sloped,opacity=1] {$\subseteq$} (semi_mdp3b);
    \node[] (pomdp4) at (18.2,0) {POMDP};
    \node[] (rev_pomdp4) at ($(pomdp4) + (-1.4,-1)$) {RevPOMDP};
    \node[] (semi_mdp4a) at ($(rev_pomdp4) + (0,-1)$) {SOMDP$_{\eta_s \in (0,1]}$};
    \node[] (inter_mdp4a) at ($(semi_mdp4a) + (0,-1)$) {IOMDP$_{\lambda \in (0,1]}$};
    \node[] (semi_mdp4b) at ($(rev_pomdp4) + (2.8,-1)$) {SOMDP$_{\eta_s \in [0,1]}$};
    \node[] (inter_mdp4b) at ($(semi_mdp4b) + (0,-1)$) {IOMDP$_{\lambda \in [0,1]}$};
    \draw[opacity=0] (rev_pomdp4) to node[sloped,opacity=1] {$\subseteq$} (pomdp4);
    \draw[opacity=0] (semi_mdp4a) to node[sloped,opacity=1] {$\subseteq$} (rev_pomdp4);
    \draw[opacity=0] (inter_mdp4a) to node[sloped,opacity=1] {$\subseteq$} (semi_mdp4a);
    \draw[opacity=0] (semi_mdp4b) to node[sloped,opacity=1,xscale=-1] {$\subseteq$} (pomdp4);
    \draw[opacity=0] (inter_mdp4b) to node[sloped,opacity=1] {$\subseteq$} (semi_mdp4b);
    \draw[opacity=0] (inter_mdp4a) to node[sloped,opacity=1] {$\subseteq$} (inter_mdp4b);
    \draw[opacity=0] (semi_mdp4a) to node[sloped,opacity=1] {$\subseteq$} (semi_mdp4b);
    \fill[lightgray!40] (-0.2,1.4) rectangle (21,0.6);
    \node[] (models1) at ($(pomdp1) + (0,0.97)$) {Single-environment models};
    \node[] (models2) at ($(pomdp2) + (0,1)$) {Multi-environment models};
    \node[] (models3) at ($(pomdp3) + (0,1)$) {Multi-observation models};
    \node[] (models3) at ($(pomdp4) + (0,0.97)$) {Multi-observation models, fixed support};
    \draw[-,loosely dashed] (-0.2,0.6) to (21,0.6);
    \draw[-,loosely dashed] (5,1.4) to (5,-4.7);
    \draw[-,loosely dashed] (10.2,1.4) to (10.2,-4.7);
    \draw[-,loosely dashed] (15.4,1.4) to (15.4,-4.7);
    \draw[color-memoryless] (-1.3,0.25) to[out=0,in=180] (3.7,0.25) to[out=0,in=95] (5,-2.15) to[out=-85,in=180] (6.3,-4.55) to[out=0,in=180] (8.9,-4.55) to[out=0,in=-95] (10.2,-2.15);
    \draw[color-memoryless, dashed] (10.2,-2.15) to[out=85,in=180] (11.5,-0.2) to (15.4,-0.2);
    \node[color-memoryless, draw=none,anchor=west] at (-1.4,0.37) {\footnotesize{ETR-hard}};
    \node[color-memoryless, draw=none,anchor=west] at (-1.4,0.13) {\footnotesize{in NP}};
    \node[color-memoryless, draw=none,anchor=east] at (15.45,-0.08) {\footnotesize{ETR-hard}};
    \node[color-memoryless, draw=none,anchor=east] at (15.45,-0.32) {\footnotesize{open}};
    \draw[color-memoryless] (-1.175,-3.75) to[out=0,in=180] (1.5,-3.75) to[out=0,in=120] (2.4,-4.3) to[out=-60,in=180] (3.3,-4.7) to[out=0,in=180] (10.2,-4.7);
    \node[color-memoryless, draw=none,anchor=west] at (-1.275,-3.63) {\footnotesize{NP-hard}};
    \node[color-memoryless, draw=none,anchor=west] at (-1.275,-3.87) {\footnotesize{in P}};
\end{tikzpicture}}
    \caption{Overview of the complexity results memoryless policies}
    \label{fig:complexity_overview_memoryless}
\end{figure}
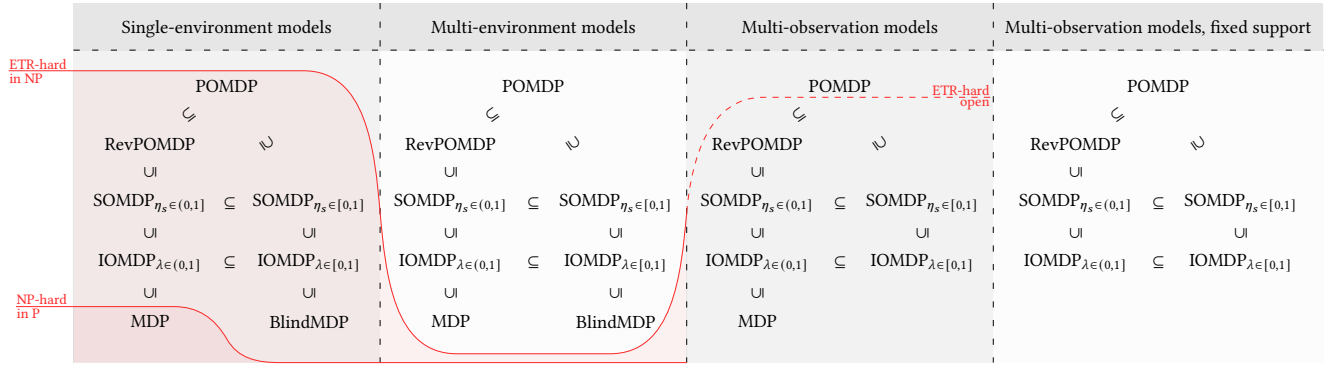

\onecolumn

\section{Proofs for Section~\ref{sec_base_results}}
\label{sec:single_environment_full_proofs}
\subsection{Intermittently-Observable MDPs}\label{app_IOMDPs_single_environment}

\propIOMDPsBase*
\begin{proof}[Proof]
    We reduce 3SAT to this problem.
    The reduction is very similar to the reduction from 3SAT to Blind MDPs as presented in \cite{DBLP:journals/mor/PapadimitriouT87}.
    The difference lies in the way we need to deal with state revelations, which do not occur in BlindMDPs.

    Fix a 3SAT instance $\varphi$, and let $\calM' = (S,A,\{\bot\},T,O,R,b,H)$ be a BlindMDP such that the following conditions hold.
    \begin{description}[font=\normalfont]
        \item[(Assumption 1)] The rewards in the POMDP $\calM'$ are less than or equal to zero.
        \item[(Assumption 2)] There exists a policy in $\calM'$ with zero expected reward if and only if $\varphi$ has a solution.
    \end{description}
    Such a POMDP exists for each 3SAT instance, as one can argue by following the reasoning in \cite{DBLP:journals/mor/PapadimitriouT87}.

    Next, let $\calM$ be an IOMDP with all components equal to $\calM'$ except for the observation set and function, which we modify such that the true state is revealed with probability $\frac{1}{2}$ at each time step.

    We claim that there exists a policy in $\calM$ with non-negative reward if and only if the 3SAT instance $\varphi$ is satisfiable.

    ($\implies$) 
    By assumption, the IOMDP $\calM$ has a history-dependent $\pi$ with zero reward, and we can assume this policy to be deterministic, as such policies are optimal among POMDP policies in the finite-horizon case.

    We now claim that there exists a policy for the POMDP $\calM'$ that also achieves zero reward.
    Let $h^{\bot}_t$ be the $t$ step observation history where the state is never revealed, consisting of $t+1$ $\bot$-observations.
    Note that we can, without loss of generality, assume that the optimal policy is observation-based since it is deterministic.
    Then, let $\pi'$ be a policy for the BlindMDP $\calM'$ where
    \begin{equation}
        \pi'(h^\bot_{t}) = \pi(h^\bot_{t}) \quad \forall t \in 0,\ldots, H-1.
    \end{equation}

    Evaluating the reward of $\pi$, we have
    \begin{multline*}
        \mathbb{E}^\pi_\calM\left[\sum_{t=0}^{H-1} R(s_t,a_t)\right]
        \leq \sum_{s_0,\ldots, s_{H-1} \in S^{H}}\mathbb{P}_\calM[s_0] \prod_{t=0}^{H-2} \mathbb{P}_\calM[s_{t+1}| s_{t}, \pi(h^\bot_{t})] \prod_{t=0}^{H-1} \mathbb{P}_\calM[\bot| s_t] \left( \sum_{t=0}^{H-1} R(s_t, \pi(h^\bot_{t})) \right)
        \\ = \left(\frac{1}{2} \right)^{H}\sum_{s_0,\ldots, s_{H-1} \in S^{H}} \mathbb{P}_\calM[s_0] \prod_{t=0}^{H-2} \mathbb{P}_\calM[s_{t+1}| s_{t}, \pi(h^\bot_{t})] \left( \sum_{t=0}^{H-1} R(s_t, \pi(h^\bot_{t})) \right)
        = \left(\frac{1}{2} \right)^{H}\mathbb{E}^{\pi'}_{\calM'}\left[\sum_{t=0}^{H-1}R(s_t,a_t)\right]
    \end{multline*}
    In the first inequality, we use the fact that we can split the trajectories where the agent observes the state at least once from the trajectories where the agent never observes the state.
    Since all rewards are non-positive, the reward contribution from all trajectories where the agent observes the state at least once is upper bounded by zero.
    Therefore, we can consider the reward contribution of the trajectories where the agent never observes the state as an upper bound on the total reward.
    The first \emph{equality} follows from the fact that the IOMDP $\calM$ has a revelation probability of $\frac{1}{2}$.
    The last equality follows from the fact that the POMDPs $\calM$ and $\calM'$ share a transition function, along with the definition of the policy $\pi'$.
    As $\mathbb{E}^\pi_\calM[\sum_{t=0}^{H-1} R(s_t,a_t)]$ is non-negative, we deduce the existence of a policy in $\calM'$ with non-negative reward.
    Then, using (Assumption 2), we deduce that $\varphi$ is satisfiable.

    ($\impliedby$)
    By (Assumption 2), there exists a policy in the BlindMDP $\calM'$ which attains zero expected reward. 
    We can deploy this policy in $\calM$, playing blind and ignoring any revealing observation, and get zero reward.
\end{proof}

\subsection{Semi-Observable MDPs}\label{app_SOMDPs_single_environment}

\propSOMDPsBase*

\newcommand{\indStateBeliefII}[3]{{\mathbf{b}}_{(#1,#2,#3)}^{\mathsf{seen}}}
\newcommand{\indStateActionII}[3]{{a}_{(#1,#2,#3)}^{\mathsf{seen}}}

\begin{proof}
Fix a SOMDP.
We show that there exists an optimal policy $\pi$ where the action $a$ for each history $h = z_0,a_0,\ldots ,z_{t-1}, a_{t-1}, z_t$ is determined by the following data.
\begin{enumerate}
    \item $t' = \sup\{u : z_u \neq \bot\}$, i.e., the last revelation time\footnote{Using the convention that $\sup(\emptyset) = -\infty$.}.
    \item $s = \begin{cases}
        q : z_{t'} = \top_q & \text{ if } t' \neq -\infty, \\ 
        \mathsf{blind} & \text{ if } t' = -\infty
    \end{cases}$, i.e., the last revealed state (if any).
    \item $t$, i.e., the current time.
\end{enumerate}

Let $(\pi^*_t)_{t = 0}^{H-1}$ be a deterministic time-dependent belief-based optimal policy for the SOMDP.
We will build a policy $\pi$ which attains the same value while tracking the smaller information set described above.
In particular, we will define a policy 
\begin{equation}
    \pi : (S \cup \{\mathsf{blind}\}) \times (\{0,\ldots, H-1\} \cup \{-\infty\}) \times \{0,\ldots, H-1\} \rightarrow A. 
\end{equation}
For $s\in S$, $t' \in (\{0,\ldots, H-1\} \cup \{-\infty\})$ and $t \in \{0,\ldots,H-1\}$, $\pi(s,t',t)$ will encode the action to take when
\begin{itemize}
    \item The state was last revealed at $t'$,
    \item The current time is $t$,
    \item The state was last revealed to be $s$.
\end{itemize}
Meanwhile, $\pi(\mathsf{blind}, -\infty, t)$ will encode the action to take in the event when, at time $t$, the state has never been revealed.

We will build the new policy $\pi$ by inductively computing beliefs and the resulting action. 
Let $\indStateBeliefII{s}{t'}{u}$ denote the belief on the SOMDP's state when
\begin{itemize}
    \item The state was last revealed at $t'$,
    \item The current time is $u$,
    \item The state was last revealed to be $s$, 
    \item We take actions using $\pi^*$. 
\end{itemize}
We can inductively define $\indStateBeliefII{s}{t'}{u}$. 
First, we have
\begin{equation}
    \indStateBeliefII{s}{t'}{t'} = \delta_s.
\end{equation}
That is, if the state was revealed at the current timestep, the belief collapses.
Then, we can iteratively update beliefs as follows.
\begin{equation}
    \indStateActionII{s}{t'}{u} = \pi^*_{u}(\indStateBeliefII{s}{t'}{u})
    \qquad \qquad
    \indStateBeliefII{s}{t'}{u+1}(s') = \frac{(1-\eta(s')) \cdot \sum_{s_u\in S} T(s_u,\indStateActionII{s}{t'}{u})(s')\cdot \indStateBeliefII{s}{t'}{u}(s_u)}{\sum_{s_{u+1}\in S}(1-\eta(s_{u+1}))\cdot \sum_{s_u\in S}T(s_u,\indStateActionII{s}{t'}{u})(s_{u+1})\cdot \indStateBeliefII{s}{t'}{u}(s_u)}
\end{equation}
In the above equation, $\indStateActionII{s}{t'}{u}$ encodes the action the optimal policy will take at time $u$ if the state was last observed at $t'$ to be $s$. 
With this recursion, we can then set
\begin{equation}
    \pi(s,t',t) = \indStateActionII{s}{t'}{t}
\end{equation}

To compute $\pi(\mathsf{blind}, -\infty, t, 0)$, we simply propagate the initial belief as appropriate. 
In particular, let $b$ be the initial state distribution. We define $a^{\mathsf{blind}}_t$ inductively.
We set $\mathbf{b}^{\mathsf{blind}}_{0} = b$, and then we use the following recursion.
\begin{equation}
    a^{\mathsf{blind}}_t = \pi^*_t(\mathbf{b}^{\mathsf{blind}}_{t})
    \qquad \quad
    \mathbf{b}^{\mathsf{blind}}_{t+1}(s') = \frac{(1-\eta(s')) \cdot \sum_{s_t\in S} T(s_t,a^{\mathsf{blind}}_t)(s')\cdot \mathbf{b}^{\mathsf{blind}}_{t}(s_t)}{\sum_{s_{t+1}\in S}(1-\eta(s_{t+1}))\cdot \sum_{s_t\in S}T(s_t,a^{\mathsf{blind}}_t)(s_{t+1})\cdot \mathbf{b}^{\mathsf{blind}}_{t}(s_t)}
\end{equation}
We then set $\pi(\mathsf{blind}, -\infty, t)$ to $a^{\mathsf{blind}}_t$.

We now claim that $\pi$ is optimal. 
We first argue that, for any observable history $h$ that is possible\footnote{I.e., occurring with probability greater than zero.} when the policy $\pi^*$ is used, $\pi^*$ and $\pi$ choose the same action. 
Fix an observable history $h$ under $\pi^*$, where $h = z_0,a_0,\ldots ,z_{t-1}a_{t-1},z_t$, and let $a_t$ be the action $\pi^*$ chooses upon observing this history.
Since $h$ is observable and $\pi^*$ is deterministic, we know $\forall i \in \{0,\dots,t-1\}, \pi^*_{i}(b_{i}) = a_i$ where $b_{i}$ is the belief after observing history $h_i = z_0,a_0,\dots,z_i$.
Additionally, let $t'$ be the latest time $u$ such that $z_u \neq \bot$, and 
let $s$ be the state such that $z_{t'} = \top_s$.

We argue by induction that, for each $l \in \{0,\ldots, t - t'\}$, we have $a_{t'+l} = \indStateActionII{s}{t'}{t' + l} = \pi(s,t',t'+l)$ and $b_{t'+l} = \indStateBeliefII{s}{t'}{t'+l}$

For the base case, when $l = 0$, the state has been revealed and $b_{t'} = \delta_s = \indStateBeliefII{s}{t'}{t'}$.
Additionally, observe that
\begin{equation}
    a_{t'} = \pi^*_{t'}(\delta_{s}) = \indStateActionII{s}{t'}{t'}.
\end{equation}
As stated above, the first equality holds as $h$ is an observable history.
Meanwhile, the second equality simply holds from the definition of $ \indStateActionII{s}{t'}{t'}$.

Consider $l \le t-t'-1$.
For the inductive step, assume $a_{t'+l} = \indStateActionII{s}{t'}{t'+l}$ and $b_{t'+l} = \indStateBeliefII{s}{t'}{t'+l}$.
We show $a_{t'+l+1} = \indStateActionII{s}{t'}{t'+l+1}$:
\begin{align*}
    a_{t'+l + 1} &= \pi^*_{t' + l + 1}(b_{t'+l+1})\\
    &= \pi^*_{t' + l + 1}(b_{upd}(b_{t'+l},a_{t'+l},z_{t'+l+1}))\\
    &= \pi^*_{t' + l + 1}(b_{upd}(\indStateBeliefII{s}{t'}{t'+l},\indStateActionII{s}{t'}{t'+l},\bot))\\
    &= \pi^*_{t' + l + 1}(\indStateBeliefII{s}{t'}{t'+l+1})\\
    &= \indStateActionII{s}{t'}{t' + l + 1}
\end{align*}
Additionally, since we update the same belief $b_{t'+l} = \indStateBeliefII{s}{t'}{t'+l}$ with the same action and receive the same blind observation, we may also conclude $b_{t'+l+1} = \indStateBeliefII{s}{t'}{t'+l+1}$.

By the above induction, $a_{t} = \indStateActionII{s}{t'}{t}$
By definition $\pi$ will take the action $\indStateActionII{s}{t'}{t}$ on seeing history $h$.
Thus, on any possible history where the state is revealed at some point, $\pi$ and $\pi^*$ give the same action.

For the case of a history $h = z_0,a_0,\ldots z_{t-1},a_{t-1},z_t$ where $z_k = \bot$ for all $k$, we can use similar arguments.
In this case, we must show that $a_i = a^{\mathsf{blind}}_i$ and $b_i = \mathbf{b}^{\mathsf{blind}}_i$ for all $i \in \{0,\dots t-1\}$.
For the base case, when $i=0$, we have by definition $b_0 = b = \mathbf{b}^{\mathsf{blind}}_0$ and:
\begin{align*}
    a_0 &= \pi^*_0(b)\\
    &= \pi^*_0(\mathbf{b}^{\mathsf{blind}}_{0})\\
    &= a^{\mathsf{blind}}_0
\end{align*}
For the inductive step, assume $a_i = a^{\mathsf{blind}}_i$ and $b_i = \mathbf{b}^{\mathsf{blind}}_i$.
Then we have:
\begin{align*}
    a_{i+1} &= \pi^*_{i+1}(b_{i+1})\\
    &= \pi^*_{i+1}(b_{upd}(b_{i},a_i,\bot))\\
    &= \pi^*_{i+1}(b_{upd}(\mathbf{b}^{\mathsf{blind}}_{i},a_i,\bot))\\
    &= \pi^*_{i+1}(b_{upd}(\mathbf{b}^{\mathsf{blind}}_{i},a^{\mathsf{blind}}_i,\bot))\\
    &= \pi^*_{i+1}(\mathbf{b}^{\mathsf{blind}}_{i+1})\\
    &= a^{\mathsf{blind}}_{i+1}
\end{align*}
Additionally, since we update the same belief $b_i = \mathbf{b}^{\mathsf{blind}}_i$ with the same action and receive the same blind observation, we may also conclude $b_{i+1} = \mathbf{b}^{\mathsf{blind}}_{i+1}$.

By the above induction, we can conclude that $a_t = a^{\mathsf{blind}}_t$, but by definition $\pi$ will take the action $a^{\mathsf{blind}}_t$ on seeing history $h$.
Thus, on any possible history where the state is revealed at some point $\pi$ and $\pi^*$ give the same action.

We can easily deduce now that $\pi$ is optimal.
Indeed, let $\mathsf{Paths}^{\pi^*}$ denote the set of possible $H$-length paths under the policy $\pi^*$, where a path $\rho$ takes the form:
\begin{equation}
    \rho = s_0,z_0,a_0,\ldots,s_{H-1},z_{H-1},a_{H-1}.
\end{equation}
Note that $\sum_{\rho \in \mathsf{Paths}^{\pi^*}} \mathbb{P}^{\pi^*}[\rho] = 1$, and $V^{\pi^*} = \sum_\rho \mathbb{P}^{\pi^*}[\rho] R(\rho)$, where $R(\rho)$ is the reward accumulated along a path.
It is thus sufficient to show that, for all $\rho \in \mathsf{Paths}^{\pi^*}$, we have $\mathbb{P}^{\pi^*}[\rho] = \mathbb{P}^{\pi}[\rho]$, where $\pi$ is our newly constructed policy.
We can easily observe this as follows
\begin{align*}
    \mathbb{P}^{\pi^*}[s_0,z_0,a_0,\ldots ,s_{H-1},z_{H-1},a_{H-1}] 
    &= \mathbb{P}[s_0] 
    \prod_{t=0}^{H-1} \mathbb{P}[z_t| s_t]
    \prod_{t=0}^{H-2} \mathbb{P}[s_{t+1}| s_{t},a_{t}]
    \prod_{t=0}^{H-1} \mathbb{P}^{\pi^*}[a_{t} | z_0 ,a_0, \ldots ,z_{t_1}a_{t-1},z_t]
    \\ 
    &= \mathbb{P}[s_0] 
    \prod_{t=0}^{H-1} \mathbb{P}[z_t| s_t]
    \prod_{t=0}^{H-2} \mathbb{P}[s_{t+1}| s_{t},a_{t}]
    \prod_{t=0}^{H-1} \mathbb{P}^{\pi}[a_{t} | z_0,a_0, \ldots, z_{t_1},a_{t-1},z_t]\\
    &= \mathbb{P}^{\pi}[s_0,z_0,a_0,\ldots ,s_{H-1},z_{H-1},a_{H-1}] 
\end{align*}
In the first and last equality, we simply apply conditional probability rules. Note that we drop superscripts in the probability measure when they do not depend on the deployed policy. 
In the second equality, we use the following facts
\begin{itemize}
    \item As the path $\rho$ is possible under $\pi^*$, $\mathbb{P}^{\pi^*}[a_{t} | z_0, a_0, \ldots, z_{t_1},a_{t-1},z_t] = 1$ for each $t$.
    \item Since $\pi$ takes the same action as $\pi^*$ for any possible history,  $\mathbb{P}^{\pi}[a_{t} | z_0, a_0, \ldots ,z_{t-1},a_{t-1},z_t] = 1$ for each $t$
\end{itemize}
We now conclude that $\pi^*$ and $\pi$ give the same reward.

Next, we show that we can evaluate a policy of the proposed structure in polynomial time.
First, note that we can represent the policy class using time-dependent FSCs with memory state space $Q = (S \cup \{s_\bot\})\times (\{0,\dots,H-1\} \cup \{-\infty\})$.
Given a policy $\pi\colon (S \cup \{\mathsf{blind}\}) \times (\{0,\dots,H-1\} \cup \{-\infty\}) \times \{0,\dots,H-1\} \to A$, we construct time-dependent FSC policy $(\iota, (\gamma_t)_{t\in \{0,\dots,H-1\}}, (\eta_t)_{t\in \{0,\dots,H-1\}})$ such that
\begin{align*}
    \iota &= \{(s_\bot,-\infty) \mapsto 1\},\\
    \gamma_t((s,t'),z) &= \begin{cases}
        \pi(s',t,t) & \text{if $z = \top_{s'}$},\\
        \pi(s,t',t) & \text{if $z = \bot$},
    \end{cases}\\
    \eta_t((s,t'),z,a) &= \begin{cases}
        (s',t) & \text{if $z = \top_{s'}$},\\
        (s,t') & \text{if $z = \bot$},
    \end{cases},
\end{align*}
where we write $\top_s \in Z$ for the unique revealing observation of state $s$ and $\bot \in Z$ for the blind observation.
Note that this FSC is deterministic.

We then use \cref{lem:fsc_product_pomdp} and \cref{lem:dtd_pol_eval_rep_and_time} to conclude that we can evaluate a deterministic time-dependent FSC in polynomial time.

We can conclude that there exists a non-deterministic algorithm to solve this problem.
In particular, we 
\begin{enumerate}
    \item Guess a policy by guessing an action for each combination of current time, last revelation time, and revealed state.
    \item Evaluate the guessed policy via the FSC construction.
\end{enumerate}
\end{proof}

\subsection{Revealing POMDPs}\label{app_RevPMDPs_single_environment}

\propRevPOMDPsBase*

\begin{proof}[Proof]
    We reduce QSAT to this problem.
    The reduction is similar very similar to the reduction from QSAT to POMDPs as presented in \cite{DBLP:journals/mor/PapadimitriouT87}.
    The difference lies in the way we need to deal with state revelations, which do not occur in the POMDPs used in the QSAT reduction.

    Fix a QSAT instance $\varphi$, and let $\calM' = (S,A,Z,T,O,R,b,H)$ be a POMDP such that the following conditions hold.
    \begin{description}[font=\normalfont]
        \item[(Assumption 1)] The rewards in the POMDP $\calM'$ are less than or equal to zero.
        \item[(Assumption 2)] There exists a policy in $\calM'$ with zero expected reward if and only if $\varphi$ has a solution.
    \end{description}
    Such a POMDP exists for each QSAT instance \cite{DBLP:journals/mor/PapadimitriouT87}.

    Next, let $\calM$ be an RevPOMDP with all components equal to $\calM'$ except for the observation set and function, which we modify such that the true state is revealed with probability $\frac{1}{2}$.

    We claim that there exists a policy in $\calM$ with non-negative reward if and only if the QSAT instance $\varphi$ is satisfiable.

    ($\implies$) 
    By assumption, the RevPOMDP $\calM$ has a policy $\pi$ with zero reward, and we can assume this policy to be deterministic, as such policies are optimal among POMDP policies in the finite-horizon case.

    We now claim that there exists a policy for the POMDP $\calM'$ that also achieves zero reward.
    Since the RevPOMDP is the POMDP extended with state revelations, the set of every non-revealing observation history in the RevPOMDP and the set of all observation histories in the POMDP are the same.
    Let $h^\pi_{t,\bot}$ be a history obtained by playing $\pi$ for $t$ steps where the state is never revealed.
    Then, let $\pi'$ be a policy for the BlindMDP $\calM'$ where
    \begin{equation}
        \pi'(h^{\pi}_{t,\bot}) = \pi(h^\pi_{t,\bot}) \quad \forall t \in \{0,\ldots, H-1\}. 
    \end{equation}

    Evaluating the reward of $\pi$, we have
    \begin{multline*}
        \mathbb{E}^\pi_\calM\left[\sum_{t=0}^{H-1} R(s_t,a_t)\right]
        \leq \sum_{s_0,\ldots, s_{H-1} \in S^{H}}\mathbb{P}_\calM[s_0] \prod_{t=0}^{H-2} \mathbb{P}_\calM[s_{t+1}| s_{t}, \pi(h^{\pi}_{t,\bot})] \prod_{t=0}^{H-1} \mathbb{P}_\calM[\bot| s_t] \left( \sum_{t=0}^{H-1} R(s_t, \pi(h^\pi_{t,\bot})) \right)
        \\ = \left(\frac{1}{2} \right)^{H}\sum_{s_0,\ldots, s_{H-1} \in S^{H}} \mathbb{P}_\calM[s_0] \prod_{t=0}^{H-2} \mathbb{P}_\calM[s_{t+1}| s_{t}, \pi(h^\pi_{t,\bot})] \left( \sum_{t=0}^{H-1} R(s_t, \pi(h^\pi_{t,\bot})) \right)
        = \left(\frac{1}{2} \right)^{H}\mathbb{E}^{\pi'}_{\calM'}\left[\sum_{t=0}^{H-1}R(s_t,a_t)\right]
    \end{multline*}
    In the first inequality, we use the fact that we can split the trajectories where the agent observes the state at least once from the the trajectories where the agent never observes the state.
    Since all rewards are non-positive, the reward contribution from all trajectories where the agent observes the state at least once is upper bounded by zero.
    Therefore, we can consider the reward contribution of the trajectories where the agent never observes the state as an upper bound on the total reward.
    The first equality follows from the fact that the RevPOMDP $\calM$ has a revelation probability of $\frac{1}{2}$.
    The last equality follows from the fact that the POMDPs $\calM$ and $\calM'$ share transition dynamics, along with the definition of the policy $\pi'$.
    As $\mathbb{E}^\pi_\calM[\sum_{t=0}^{H-1} R(s_t,a_t)]$ is non-negative, we deduce the existence of a policy in $\calM'$ with non-negative reward.
    Then, using (Assumption 2), we know $\varphi$ is satisfiable.

    ($\impliedby$)
    By (Assumption 2), there exists a policy in the BlindMDP $\calM'$ which attains zero expected reward. 
    We can deploy this policy in $\calM$, playing blind when we receive a revealing observation, and get zero reward.
\end{proof}

\subsection{Time-dependent Markovian policies in IOMDPs with non-zero revelation probabilities}

\propIOMDPsMemoryless*
\begin{proof}
    We can prove this result by adapting the argument from Proposition~\ref{prop:IOMDPs_base} to time-dependent policies.
    In particular, fix a 3SAT instance $\varphi$, and let $\calM' = (S,A,\{\bot\},T,O,R,b,H)$ be the same BlindMDP satisfying Assumptions 1 and 2 from the proof of Proposition~\ref{prop:IOMDPs_base}.
    \begin{description}[font=\normalfont]
        \item[(Assumption 1)] The rewards in the POMDP $\calM'$ are less than or equal to zero.
        \item[(Assumption 2)] There exists a policy in $\calM'$ with zero expected reward if and only if $\varphi$ has a solution.
    \end{description}
    Also, again, let $\calM$ be the IOMDP built from $\calM'$ where the true state is revealed with probability $\frac{1}{2}$.
    
    We claim that there exists a time-dependent Markovian policy in $\calM$ with non-negative reward if and only if the SAT instance $\varphi$ is satisfiable.

    ($\implies$) 
    By assumption, the IOMDP $\calM$ has a time-dependent Markovian policy $(\pi_t)_{t=0}^{H-1}$ with non-negative reward, and we can assume this policy to be deterministic, as such policies are optimal among time-dependent Markovian policies in the finite-horizon case \cite{DBLP:conf/nips/BagnellKNS03}.

    We now claim that there exists a policy for the POMDP $\calM$ that also achieves zero reward.
    Indeed, let $\pi'$ be a time-dependent policy for the BlindMDP $\calM'$ where
    \begin{equation}
        \pi'_t(\bot) = \pi_t(\bot) \quad \forall t \in 0,\ldots, {H-1}. 
    \end{equation}
    By the same reasoning as in the proof of Proposition~\ref{prop:IOMDPs_base}, we can argue that
    \begin{equation}
          \mathbb{E}^\pi_\calM\left[\sum_{t=0}^{H-1} R(s_t,a_t)\right] \leq \left(\frac{1}{2} \right)^{H}\mathbb{E}^{\pi'}_{\calM'}\left[\sum_{t=0}^{H-1}R(s_t,a_t)\right],
    \end{equation}
    which implies that $\pi'$ attains non-negative reward in $\calM'$, and thus $\varphi$ is satisfiable.

    ($\impliedby$)
    By (Assumption 2), there exists a history-dependent policy $\pi'$ in the BlindMDP $\calM'$ which attains non-negative expected reward.
    Without loss of generality, we can assume this policy to be deterministic, and as such, the policy $\pi'$ is equivalent to a policy that plays a sequence of actions deterministically.
    Thus, there exists a time-dependent Markovian policy in the BlindMDP $\calM'$ which attains non-negative expected reward. 
    We can deploy this time-dependent Markovian policy in $\calM$ as well and get non-negative reward.

\end{proof}

\subsection{Time-dependent FSC-based policies}

\lemPOMDPsFSCs*

\begin{proof}
    Fix a POMDP $\calM = (S,A,Z,T,O,R,b,H)$. 
    By Lemma~\ref{lem:fsc_product_pomdp}, we can, in polynomial time, build a new POMDP
    \begin{equation}
        \calM' = (S',A',Z',T',O',R',b',H+1)
    \end{equation}
    such that
    any $|Q|$-state time-dependent FSC in $\calM$ has a corresponding time-dependent Markovian policy in $\calM'$ with the same reward, and vice versa. 
    From this fact, it follows that
    \begin{equation}
        \max_{\pi \in \Pi_{\mathsf{FSC}(Q)}} \mathbb{E}_\calM^\pi\left[\sum_{t=0}^{H-1} R(s_t,a_t)\right] = \max_{\pi' \in \polMR} \mathbb{E}_{\calM'}^{\pi'}\left[\sum_{t=0}^H R(s'_t,a'_t)\right]
    \end{equation}
    where $\Pi_{\mathsf{FSC}(Q)}$ is the set of time-dependent-FSC-based policies with $Q$ states.

    Additionally, note that by Lemma~\ref{lem:fsc_product_pomdp}, we can take any deterministic time-dependent Markovian policy in $\calM'$ and convert it to a deterministic time-dependent FSC policy in $\calM$. 
    Since there exists an optimal time-dependent Markovian policy in $\calM'$ that is deterministic, there also exists an optimal time-dependent FSC-based policy in $\calM$ that is deterministic.
    
    To decide whether the POMDP admits an FSC-based policy with $|Q|$ states and reward at least $B$, we can
    \begin{enumerate}
        \item Build the expanded POMDP $\calM'$.
        \item Guess a deterministic time-dependent Markovian policy in $\calM'$.
        \item Evaluate the policy and check whether it attains reward at least $B$.
    \end{enumerate}
    The first two steps can clearly be done in non-deterministic polynomial time, while the last stage can be done in polynomial time by Lemma~\ref{lem:dtd_pol_eval_rep_and_time}.
    We note that the assumption of determinism for the time-dependent Markovian policy is without loss of generality, a fact that holds due to the optimality of deterministic policies among time-dependent Markovian policies.
    
    $\NP$-hardness follows from the time-dependent Markovian case, which we can encode by restricting $|Q|$ to be one.
\end{proof}

\clearpage\newpage
\section{Proofs for Section~\ref{sec:complexity_stays_prior_universal}}
\label{sec:proofs_complexity_stays_prior_universal}

\subsection{Optimal ME-POMDP Policies via Hidden-Mixture POMDPs}
\propPriorUniversal*

\begin{proof}
    Let $(\calM_i)_{i\in [n]}$ be a ME-POMDP with POMDPs in $\calC$ where each POMDP has state, action and observation spaces $S$, $A$, and $Z$. The value satisfies the following equality chain.
    \begin{equation}
        \label{eq:appGameEqChain}
        \max_{\pi \in \Pi} \min_{i \in [n]} V_{\calM_i}^\pi = \min_{\mathbf{p} \in \dist{[n]}} \max_{\pi \in \polDet^{Z,A, H}} \sum_{i \in [n]} \vecat{\mathbf{p}}{i} V^\pi_{\calM_i}
        \\ = \min_{\mathbf{p} \in \dist{[n]}} \max_{\pi \in \Pi_{\calC}^{S\times [n], Z, A, H}} \sum_{i \in [n]} \vecat{\mathbf{p}}{i} V^\pi_{\calM_i}
        = \max_{\pi \in \mathsf{Mix}\left( \Pi_{\calC}^{S\times [n], Z,A, H}\right)} \min_{i \in [n]} V_{\calM_i}^\pi.
    \end{equation}
    The first and last equalities apply the min-max theorem for zero-sum matrix games \cite{basar1998dynamic}, which holds due to the finiteness of $\polDet$ and the finite number of environments. 
    The middle inequality applies the optimality of policies in $\Pi_\calC$ for mixture POMDPs, along with the fact that each mixture POMDP is a POMDP with state space $S \times [n]$.
    
    Similar to the proof in 
    \cite{DBLP:journals/corr/abs-2605-07537}, we can then argue that there exists an optimal policy mixing over at most $n$ policies.
    
    To align with later algorithmic results, we give a proof based on linear programming (LP).
    In particular, observe that the final expression in \eqref{eq:appGameEqChain} is equal to the optimal value of the following LP. 
    \begin{align}
        \max_{\mathbf{q} \in \RR^{\Pi_{\calC}^{S\times [n], Z,A, H}}, B \in \RR} \quad & B \\
        \text{s.t.} \quad  & \sum_{\pi \in \Pi_{\calC}^{S\times [n], Z,A, H}} \vecat{ \mathbf{q}}{\pi} = 1,\label{eq:sum_constraint}\\
        &
        \vecat{ \mathbf{q}}{\pi} \geq 0 & \forall \pi \in \Pi_{\calC}^{S\times [n], Z,A, H}, \label{eq:nonneg_constraint}\\
        & \sum_{\pi \in \Pi_{\calC}^{S\times [n], Z,A, H}} \vecat{ \mathbf{q}}{\pi} V^\pi_i \geq B & \forall i \in [n]. \label{eq:env_constraint}
    \end{align}
    The variable $B$ encodes the value of the ME-POMDP, while the vector $\mathbf{q}$ encodes a distribution over deterministic history-dependent policies.
    The constraint \eqref{eq:env_constraint} encodes that the policy must attain value $B$ in all environments, while the constraints \eqref{eq:sum_constraint} and \eqref{eq:nonneg_constraint} encode the validity of the distribution $\mathbf{q}$.
    Note that we identify the element of the vector $\mathbf{q}$ corresponding to the deterministic policy $\pi$ with $\vecat{ \mathbf{q}}{\pi}$. 
    This LP has an optimal solution since the objective is bounded above by the greatest expected reward obtained under any combination of a POMDP $\calM_i$ and a deterministic policy $\pi$.

    We now show that the feasible set has an extreme point.
    For any deterministic policy $\pi'$ in the policy class, the point $\mathbf{q}' \coloneq \mathbf{e}_{\pi'}, B' \coloneq \min_{i\in [n]} V_i^{\pi'}$ is an extreme point, as the matrix defined by the active constraints at this point is of full rank.
    
    Because the feasible set has an extreme point, by convex optimization theory, this LP has an optimal solution at an extreme point \cite[Proposition 3.4.2]{bertsekas2003convex}. We now show that any such extreme point has at most $n$ non-zero elements in the corresponding probability vector $\mathbf{q}$.
    Given an extreme point $(\mathbf{q},B)$, suppose there are $L$ environments for which \eqref{eq:env_constraint} is tight.
    Let $\{\iota_1,\ldots,\iota_L\} \subseteq [n]$ be the environments $i$ for which $\sum_{\pi \in \Pi_{\calC}^{S\times [n], Z,A, H}} \vecat{ \mathbf{q}}{\pi} V_i^{\pi} = B$.
    Similarly, suppose $\{\kappa_1,\ldots,\kappa_R\} \subset \Pi_{\calC}^{S\times [n], Z,A, H}$ contains the policies corresponding to the zero elements of $\mathbf{q}$.
    Let $V_{i}$ denote the vector in $\mathbb{R}^{\Pi_{\calC}^{S\times [n], Z,A, H}}$ such that $V_{i}[\pi] = V_{i}^\pi$ for all policies $\pi$ in $\Pi_{\calC}^{S\times [n], Z,A, H}$.
    We must have
    \begin{equation}
    \mathsf{rank} \left(
         \begin{bmatrix}
    - V_{\iota_1}^\top & 1\\
    \vdots \\
    - V_{\iota_L}^\top & 1\\
    \mathbf{1}^\top & 0\\
    \mathbf{e}_{\kappa_1}^\top & 0\\
    \vdots  \\
    \mathbf{e}_{\kappa_R}^\top & 0
    \end{bmatrix}
    \right) = 1 + |\Pi_{\calC}^{S\times [n], Z,A, H}|,
    \end{equation}
    as $(\mathbf{q},B)$ is an extreme point \cite[Proposition 3.3.3]{bertsekas2003convex}, where $\mathbf{e}_j$ is the $j^{\text{th}}$ unit vector in the appropriate dimension.
    However, as the number of rows trivially bounds the rank, and $L \leq n$, we then have
    \begin{equation}
        1 + |\Pi_{\calC}^{S\times [n], Z,A, H}| \leq n + 1 + R,
    \end{equation}
    which implies that $R \geq |\Pi_{\calC}^{S\times [n], Z,A, H}| - n$.
    Thus, at an extreme point solution, at most $n$ coordinates of $\mathbf{q}$ are non-zero, and we conclude that the ME-POMDP has an optimal policy that mixes over at most $n$ elements of $\Pi_{\calC}^{S\times [n], Z,A, H}$.
\end{proof}

\subsection{Deterministic POMDPs}
\corDetMEPOMDP*
\begin{proof}
    When we mix over a set of deterministic POMDPs $(\calM_i)_{i \in [n]}$, we get another deterministic POMDP $\tilde{\calM}$. Indeed, let POMDP $i$ be
    \begin{equation}
        \calM_i = (S,A,Z,T_i,O_i,R_i, b_i, H).
    \end{equation}
    Given a mixture distribution $\mathbf{p}$, we can then construct a new deterministic POMDP
    \begin{equation}
        \tilde{\calM} = (S\times [n],A,Z,\tilde{T},\tilde{O},\tilde{R},\tilde{b}, H),
    \end{equation}
    where, for all $s\in S$, $i\in [n]$, we have
    \begin{align*}
        &\tilde{T}((s,i),a) = T_i(s,a) \times \delta_i \\
        &\tilde{O}((s,i)) = O_i(s) \\
        &\tilde{R}((s,i),a) = R_i(s,a) \\
        & \tilde{b}(s,i) = \vecat{\mathbf{p}}{i} b_i(s).
    \end{align*}
    Any policy $\pi$ in the original ME-POMDP is a valid policy in $\tilde{\calM}$ where
    \begin{equation}
        \mathbb{E}^\pi_{\tilde{\calM}}\left[\sum_{t=0}^{H-1} \tilde{R}(\tilde{s}_t,\tilde{a}_t)\right] = \sum_{i\in[n]} \vecat{\mathbf{p}}{i} \mathbb{E}^\pi_{\calM_i}\left[\sum_{t=0}^{H-1} R_i(s_t,a_t)\right].
    \end{equation}
    Critically, this new POMDP $\tilde{\calM}$ is also deterministic, as the transition functions $(T_i)_{i\in[n]}$ and observation functions $(O_i)_{i\in[n]}$ are deterministic.
    
    By Lemma 6.3 in \cite{littman1996algorithms}, deterministic POMDPs have optimal policies with concise representations.
    In particular, \cite{littman1996algorithms} introduces a notion of a \emph{policy tree} to reason about optimal policies in deterministic POMDPs.
    A policy tree is a tree where each node is labeled with an action, and the descendant relation encodes which subtree to follow for each possible observation \cite{littman1996algorithms}.
    To ease exposition, we introduce a slightly more general notion in the form of a policy-directed-acyclic-graph (policy-DAG).
    In particular, a policy-DAG comprises $H$ node sets $(U_t)_{t=0}^{H-1}$, along with two sequences of functions
    \begin{align*}
        & \mathsf{Act}_t: U_t \rightarrow A & \text{for each $t \in 0,\ldots, H-1$, and}\\
        & \mathsf{NodeUpd}_t: U_t \times Z \rightarrow U_{t+1} & \text{for each $t \in 0,\ldots, H-2$},
    \end{align*}
    along with a mapping $\mathsf{Init} : Z \rightarrow U_0$ specifying an initial node for each initial observation.
    By Lemma 6.3 in \cite{littman1996algorithms}, a deterministic POMDP with $m$ states has an optimal policy in the form of a policy-DAG where the number $|U_t|$ of nodes at each stage is bounded by $m$.
    
    We can apply Proposition~\ref{prop:prior_universal} to deduce that there exists an optimal policy for the original ME-POMDP that mixes over $n$ policy DAGs with at most $n|S|$ nodes per time step, and this fact suggests the following non-deterministic algorithm to decide if the ME-POMDP has a policy with reward at least $B$.
    \begin{enumerate}
        \item Guess $n$ policies $(\pi_j)_{j \in [n]}$ where each policy is a policy DAG with at most $n|S|$ nodes in each set $U_t$.
        \item Evaluate each policy $\pi_j$ in each POMDP $\calM_i$.
        \item Solve an LP to determine if there exists a distribution on these policy DAGs such that the induced mixed policy attains reward at least $B$ in the ME-POMDP
    \end{enumerate}

    The first stage of the above procedure clearly takes polynomial time, and we now argue that the last two stages do as well.
    We first fix notation, using $\calN$ to refer to the representation size of the ME-POMDP instance.
    Now, fix a POMDP-index $i \in [n]$, 
    an initial state $s_0 \in \mathsf{supp}(b_i)$,
    and a policy tree $\pi_j$ with associated update functions $\mathsf{Init}^{(j)},\mathsf{NodeUpd}^{(j)}$ and $\mathsf{Act}^{(j)}$.
    As the POMDP $\calM_i$ and policy $\pi_j$ are deterministic, there exists a unique trajectory for each initial state, and we can, in polynomial time, compute the resulting trajectory by applying the following equalities.
    \begin{align*}
        z_t = O_i(s_t) & \quad \forall t \in \{0,\ldots, H-1\},\\
        u_0 = \mathsf{Init}^{(j)}(z_0),\\
        u_{t} = \mathsf{NodeUpd}^{j}_t(u_{t-1},z_{t}) & \quad \forall t \in \{0,\ldots, H-2\},\\
        a_t = \mathsf{Act}^{(j)}_t(u_t) & \quad \forall t \in \{0,\ldots, H-1\},\\
        s_{t+1} = T_i(s_t,a_t) & \quad \forall t \in \{0,\ldots, H-2\}.
    \end{align*}
    We can evaluate the reward of this trajectory in polynomial time, and we can represent the reward as a rational $r_{is}^{(j)} \coloneq p_{is}^{(j)}/q_{is}^{(j)}$ where $p_{is}^{(j)} \leq H 2^{ \calN H}$
    and $q_{is}^{(j)} \leq 2^{\calN H}$ and $\calN$ is the representation-size.
    To evaluate the reward of the policy $\pi_j$ in POMDP $\calM_i$, we simply need to compute the expectation $\sum_{s \in S} r_{is}^{(j)} b_i(s)$. 
    By similar reasoning to Lemma~\ref{lem:dtd_pol_eval_rep_and_time}, we can deduce that evaluating the reward in environment $i$ for policy $j$ can be done in polynomial time, and furthermore, this expected reward has a polynomially-sized representation with respect to the original ME-POMDP instance size.
    Repeating this procedure for each environment $i$ and policy $j$ defines a table of values $(V_i^{(j)})_{i\in [n], j \in [n]}$ where $V_i^{(j)}$ encodes the reward of policy-DAG $j$ in POMDP $i$ and $V_i^{(j)}$ is a rational number with size bounded polynomially in $\calN$.
        
    Finally, to determine whether the requisite mixture exists, we need to determine whether a distribution $\mathbf{q}$ exists satisfying the following linear inequalities
    \begin{equation*}
        \sum_{j\in [n]} \vecat{\mathbf{q}}{j} V_i^{(j)} \geq B \quad \forall i \in [n].
    \end{equation*}
    We can solve this problem by solving an LP.
    While the corresponding LP has rational coefficients, we can convert this LP to an LP with integer coefficients that is polynomially sized in terms of the original ME-POMDP instance.
    We can then solve this LP in polynomial time \cite{khachiyan1980polynomial,DBLP:conf/stoc/Karmarkar84}.

\end{proof}

\subsection{Blind MDPs}
\corBlindMEPOMDP*
\begin{proof}
    When we mix over a set of BlindMDPs $(\calM_i)_{i \in [n]}$, we obviously obtain another BlindMDP with an expanded state space. 
    The optimal policy for this BlindMDP is then a sequence of $H$ actions, and thus by Proposition~\ref{prop:prior_universal}, there exists an optimal policy for the ME-BlindMDP which is a mixture of $n$ such sequences. 
    
    As such, there exists a non-deterministic algorithm to solve this problem.
    In particular, we 
    \begin{enumerate}
        \item Guess $n$ sequences of actions, each of length $H$, where $H$ is the horizon of the ME-BlindMDP.
        \item Evaluate each of the $n$ sequences of actions in each BlindMDP.
        \item Solve a linear program to determine if some mixture of these sequences results in a policy with reward at least $B$.
    \end{enumerate}
    The first step clearly runs in polynomial time, and we next argue the same for the last two steps.

    Each sequence of actions is equivalent to some time-dependent Markovian policy, and so we can evaluate it in polynomial time with the resulting reward having a polynomially-sized representation.
    Indeed, let $(\pi_j)_{j \in [n]}$ be the $n$ policies. Lemma~\ref{lem:dtd_pol_eval_rep_and_time} implies that, in polynomial time, we can evaluate each policy in each POMDP. 
    From this procedure we obtain a set of reward values $(V_i^{(j)})_{i \in [n], j \in [n]}$ where $V_i^{(j)}$ encodes the reward of the $j^{\text{th}}$ policy in the $i^{\text{th}}$ BlindMDP, and each value $V_i^{(j)}$ has a representation $V_{i}^{(j)} = \nicefrac{u_{i}^{(j)}}{v_{i}^{(j)}}$ such that
    \begin{equation}
        \max_{i \in [n],j \in [n]} \max(|u_{i}^{(j)}|,v_{i}^{(j)})\leq 2^{f(\calN)}
    \end{equation}
    for some fixed instance-independent polynomial $f(\calN)$ of the instance size $\calN$. 

    To deduce that the third stage runs in polynomial time, we can then use the same argument as in the proof of Corollary~\ref{cor:det_MEPOMDP}.
    That is, deciding the third stage is equivalent to solving a linear program with integral coefficients, where the coefficients have bit-size bounded by a polynomial of the size of the original ME-BlindMDP representation.
    
\end{proof}

\subsection{On the definition of policy classes}

\label{sec:policy_class_soas}

In Section~\ref{sec:complexity_stays_prior_universal}, we defined the notion of policy classes by invoking the notion of the set of all sets. 
We can make this notion rigorous by assuming that all states, actions, and observations come from some universal set. 
For example, we could assume that each action, state, and observation is a natural number, i.e., 
\begin{equation}
    S \subset \mathbb{N}, A \subset \mathbb{N}, Z \subset \mathbb{N}.
\end{equation}
In this case, we can identify the set of all state sets with $2^{\mathbb{N}}$, and we can use a similar argument for the sets of all action sets and observation sets.

\clearpage\newpage
\section{Proofs for Section~\ref{sec:jumps_subclasses}}
\label{sec:proofs_jumps_subclasses}

\subsection{Multi-Observation MDPs}
\propMOMDPsBase*

\newcommand{\decvar}[2]{A_{#1,#2}}
\newcommand{\truevar}[2]{T_{#1,#2}}
\newcommand{\falsevar}[2]{F_{#1,#2}}
\newcommand{\actTrue}{\mathsf{setTrue}}
\newcommand{\actFalse}{\mathsf{setFalse}}

\newcommand{\statedecvar}[2]{s_{\decvar{#1}{#2}}}
\newcommand{\statetruevar}[2]{s_{\truevar{#1}{#2}}}
\newcommand{\statefalsevar}[2]{s_{\falsevar{#1}{#2}}}

\newcommand{\obsdecvar}[2]{z_{\decvar{#1}{#2}}}
\newcommand{\obstruevar}[2]{z_{\truevar{#1}{#2}}}
\newcommand{\obsfalsevar}[2]{z_{\falsevar{#1}{#2}}}

\begin{proof}
Let $\exists x_1 \forall x_2 \ldots \exists x_{n-1} \forall x_n, F(x_1,\ldots, x_n)$ be a quantified formula in conjunctive normal form with $n$ variables and $m$ clauses. We construct a MO-POMDP which has the following properties.
\begin{itemize}
    \item Each environment is an MDP.
    \item Each environment only differs in how we label the states.
    \item There exists a policy with value at least $0$ if and only if the QSAT instance is true. 
\end{itemize}

Intuitively, each environment has the same transition structure as the POMDP constructed in \cite{DBLP:journals/mor/PapadimitriouT87} to prove PSPACE-hardness, with one branch per clause. 
We then use the multiple observation functions to mimic partial observability. 
In particular, we construct an MO-POMDP with environments $(\calM_i)_{i \in [m]}$ where
\begin{equation}
    \calM_i = (S, \underbrace{\{\actTrue,\actFalse\}}_{A}, Z, T, O_i, R, b, \underbrace{2n + 1}_{H}).
\end{equation}
The state space comprises an acceptance state, a dead state, and three states for each clause-variable pair.
In particular, we have
\begin{equation}
    S = \{s_{\mathsf{fail}}, s_{\mathsf{acc}}\} \cup \bigcup_{l \in [m], j \in [n]} \{\statetruevar{l}{j}, \statefalsevar{l}{j}, \statedecvar{l}{j}\}.
\end{equation}
The action space comprises two actions to encode variable assignments
\begin{equation}
    A = \{\actTrue, \actFalse\}.
\end{equation}
Finally, we have a unique observation for each state, i.e.
\begin{equation}
    Z = \{z_{\mathsf{fail}}, z_{\mathsf{acc}}\} \cup \bigcup_{l \in [m], j \in [n]} \{\obstruevar{l}{j}, \obsfalsevar{l}{j}, \obsdecvar{l}{j}\}.
\end{equation}
The transition dynamics then encode the QSAT instance in a similar manner to \cite{DBLP:journals/mor/PapadimitriouT87}.
We set the initial distribution $b$ so that the initial state of the POMDP is uniformly chosen from $\{\statedecvar{1}{1},\ldots,\statedecvar{m}{1}\}$, and this random initial state choice corresponds to a random choice of a clause.
If $x_j$ is an \emph{existential variable}, the next state after $\statedecvar{l}{j}$ follows the agent's action. That is, for each $l \in [m]$, $j \in \{1,3,\ldots, n-1\}$, we have
\begin{align*}
    & T(\statedecvar{l}{j}, \actTrue) = \statetruevar{l}{j} \\
    & T(\statedecvar{l}{j}, \actFalse) = \statefalsevar{l}{j}.
\end{align*}
Meanwhile, if $x_j$ is a \emph{universal variable}, the next state is randomly chosen from $\{\statetruevar{l}{j}, \statefalsevar{l}{j}\}$, i.e., for each $l \in [m], j \in \{2,4,\ldots, n\}$, we have
\begin{equation*}
     T(\statedecvar{l}{j}, a) = \left\{\statetruevar{l}{j} \mapsto \frac{1}{2}, \statefalsevar{l}{j} \mapsto \frac{1}{2}\right\} \quad \forall a \in A. 
\end{equation*}
The transitions from $\statefalsevar{l}{j}$ and $\statetruevar{l}{j}$ then encode the QSAT instance.
In particular, for each clause $l$, the variables $\cup_{j \in [n]} \{\statedecvar{l}{j},\statefalsevar{l}{j}$, $\statetruevar{l}{j}\}$ form a branch where a clause-satisfying variable assignment sends the process to the accepting state.
More precisely, for each $j \leq n - 1$, $l \in [m]$, we set
\begin{align*}
    T(\statetruevar{l}{j}, a) = \begin{cases}
        s_{\mathsf{acc}} & \text{if $x_j$ appears in $C_l$} \\
        \statedecvar{l}{j+1} & \text{ otherwise. }
    \end{cases}
    \quad & \forall a \in A\\
    T(\statefalsevar{l}{j}, a) = \begin{cases}
        s_{\mathsf{acc}} & \text{if $\neg x_j$ appears in $C_l$} \\
        \statedecvar{l}{j+1} & \text{ otherwise. }
    \end{cases}
    \quad & \forall a \in A.
\end{align*}
We additionally have, for each $l \in [m]$, 
\begin{align*}
    T(\statetruevar{l}{n}, a) = \begin{cases}
        s_{\mathsf{acc}} & \text{if $x_n$ appears in $C_l$} \\
        s_{\mathsf{fail}} & \text{ otherwise. }
    \end{cases}
    \quad & \forall a \in A\\
    T(\statefalsevar{l}{n}, a) = \begin{cases}
        s_{\mathsf{acc}} & \text{if $\neg x_n$ appears in $C_l$} \\
        s_{\mathsf{fail}} & \text{ otherwise. }
    \end{cases}
    \quad & \forall a \in A.
\end{align*}
At this stage, we have merely recalled the transition structure from \cite{DBLP:journals/mor/PapadimitriouT87}, but we will next introduce the observation function.

Before introducing the observation function, we first fix notation.
In particular, for $k \in \{1,\ldots, m\}$, we use $\varphi_k: [m] \rightarrow [m]$ to denote the function which permutes the index. That is,
\begin{equation}
    \varphi_k(l) = \begin{cases}
        l + k & \text{ if } l + k \in [m] \\
        l + k - m & \text{ otherwise }.
    \end{cases}
\end{equation}

In environment $k$, for $k \in [m]$,  we then set the observation function to satisfy
\begin{equation}
    O_k(\statetruevar{l}{j}) = z_{T_{\varphi_k(l),j}} \quad 
    O_k(\statefalsevar{l}{j}) = z_{F_{\varphi_k(l),j}} \quad
    O_k(\statedecvar{l}{j}) = z_{A_{\varphi_k(l),j}} \quad \forall l \in [m], j \in [n].
\end{equation}
In words, this definition of the observation function implies that, in environment $k$, variables corresponding to clause $l$ emit observations corresponding to clause $\varphi_k(l)$.
For all environments, we set $O_k(s_{\mathsf{acc}}) = z_{{\mathsf{acc}}}$ and $O_k(s_{\mathsf{fail}}) = z_{\mathsf{fail}}$.

Finally, we set the reward function to penalize reaching the dead state $s_{\mathsf{fail}}$. That is, we set $R(s_{\mathsf{fail}}, a) = - 1$ for all actions $a$ in $A$, and we set the reward to be zero for all other state-action pairs.

Clearly, this ME-POMDP is a MO-MDP, since the shuffling procedure only changes the labels, and for any fixed environment, each state has a unique observation. 
We can also clearly build this MO-MDP in polynomial time in the size of the QSAT instance.

We now claim that the QSAT instance is true if and only if the MO-MDP has a policy with non-negative reward.

($\implies$)
This fact follows similar logic to the original $\PSPACE$-completeness proof for POMDPs \cite{DBLP:journals/mor/PapadimitriouT87}.
Indeed, suppose $\sigma: \bigcup_{t = 0}^{\frac{n}{2}-1} \{0,1\}^{2t} \rightarrow \{0,1\}$ is a winning strategy for the QSAT instance where $\sigma(\bit_1,\ldots,\bit_{2t})$ encodes an existential assignment for each bit-history that ensures the QSAT instance is satisfied for all possible universal variable assignments.
We can define a natural strategy for the MO-MDP that, for each history,
\begin{enumerate}
    \item Translates $z_{T_{\cdot,\cdot}}$ and $z_{F_{\cdot,\cdot}}$ observations to $1$ and $0$ respectively.
    \item Deploys the action that $\sigma$ defines.
\end{enumerate}
Clearly, this policy reaches the acceptance state for any clause branch and for any of the random choices of the universal variables.

($\impliedby$)
Now, assume the MO-MDP has a policy $\pi$ with a non-negative reward. 

We can argue that this strategy is, \wlogg, deterministic.
Indeed, let $\tilde{\pi}$ be a general history-dependent randomized strategy.
This strategy has a representation as a mixed strategy, a fact that follows from \cite{kuhn1953extensive}.
As such, for a deterministic history-dependent policy $\pi \in \polDet$, let $p_\pi$ be the probability of $\pi$ under the mixed-policy representation of $\tilde{\pi}$.
By assumption, $\sum_{\pi \in \polDet} p_\pi \EE_{\calM_i}^\pi[\sum_{t=0}^{H-1} R(s_t,a_t)] \geq 0$ for all environments $i$.
Additionally, all rewards are non-positive.
Thus, we can deduce that, for any environment $i \in [m]$, and any deterministic policy $\pi$ such that $p_\pi > 0$, we have $ \EE_{\calM_i}^\pi[\sum_{t=0}^{H-1} R(s_t,a_t)] \geq 0$.
Hence, there must exist some deterministic strategy that attains zero reward in all environments.

We can construct a winning strategy for the QSAT instance using the behavior of $\pi$ on seeing observations in the set 
\begin{equation}
    \tilde{Z}_1 := \bigcup_{j \in [n]} \{z_{T_{1,j}}, z_{F_{1,j}}, z_{A_{1,j}}\}.
\end{equation}
Intuitively, because any clause can produce observations in $\tilde{Z}_1$, the corresponding strategy must provide a valid clause-independent assignment strategy.
In particular, we first construct $\sigma: \bigcup_{t = 0}^{\frac{n}{2}-1} \{0,1\}^{2t} \rightarrow \{0,1\}$ so that a sequence $\bit_1,\ldots,\bit_{2t}$ is assigned the action that $\pi$ takes when the corresponding history is observed. 
More precisely, let $\mathsf{V}_j(0) = z_{\falsevar{1}{j}}$ and $\mathsf{V}_j(1) = z_{\truevar{1}{j}}$ for each $j \in [n]$. 
Then, fix an even-length bit string $\bit_1,\ldots,\bit_j$, i.e., a string ending in a universally quantified variable.
Next, inductively define a sequence of actions as follows.
\begin{align*}
    & a_1^{(1)} = \pi(z_{\decvar{1}{1}}) \\
    & a_1^{(2)} = \pi(z_{\decvar{1}{1}}, a_1^{(1)}, {\mathsf{V}_1(\bit_1)}) \\
    & \ldots \\
    & a_t^{(1)} = \pi(z_{\decvar{1}{1}}, a_1^{(1)}, \ldots, {\mathsf{V}_{t-1}(\bit_{t-1})}, a_{t-1}^{(2)}, z_{\decvar{1}{t}} ) \\
    & a_t^{(2)} = \pi(z_{\decvar{1}{1}}, a_1^{(1)}, \ldots, {\mathsf{V}_{t-1}(\bit_{t-1})}, a_{t-1}^{(2)}, z_{\decvar{1}{t}}, a_t^{(1)}, {\mathsf{V}_{t}(\bit_t)})  \\
\end{align*}
We then define $\sigma(\bit_1,\ldots,\bit_j)$ as
\begin{equation}
    \sigma(\bit_1,\ldots,\bit_j) = \begin{cases}
        1 & \text{ if } a_{j+1}^{(1)} = \actTrue, \\
        0 & \text{ if } a_{j+1}^{(1)} = \actFalse.
    \end{cases}
\end{equation}

We now argue that this strategy $\sigma$ is indeed winning for the QSAT instance.
Fix an arbitrary clause $l \in [m]$.
Let $x_1,\ldots, x_n$ be the bit string resulting from some choice of universal variables under the strategy $\sigma$. 
We argue that if $x_1,\ldots,x_n$ does not satisfy clause $C_l$, then for some environment in the MO-POMDP, the policy $\pi$ gets negative reward.
Indeed, first observe that, in environment $(m+1) - l$, the $l^{th}$ branch of the transition function produces observations in the set $\tilde{Z}_1 = \bigcup_{j \in [n]} \{z_{T_{1,j}}, z_{F_{1,j}}, z_{A_{1,j}}\}$ while visiting states in the set 
$\cup_{j \in [n]} \{\statedecvar{l}{j},\statefalsevar{l}{j}, \statetruevar{l}{j} \}$ and following a transition structure that sends the agent to the accepting state $s_{\mathsf{acc}}$ if they satisfy $C_l$.
Now observe that the transition dynamics will follow the bit string $x_1,\ldots,x_n$ with positive probability.
Indeed, when the agent takes an action using $\pi$ at existential decision states, the agent will play an action corresponding to the assignment strategy $\sigma$, and at universal states, nature will play the next bit in $x_1,\ldots,x_n$ with probability $\nicefrac{1}{2}$.
Now, by assumption, $x_1,\ldots,x_n$ does not satisfy the clause $C_l$, and so this path will reach the dead state at the last timestep, where it gets reward $-1$, but this fact implies that the strategy $\pi$ obtains negative reward in at least one environment.

Ultimately, the above reasoning follows the same logic as \cite{DBLP:journals/mor/PapadimitriouT87}. The main difference appears in how we use the multi-observation structure to ensure that $\sigma$ must be a satisfying assignment strategy for every clause $l \in [m]$. 

\end{proof}

\subsection{Multi-Observation Semi-Observable MDPs with Fixed Support}

\newcommand{\decvarsim}[1]{A_{#1}}
\newcommand{\truevarsup}[2]{T_{#1}^{(#2)}}
\newcommand{\falsevarsup}[2]{F_{#1}^{(#2)}}

\newcommand{\statedecvarsim}[1]{s_{\decvarsim{#1}}}
\newcommand{\statetruevarsup}[2]{s_{\truevarsup{#1}{#2}}}
\newcommand{\statefalsevarsup}[2]{s_{\falsevarsup{#1}{#2}}}

\newcommand{\obsdecvarsim}[1]{z_{\decvarsim{#1}}}
\newcommand{\obstruevarsup}[2]{z_{\truevarsup{#1}{#2}}}
\newcommand{\obsfalsevarsup}[2]{z_{\falsevarsup{#1}{#2}}}

\propMOSOMDPs*

\begin{proof}
    Let $\exists x_1 \forall x_2 \ldots \exists x_{n-1} \forall x_n F(x_1,\ldots, x_n)$ be a quantified formula with $n$ variables, where $F$ is in conjunctive normal form and has $m$ clauses.

    We design a MO-SOMDP where there is one environment $\calM_i = (S,\underbrace{\{\actTrue, \actFalse\}}_{A},Z,T,O_i,R,b, \underbrace{3n + 1}_{H})$ for each clause $i \in [m]$.
    
    The state space $S$ comprises an acceptance state, a dead state, and five states per variable
    \begin{equation}
        S = \{s_{\mathsf{acc}}, s_{\mathsf{fail}}\} \cup \bigcup_{j \in [n]} \{\statedecvarsim{j}, \statetruevarsup{j}{0}, \statetruevarsup{j}{1},\statetruevarsup{j}{-1},  \statefalsevarsup{j}{0}, \statefalsevarsup{j}{1}, \statefalsevarsup{j}{-1}\}.
    \end{equation}
    The states $\statedecvarsim{j}$ correspond to decision nodes where either the agent chooses a value for an existential variable, or a value for a universal variable is drawn randomly.   
    The state triplets $\{\statetruevarsup{j}{0}, \statetruevarsup{j}{1},\statetruevarsup{j}{-1}\}$ and $\{\statefalsevarsup{j}{0}, \statefalsevarsup{j}{1}, \statefalsevarsup{j}{-1}\}$ then encode true and false assignments to $x_j$ respectively, and the different states within these triplets encode coin flip outcomes.
    
    We have a unique observation for each state, along with a special blind observation. That is,
    \begin{equation}
        Z = \{z_{\mathsf{acc}}, z_{\mathsf{fail}}\} \cup \left( \bigcup_{j \in [n]} \{\obsdecvarsim{j}, \obstruevarsup{j}{0}, \obstruevarsup{j}{1},\obstruevarsup{j}{-1},  \obsfalsevarsup{j}{0}, \obsfalsevarsup{j}{1}, \obsfalsevarsup{j}{-1}\} \right) \cup \{\bot\}.
    \end{equation}
    
    The transition dynamics encode variable assignment and the random coin flips.
    In particular, if $x_j$, for $j \in [n]$, is an \emph{existential} variable, then 
    \begin{align*}
        & T(\statedecvarsim{j}, \actTrue) = \statetruevarsup{j}{0} \\
        & T(\statedecvarsim{j}, \actFalse) = \statefalsevarsup{j}{0}.
    \end{align*}
    Meanwhile, if $x_j$ is a \emph{universal} variable, the next state is random, i.e.,
    \begin{equation*}
         T(\statedecvarsim{j}, a) = \left\{\statetruevarsup{j}{0} \mapsto \frac{1}{2},  \statefalsevarsup{j}{0} \mapsto \frac{1}{2}\right\} \quad \forall a \in A. 
    \end{equation*}
    From states in $\{\statetruevarsup{j}{0}, \statefalsevarsup{j}{0}\}$, the environment flips a coin, which we encode in the successor state via the superscript lying in $\{(-1),(1)\}$.
    \begin{equation*}
        T(\statetruevarsup{j}{0}, a) = \left\{\statetruevarsup{j}{1} \mapsto \frac{1}{2},  \statetruevarsup{j}{-1} \mapsto \frac{1}{2}\right\} \quad \quad
        T(\statefalsevarsup{j}{0}, a) = \left\{\statefalsevarsup{j}{1} \mapsto \frac{1}{2},  \statefalsevarsup{j}{-1} \mapsto \frac{1}{2}\right\}
    \end{equation*}
    If the agent matches their action to the coin, they move to the accepting state.
    Meanwhile, if the agent fails to match the coin, they move to the next variable node, except on the last coin flip, when they move to the dead state. 
    That is, for each $j \in [n]$, we have
    \begin{align*}
        & T(\statetruevarsup{j}{1}, \actTrue) = 
        T(\statefalsevarsup{j}{1}, \actTrue) = 
        T(\statetruevarsup{j}{-1}, \actFalse) = 
        T(\statefalsevarsup{j}{-1}, \actFalse) = s_{\mathsf{acc}}\\ 
        & T(\statetruevarsup{j}{1}, \actFalse) = 
        T(\statefalsevarsup{j}{1}, \actFalse) = 
        T(\statetruevarsup{j}{-1}, \actTrue) = 
        T(\statefalsevarsup{j}{-1}, \actTrue) = \begin{cases}
            \statedecvarsim{j+1} & \text{ if } j < n \\
            s_{\mathsf{fail}} & \text{ if } j = n.
        \end{cases}\\ 
    \end{align*}
    
    The initial state in each SOMDP is $\statedecvarsim{1}$. That is $b = \delta_{\statedecvarsim{1}}$.

    The observation changes with the environment in a manner such that satisfying variable assignment allows the agent to see the coin flip value with high probability.
    We tune the observability of states in the set $\bigcup_{j \in [n]} \{\statetruevarsup{j}{1}, \statetruevarsup{j}{-1},  \statefalsevarsup{j}{1}, \statefalsevarsup{j}{-1}\}$ according to the environment. 
    First, let $\epsilon$ be a positive number to be tuned later. 
    If variable $j$ appears positively in clause $C_i$, then $\statetruevarsup{j}{1}$ and $\statetruevarsup{j}{-1}$ are fully observable with high probability. That is
    \begin{align*}
        & O_i(\statetruevarsup{j}{1}) = \{z_{\statetruevarsup{j}{1}} \mapsto 1 - \epsilon, \bot \mapsto \epsilon \} \\
        & O_i(\statetruevarsup{j}{-1}) = \{z_{\statetruevarsup{j}{-1}} \mapsto 1 - \epsilon, \bot \mapsto \epsilon \}.
    \end{align*}
    Conversely, if variable $j$ does not appear positively in clause $C_i$, then the SOMDP emits the blind observation with probability $\nicefrac{1}{2}$, i.e, 
    \begin{align*}
        & O_i(\statetruevarsup{j}{1}) = \left\{z_{\statetruevarsup{j}{1}} \mapsto \frac{1}{2}, \bot \mapsto \frac{1}{2} \right\} \\
        & O_i(\statetruevarsup{j}{-1}) = \left\{z_{\statetruevarsup{j}{-1}} \mapsto \frac{1}{2}, \bot \mapsto \frac{1}{2} \right\}.
    \end{align*}
    Note that in both cases, the support of $O_i(\statetruevarsup{j}{1})$ is $\{z_{\statetruevarsup{j}{1}}, \bot\}$, and the same condition holds for $\statetruevarsup{j}{-1}$.
    We define the observation function in an analogous way for the states corresponding to false variable assignment $\{\statefalsevarsup{j}{1},\statefalsevarsup{j}{-1} \}$ depending on whether or not the corresponding variable appears negatively in clause $C_i$.
    For every environment $i$, we design the observation such that all states in the set
    \begin{equation}
        \{s_{\mathsf{acc}}, s_{\mathsf{fail}}\} \cup \bigcup_{j \in [n]} \{\statedecvarsim{j}, \statetruevarsup{j}{0},  \statefalsevarsup{j}{0}\}
    \end{equation}
    are fully observable. That is, the only information possibly hidden from the agent is the result of the coin flip at each stage.

    Finally, we set the reward function to penalize reaching the dead state $s_{\mathsf{fail}}$. That is, we set $R(s_{\mathsf{fail}}, a) = - 1$ for all actions $a$ in $A$, and we set the reward to be zero for all other state-action pairs.

    \newcommand{\smallrew}{\nu}
    We now argue that the QSAT instance has a solution if and only if
    there exists some policy $\pi$ that attains reward at least $-\smallrew$, where we set
    \begin{equation}
        \smallrew :=  \frac{1}{2m}\left( \frac{1}{8} \right)^n \quad \quad \epsilon := \frac{1}{2m}\left( \frac{1}{8} \right)^n.
    \end{equation}
    Note that we can represent both $\smallrew$ and $\epsilon$ as rationals where the denominator is $2m 8^n$. As such $\smallrew$ and $\epsilon$ have a polynomial bit representation in terms of the QSAT instance, and thus we can still implement this reduction in polynomial time.
    However, in Remark~\ref{rem:splitting_procedure}, we also discuss how to modify this transformation so that all integers used to represent the MO-SOMDP are themselves polynomially bounded in terms of the size of the QSAT instance.

    ($\implies$)
    First, assume that the QSAT instance is satisfiable with some variable assignment strategy $\sigma$.
    Then, let $\pi$ be the corresponding MO-SOMDP policy that
    \begin{enumerate}
        \item Takes the action $\sigma$ recommends at the existential decision nodes.
        \item Takes action $\actTrue$ whenever a state in $\{\statetruevarsup{j}{1}, \statefalsevarsup{j}{1}\}$ is revealed, and takes action $\actFalse$ whenever a state in $\{\statetruevarsup{j}{-1}, \statefalsevarsup{j}{-1}\}$ is revealed.
    \end{enumerate} 
    Note that the MO-SOMDP observation functions reveal the prior true-false assignments to the agent, before the hidden coin flip occurs, so that the agent can implement the action-taking strategy described in the first point.
    The agent can take any arbitrary action when it observes $\bot$.

    Now fix an environment and clause corresponding to the index $i$. 
    As $\sigma$ is a winning strategy for the QSAT instance, it follows that, for any positive-probability path, the agent will either reach $\statetruevarsup{j}{0}$ for a variable $x_j$ that appears positively in $C_i$, or the agent will reach $\statefalsevarsup{j}{0}$ for a variable $x_j$ that appears negatively in $C_i$. 
    On reaching such a state, the environment flips a coin and reveals its value with probability $1 - \epsilon$. As such, the agent's policy reaches $s_{\mathsf{acc}}$ before the final time step with probability at least $1 - \epsilon$, and the agent's reward is at least $- \epsilon = - \smallrew$.

    ($\impliedby$)
    Now, conversely, assume that a possibly randomized strategy $\pi^*$ exists that attains reward at least {$-\smallrew$}, and let this strategy be represented by some distribution on the set $\polDet$ of deterministic policies where $p_\pi$ is the probability of $\pi$ for $\pi \in \polDet$. 
    We additionally assume that these deterministic policies are observation-based, i.e., they are mappings
    \begin{equation}
        \pi : \bigcup_{t = 1}^{H} Z^t \rightarrow A,
    \end{equation} 
    and we use $\polDetObs$ to denote the set of observation-based deterministic policies.
    This assumption is \wlogg, as for any deterministic policy, there exists a deterministic observation-based policy which induces the same path distribution in any POMDP.
    We argue that the randomized strategy $\pi^*$ must place positive weight on a deterministic MO-SOMDP strategy that corresponds to a satisfying assignment.

    We first fix some notation. 
    In particular, for each $j \in [n]$, let 
    \begin{equation*}
        \mathsf{V}_j(0) = \obsfalsevarsup{j}{0} \quad \text{ and } \mathsf{V}_j(1) = \obstruevarsup{j}{0}.
    \end{equation*}That is, $\mathsf{V}_j$ encodes a mapping to convert bits to relevant MO-SOMDP observations.
    
    We now define a mapping $\xi$ from $\polDetObs$ to the set $\Sigma$ of mappings $\sigma: \cup_{t=0}^{n/2 - 1} \{0,1\}^{2t} \rightarrow \{0,1\}$ that represent QSAT assignment strategies.
    In particular, for a deterministic observation-based policy $\pi$, and a bit string $\bit_1,\ldots, \bit_{2t}$, we define an assignment strategy $\xi(\pi)$ according to
    \begin{equation}
        \left( \xi(\pi) \right) (\bit_1,\ldots, \bit_{2t}) = \pi(\obsdecvarsim{1}, \mathsf{V}_1(\bit_1), \bot, \obsdecvarsim{2}, \mathsf{V}_2(\bit_2), \bot, \ldots, \obsdecvarsim{2t}, \mathsf{V}_{2t}(\bit_{2t}), \bot ,\obsdecvarsim{2t + 1}).
    \end{equation}
    That is, we define $\xi(\pi)$ using the assignment strategy that $\pi$ encodes on the event that the agent never sees the coin flip's value.
    Note that we slightly abuse notation, as the policy outputs actions in the set $\{\actFalse,\actTrue\}$ rather than bits, but we use the natural remapping.

    We next define a partition of the space of deterministic policies.
    First, for each $i\in [m]$, define $\calP_i$ as the set of observation-based policies $\pi$ such that strategy $\xi(\pi)$ fails to ensure the clause $C_i$ is satisfied.
    Next, define $\calU_1 = \calP_1$, and for all $i \in \{2,\ldots,m\}$, define $\calU_i = \calP_i \setminus (\calP_1 \cup \ldots \cup \calP_{i-1})$.
    Finally, define $\calU_0$ as $\polDetObs \setminus (\calP_1 \cup \ldots \cup \calP_{m})$.
    Clearly $\calU_0,\ldots, \calU_m$ forms a partition of $\polDetObs$. 
    Also, note that $\calU_0$ contains deterministic POMDP policies which map, via $\xi$, to QSAT strategies that satisfy all clauses for any universal variable assignment.

    We can now bound the probability with which policies in $\calU_i$ are played as part of the mixed strategy for each $i \in [m]$. In particular, fix $\pi \in \calU_i$ as an arbitrary observation-based policy such that the corresponding QSAT assignment strategy fails to satisfy clause $i$ for some universal variable assignment.
    We first lower bound the probability of reaching the dead state in environment $i$ when following $\pi$.
    By definition, there exists a universal variable assignment such that, when paired with the assignment strategy $\xi(\pi)$, clause $C_i$ is not satisfied. 
    Denote the resulting bit string $\bit_1,\ldots, \bit_n$, where by definition $\bit_1,\ldots, \bit_n \not\models C_i$.
    We now argue by induction that the agent will reach the dead state with probability at least $(\nicefrac{1}{8})^n$.
    First, assume the history is $\obsdecvarsim{1},\mathsf{V}_1(\bit_1),\bot,\obsdecvarsim{2},\mathsf{V}_2(\bit_2),\bot ,\ldots,\obsdecvarsim{2t-1},\mathsf{V}_{2t-1}(\bit_{2t-1}),\bot,\obsdecvarsim{2t}$, ending in a decision-state for a universal variable. 
    With probability $\nicefrac{1}{2}$, the environment chooses its action matching $\bit_{2t}$, so that the bit string continues to be produced and the clause remains unsatisfied. 
    For the guessing stage following this assignment, 
    the next stage has a $+1$-superscript with probability $\nicefrac{1}{2}$, and a $-1$-superscript with probability $\nicefrac{1}{2}$, and each of these states emits the blind observation $\bot$ with probability $\nicefrac{1}{2}$.
    Thus, for the guessing game, no policy can do better than a $\nicefrac{3}{4}$ success rate, and so the agent must fail the guessing game and transition to $\statedecvarsim{2t+1}$ with probability at least $\nicefrac{1}{4}$.
    For the next case, assume the history is $\obsdecvarsim{1}, \mathsf{V}_1(\bit_1), \bot, \obsdecvarsim{2}, \mathsf{V}_2(\bit_2), \bot, \ldots, \obsdecvarsim{2t}, \mathsf{V}_{2t}(\bit_{2t}), \bot, \obsdecvarsim{2t+1}$, ending in an \emph{existential} variable. 
    By assumption, the policy's next bit choice matches $\bit_{2t+1}$, and the clause $C_i$ remains unsatisfied.
    Thus, by the same logic as above, the environment moves to the state $\statedecvarsim{2t+2}$ with probability at least $\nicefrac{1}{4}$.
    By repeating this reasoning inductively, we can deduce that, when the agent follows the deterministic policy $\pi$, they will reach the dead state with probability $(\nicefrac{1}{2})^{n/2}(\nicefrac{1}{4})^n \geq \left( \nicefrac{1}{8} \right)^n$.
    
    We can use this fact to then upper bound $p_{\calU_i} := \sum_{\pi \in \calU_i} p_\pi$.
    In particular, by the above argument, we have
    \begin{equation}
        \mathbb{E}_{\calM_i}^{\pi^*}\left[ \sum_{t=0}^{H-1} R(s_t,a_t) \right] \leq 
        \sum_{\pi \in \calU_i}  p_\pi \mathbb{E}_{\calM_i}^{\pi}\left[ \sum_{t=0}^{H-1} R(s_t,a_t) \right] \leq 
        - \sum_{\pi \in \calU_i}  p_\pi \left( \frac{1}{8} \right)^n = -p_{\calU_i}\left( \frac{1}{8} \right)^n
    \end{equation}
    The first inequality follows from the non-positivity of the reward, the second inequality follows from the above reasoning about $\calU_i$, and the final equality follows by definition of $p_{\calU_i}$.

    By assumption, we have
    \begin{equation}
        -{\frac{1}{2m}\left( \frac{1}{8} \right)^n } = - \smallrew \leq \mathbb{E}_{\calM_i}^{\pi^*}\left[ \sum_{t=0}^{H-1} R(s_t,a_t) \right] \leq -p_{\calU_i}\left( \frac{1}{8} \right)^n,
    \end{equation}
    from which we can deduce that $p_{\calU_i} \leq \frac{1}{2m }$. 
    
    Finally, we can deduce that
    \begin{equation}
        p_{\calU_0} = 1 - \sum_{i \in [m]} p_{\calU_i} \geq 1 - \frac{1}{2} > 0.
    \end{equation}
    That is, the randomized policy must place positive weight on policies corresponding to assignment strategies that satisfy every clause for every universal variable assignment.
    However, this fact holds only if there exists a policy in $\calU_0$, which implies the satisfiability of the QSAT instance.
\end{proof}

\begin{remark}
    \label{rem:splitting_procedure}
    In the proof of Proposition~\ref{prop:MOSOMDPs}, the rationals used in the observation function have numerators and denominators that grow exponentially in the size of the problem, as we depict in Figure~\ref{fig:mosomdp_default_proof}.
    We remark that one can prove the same result while keeping the values of the numerators and denominators polynomially bounded, and we sketch this argument below.
    In particular, we discuss a procedure to build a ME-POMDP $(\calM_i')_{i \in [m]}$ such that $\max_{\pi}V_{(\calM'_i)_{i \in [m]}}^\pi \geq 0$ if and only if the original QSAT instance is satisfiable, and all the revelation probabilities in $(\calM_i')_{i \in [m]}$ have numerators and denominators bounded by $2n + 2m$.

    We first distribute the revelation probability across multiple copies of the coin flip states in the manner depicted in Figure~\ref{fig:splitting_procedure}.
    For each state of the form $\statetruevarsup{j}{\cdot}$ or $\statefalsevarsup{j}{\cdot}$ in the original ME-POMDP, we make $n+1$ copies, so that the agent stays in such a state for $n+1$ time steps.
    In particular, if variable $x_j$ appears positively in clause $i$, then, in environment $i$, the first copy of $\statetruevarsup{j}{1}/\statetruevarsup{j}{-1}$ is revealed with probability $ 1 -\frac{1}{2m}$, and the remaining $n$ copies are revealed with probability $1 - \frac{1}{8}$. 
    With this revelation structure, the probability that the agent sees a copy of the state at least once is $1 - \frac{1}{2m} \frac{1}{8^n}$.
    Similar reasoning holds for $\statefalsevarsup{j}{1}/\statefalsevarsup{j}{-1}$.
    Meanwhile, if variable $x_j$ does not appear positively in clause $i$, then each of the $n+1$ copies of $\statetruevarsup{j}{1}/\statetruevarsup{j}{-1}$ is revealed with probability $\frac{1}{2(n+1)}$, so that the probability of the agent seeing the state is
    \begin{equation}
        1 - \left(1 - \frac{1}{2(n+1)} \right)^{n+1} \leq \frac{1}{2}.
    \end{equation}
    The above inequality follows from the fact that $(1 - x)^r \geq 1 - rx$ for any $x \in [0,1], r \in \mathbb{N}$.
    Ultimately, this splitting procedure lets us simulate the reduction from Proposition~\ref{prop:MOSOMDPs} without requiring small probabilities.

    To handle the small threshold value $\nu$, we also add the component shown in Figure~\ref{fig:small_reward_fragment} to the start of the POMDP.
    In particular, the agent starts in this component in every environment, which guarantees it an expected reward of at least $\frac{1}{2m} \frac{1}{8^n}$.
    Adding this component to the POMDP then allows us to use $0$ as the threshold for the MO-SOMDP decision problem.

    We now argue that this new MO-SOMDP $(\calM'_i)_{i \in [m]}$ has non-negative reward if and only if the original QSAT instance is true.
    The $(\implies)$ direction follows the exact same logic as the proof of Proposition~\ref{prop:MOSOMDPs}.
    Meanwhile, the $(\impliedby)$ direction requires the following modifications.
    \begin{itemize}
        \item When defining the mapping $\xi$ we need to replace $\bot$ with $\underbrace{\bot,\ldots,\bot}_{n+1 \text{ times}}$.
        \item When reasoning about the probability of reaching the dead state, we need to observe that the probability of observing the sequence $\underbrace{\bot,\ldots,\bot}_{n+1 \text{ times}}$ is $\left(1 - \frac{1}{2(n+1)} \right)^{n+1} \geq \frac{1}{2}$ at each stage.
    \end{itemize}
\end{remark}

\begin{figure}[t]
    \centering
    \resizebox{0.55\columnwidth}{!}{\begin{tikzpicture}[auto,node distance=8mm,>=latex,font=\small]

    \tikzstyle{round}=[thick,draw=black,circle]
    \tikzstyle{round_sat}=[thick,draw=blue,circle]
    \tikzstyle{round_unsat}=[thick,draw=red,circle]
    \tikzstyle{bobble}=[thick,draw=black,fill=black,minimum size=0.7mm,inner sep=0pt,fill=black,circle]

    \node[round] at (0,0) (e1) {$\statedecvarsim{1}$};

    \node[round_sat] at (1.5,2) (t10) {$\statetruevarsup{1}{0}$};
    \node[round_sat] at (3,3) (t11) {$\statetruevarsup{1}{1}$};
    \node[round_sat] at (3,1) (t1n1) {$\statetruevarsup{1}{-1}$};
    \node[bobble] at (2.5,2) (t1b) {};

    \draw[dashed, blue, rounded corners] (2.25, 0.25) rectangle (3.75, 3.75);
    \node at (3,4.25) {{\color{blue} Reveal w.p. $1 - \frac{1}{2m}\frac{1}{8^n}$}};

    \node[round_unsat] at (1.5,-2) (f10) {$\statefalsevarsup{1}{0}$};
    \node[bobble] at (2.5,-2) (f1b) {};
    \node[round_unsat] at (3,-1) (f11) {$\statefalsevarsup{1}{1}$};
    \node[round_unsat] at (3,-3) (f1n1) {$\statefalsevarsup{1}{-1}$};

    \draw[dashed, red, rounded corners] (2.25, -0.25) rectangle (3.75, -3.75);
    \node at (3,-4.25) {{\color{red} Reveal w.p. $\nicefrac{1}{2}$}};

    \node[round] at (8,0) (u2) {$\statedecvarsim{2}$};
    \node[bobble] at (9,0) (bu2) {};

    \node[round_unsat] at (9.5,2) (t20) {$\statetruevarsup{2}{0}$};
    \node at (10.5,2) {$\cdots$};

    \node[round_sat] at (9.5,-2) (f20) {$\statefalsevarsup{2}{0}$};
    \node at (10.5,-2) {$\cdots$};

    \node[round] at (8,5) (sacc) {$s_{\mathsf{acc}}$};

    \path[->] 
    (e1) edge node[sloped, anchor=center, above] {$\actTrue$} (t10)
    (e1) edge node[sloped, anchor=center, above] {$\actFalse$} (f10);

    \path[-]
    (u2) edge (bu2);
    \path[->] 
    (bu2) edge node[anchor=center, left] {$\nicefrac{1}{2}$} (t20)
    (bu2) edge node[anchor=center, left] {$\nicefrac{1}{2}$} (f20);

    \path[-]
    (t10) edge  (t1b);
    \path[->] 
    (t1b) edge node[anchor=center, above left] {$\nicefrac{1}{2}$} (t11)
    (t1b) edge node[anchor=center, below left] {$\nicefrac{1}{2}$} (t1n1);

    \path[-]
    (f10) edge  (f1b);
    \path[->] 
    (f1b) edge node[anchor=center, above left] {$\nicefrac{1}{2}$} (f11)
    (f1b) edge node[anchor=center, below left] {$\nicefrac{1}{2}$} (f1n1);

    \path[->] 
    (f11) edge node[sloped, pos=0.7, above] {$\actFalse$} (u2)
    (f1n1) edge node[sloped, pos=0.7, above] {$\actTrue$} (u2)
    (t11) edge node[sloped, pos=0.7, above] {$\actFalse$} (u2)
    (t1n1) edge node[sloped, pos=0.7, above] {$\actTrue$} (u2);

    \path[->] 
    (f11) edge node[sloped, pos=0.6, above] {$\actTrue$} (sacc)
    (f1n1) edge node[sloped, pos=0.65, above] {$\actFalse$} (sacc)
    (t11) edge node[sloped, anchor=center, above] {$\actTrue$} (sacc)
    (t1n1) edge node[sloped, anchor=center, above] {$\actFalse$} (sacc);

\end{tikzpicture}}
    \caption{
        A fragment of an environment corresponding to the clause $(x_1 \vee \neg x_2)$ for the 
        MO-SOMDP constructed in the proof of Proposition~\ref{prop:MOSOMDPs}.
        Lines leaving nodes correspond to actions, and small dots encode stochastic resolution of the transition dynamics. 
        When an arrow leaves a node without a label, then all actions induce the same dynamics.
        Blue nodes correspond to satisfaction of the clause, while red nodes correspond to the clause not being satisfied yet.
    }
    \label{fig:mosomdp_default_proof}
\end{figure}
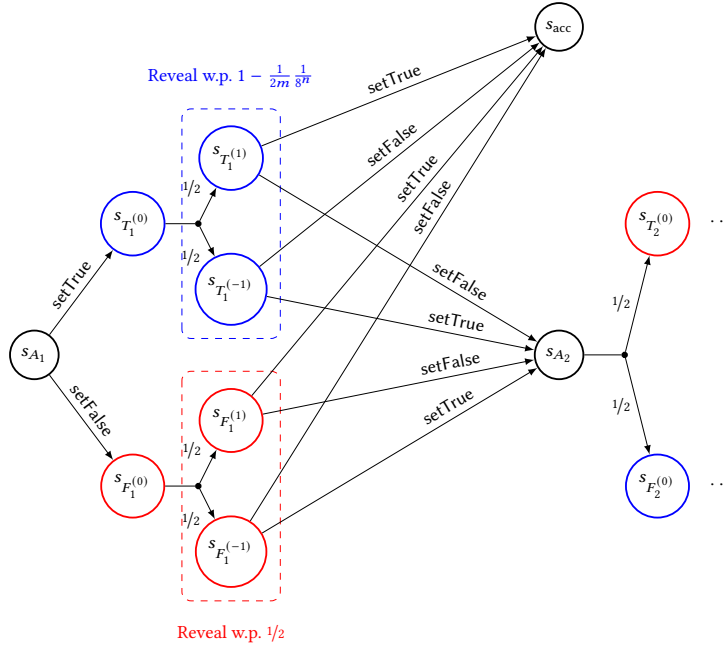

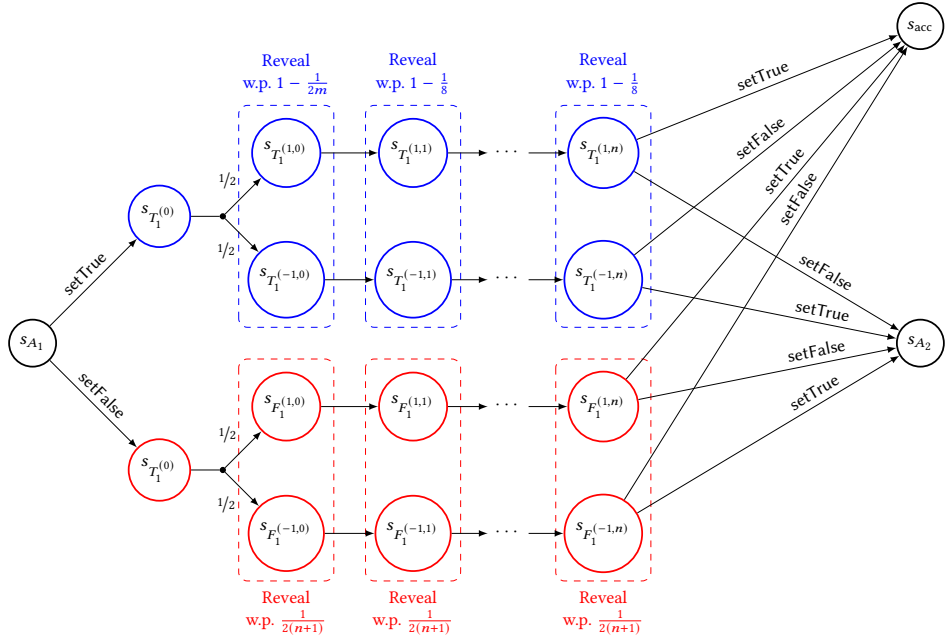
\begin{figure}[t]
    \centering
    \resizebox{0.7\columnwidth}{!}{\begin{tikzpicture}[auto,node distance=8mm,>=latex,font=\small]

    \tikzstyle{round}=[thick,draw=black,circle]
    \tikzstyle{round_sat}=[thick,draw=blue,circle]
    \tikzstyle{round_unsat}=[thick,draw=red,circle]
    \tikzstyle{bobble}=[thick,draw=black,fill=black,minimum size=0.7mm,inner sep=0pt,fill=black,circle]

    \def\horizrepsep{2}

    \node[round] at ({3 - 4.5*\horizrepsep},0) (e1) {$\statedecvarsim{1}$};

    \node[round_sat] at ({3 - 3.5*\horizrepsep},2) (t10) {$\statetruevarsup{1}{0}$};
    \node[bobble] at ({3 - 3*\horizrepsep},2) (t1b) {};

    \node[round_sat] at ({3 - 2.5*\horizrepsep},3) (t110) {$\statetruevarsup{1}{1,0}$};
    \node[round_sat] at ({3 - 1.5*\horizrepsep},3) (t111) {$\statetruevarsup{1}{1,1}$};
    \node at (3 - 0.75*\horizrepsep,3) (1tdots) {$\cdots$};
    \node[round_sat] at ({3},3) (t11) {$\statetruevarsup{1}{1,n}$};

    \node[round_sat] at ({3 - 2.5*\horizrepsep},1) (t1n10) {$\statetruevarsup{1}{-1,0}$};
    \node[round_sat] at ({3 - 1.5*\horizrepsep},1) (t1n11) {$\statetruevarsup{1}{-1,1}$};
    \node at (3 - 0.75*\horizrepsep,1) (1ntdots) {$\cdots$};
    \node[round_sat] at ({3},1) (t1n1) {$\statetruevarsup{1}{-1,n}$};

    \draw[dashed, blue, rounded corners] ({2.25 - 2.5*\horizrepsep}, 0.25) rectangle ({3.75 - 2.5*\horizrepsep}, 3.75);
    \node[align=center] at (3 - 2.5*\horizrepsep,4.25) {\color{blue} Reveal \\ \color{blue} w.p. $1 - \frac{1}{2m}$};

    \draw[dashed, blue, rounded corners] ({2.25- 1.5*\horizrepsep}, 0.25) rectangle ({3.75- 1.5*\horizrepsep}, 3.75);
    \node[align=center] at (3- 1.5*\horizrepsep,4.25) {\color{blue} Reveal \\ \color{blue} w.p. $1 - \frac{1}{8}$};

    \draw[dashed, blue, rounded corners] (2.25, 0.25) rectangle (3.75, 3.75);
    \node[align=center] at (3,4.25) {\color{blue} Reveal \\ \color{blue} w.p. $1 - \frac{1}{8}$};

    \node[round_unsat] at ({3 - 3.5*\horizrepsep},-2) (f10) {$\statetruevarsup{1}{0}$};
    \node[bobble] at ({3 - 3*\horizrepsep},-2) (f1b) {};

    \node[round_unsat] at ({3 - 2.5*\horizrepsep},-1) (f110) {$\statefalsevarsup{1}{1,0}$};
    \node[round_unsat] at ({3 - 1.5*\horizrepsep},-1) (f111) {$\statefalsevarsup{1}{1,1}$};
    \node at (3 - 0.75*\horizrepsep,-1) (1fdots) {$\cdots$};
    \node[round_unsat] at ({3},-1) (f11) {$\statefalsevarsup{1}{1,n}$};

    \node[round_unsat] at ({3 - 2.5*\horizrepsep},-3) (f1n10) {$\statefalsevarsup{1}{-1,0}$};
    \node[round_unsat] at ({3 - 1.5*\horizrepsep},-3) (f1n11) {$\statefalsevarsup{1}{-1,1}$};
    \node at (3 - 0.75*\horizrepsep,-3) (1nfdots) {$\cdots$};
    \node[round_unsat] at ({3},-3) (f1n1) {$\statefalsevarsup{1}{-1,n}$};

    \draw[dashed, red, rounded corners] ({2.25 - 2.5*\horizrepsep}, -0.25) rectangle ({3.75 - 2.5*\horizrepsep}, -3.75);
    \node[align=center] at (3 - 2.5*\horizrepsep,-4.25) {\color{red} Reveal \\ \color{red} w.p. $\frac{1}{2(n+1)}$};

    \draw[dashed, red, rounded corners] ({2.25- 1.5*\horizrepsep}, -0.25) rectangle ({3.75- 1.5*\horizrepsep}, -3.75);
    \node[align=center] at (3- 1.5*\horizrepsep,-4.25) {\color{red} Reveal \\ \color{red} w.p. $\frac{1}{2(n+1)}$};

    \draw[dashed, red, rounded corners] (2.25, -0.25) rectangle (3.75, -3.75);
    \node[align=center] at (3,-4.25) {\color{red} Reveal \\ \color{red} w.p. $\frac{1}{2(n+1)}$};

    \node[round] at (8,0) (u2) {$\statedecvarsim{2}$};
    \node[round] at (8,5) (sacc) {$s_{\mathsf{acc}}$};

    \path[->] 
    (e1) edge node[sloped, anchor=center, above] {$\actTrue$} (t10)
    (e1) edge node[sloped, anchor=center, above] {$\actFalse$} (f10);

    \path[-]
    (t10) edge  (t1b);
    \path[->] 
    (t1b) edge node[anchor=center, above left] {$\nicefrac{1}{2}$} (t110)
    (t1b) edge node[anchor=center, below left] {$\nicefrac{1}{2}$} (t1n10);

    \path[-]
    (f10) edge  (f1b);
    \path[->] 
    (f1b) edge node[anchor=center, above left] {$\nicefrac{1}{2}$} (f110)
    (f1b) edge node[anchor=center, below left] {$\nicefrac{1}{2}$} (f1n10);

    \path[->] 
    (f11) edge node[sloped, pos=0.7, above] {$\actFalse$} (u2)
    (f1n1) edge node[sloped, pos=0.7, above] {$\actTrue$} (u2)
    (t11) edge node[sloped, pos=0.7, above] {$\actFalse$} (u2)
    (t1n1) edge node[sloped, pos=0.7, above] {$\actTrue$} (u2);

    \path[->] 
    (f11) edge node[sloped, pos=0.6, above] {$\actTrue$} (sacc)
    (f1n1) edge node[sloped, pos=0.65, above] {$\actFalse$} (sacc)
    (t11) edge node[sloped, anchor=center, above] {$\actTrue$} (sacc)
    (t1n1) edge node[sloped, anchor=center, above] {$\actFalse$} (sacc);

    \path[->] 
    (t110) edge (t111)
    (t111) edge (1tdots)
    (1tdots) edge (t11);

    \path[->] 
    (t1n10) edge (t1n11)
    (t1n11) edge (1ntdots)
    (1ntdots) edge (t1n1);

    \path[->] 
    (f110) edge (f111)
    (f111) edge (1fdots)
    (1fdots) edge (f11);

    \path[->] 
    (f1n10) edge (f1n11)
    (f1n11) edge (1nfdots)
    (1nfdots) edge (f1n1);

\end{tikzpicture}}
    \caption{
        A fragment of an environment in the MO-SOMDP constructed during the ``splitting'' procedure described in Remark~\ref{rem:splitting_procedure}.
    }
    \label{fig:splitting_procedure}
\end{figure}

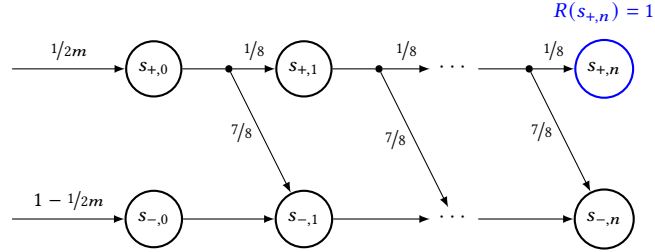
\begin{figure}[t]
    \centering
    \resizebox{0.5\columnwidth}{!}{\begin{tikzpicture}[auto,node distance=8mm,>=latex,font=\small]

    \tikzstyle{round}=[thick,draw=black,circle]
    \tikzstyle{round_sat}=[thick,draw=blue,circle]
    \tikzstyle{round_unsat}=[thick,draw=red,circle]
    \tikzstyle{bobble}=[thick,draw=black,fill=black,minimum size=0.7mm,inner sep=0pt,fill=black,circle]

    \def\nodesepscale{2}
    \def\vertnodesepscale{1}

    \node at ({-1*\nodesepscale},{\vertnodesepscale}) (pinit) {};
    \node[round] at (0,{\vertnodesepscale}) (p0) {$s_{+,0}$};
    \node[bobble] at ({0.5*\nodesepscale},{\vertnodesepscale}) (pb0) {};
    \node[round] at ({\nodesepscale},{\vertnodesepscale}) (p1) {$s_{+,1}$};
    \node[bobble] at ({1.5*\nodesepscale},{\vertnodesepscale}) (pb1) {};
    \node at ({2*\nodesepscale},{\vertnodesepscale}) (pdots) {$\cdots$};
    \node[bobble] at ({2.5*\nodesepscale},{\vertnodesepscale}) (pbn) {};
    \node[round_sat] at ({3*\nodesepscale},{\vertnodesepscale}) (pn) {$s_{+,n}$};
    \node[align=center] at ({3*\nodesepscale},{1.75*\vertnodesepscale}) {\color{blue} $R(s_{+,n}) = 1$};

    \node at ({-1*\nodesepscale},{-1*\vertnodesepscale}) (ninit) {};
    \node[round] at (0,{-1*\vertnodesepscale}) (n0) {$s_{-,0}$};
    \node[round] at ({\nodesepscale},{-1*\vertnodesepscale}) (n1) {$s_{-,1}$};
    \node at ({2*\nodesepscale},{-1*\vertnodesepscale}) (ndots) {$\cdots$};
    \node[round] at ({3*\nodesepscale},{-1*\vertnodesepscale}) (nn) {$s_{-,n}$};

    \path[->]
    (ninit) edge node[anchor=center, above] {$1 - \nicefrac{1}{2m}$} (n0)
    (pinit) edge node[anchor=center, above] {$\nicefrac{1}{2m}$} (p0);

    \path[->] 
    (pb0) edge node[anchor=center, left] {$\nicefrac{7}{8}$} (n1)
    (pb1) edge node[anchor=center, left] {$\nicefrac{7}{8}$} (ndots)
    (pbn) edge node[anchor=center, left] {$\nicefrac{7}{8}$} (nn)
    (pb0) edge node[anchor=center, above] {$\nicefrac{1}{8}$}  (p1)
    (pb1) edge node[anchor=center, above] {$\nicefrac{1}{8}$}  (pdots)
    (pbn) edge node[anchor=center, above] {$\nicefrac{1}{8}$}  (pn);

    \path[->] 
    (n0) edge (n1)
    (n1) edge (ndots)
    (ndots) edge (nn);

    \path[-]
    (p0) edge (pb0)
    (p1) edge (pb1)
    (pdots) edge (pbn);
\end{tikzpicture}}
    \caption{
        A component one can add to the start of a POMDP to increase the reward by $\frac{1}{2m}\frac{1}{8^n}$.
        The agent gets a reward of $1$ with small probability.
        From the states in the set $\{s_{-,n},s_{+,n}\}$, the agent transitions to the component shown in Figure~\ref{fig:splitting_procedure}.
    }
    \label{fig:small_reward_fragment}
\end{figure}

\subsection{Multi-Observation Intermittently-Observable MDPs with Fixed Support}
\propMOIOMDPs*

\begin{proof}

    First, consider the case where one of the IOMDPs is unobservable, i.e., $\lambda = 0$. 
    In this case, for all states $s$, $\supp (O_i(s)) = \{\bot_i\}$ where $\bot_i \in Z$ is the blind observation in environment $i$. As such, by the fixed support condition, there exists some observation $\bot$ such that
    \begin{equation}
        O_i(s) = \delta_\bot \quad \forall i \in [n], s \in S. 
    \end{equation}
    That is, a fixed support MO-IOMDP where one environment is unobservable is equivalent to a single BlindMDP.
    For the remainder of the proof, we assume that the revelation probability is non-zero in all environments.

    We first argue that each state's unique observation is the same across environments.
    As each environment $i$ is an IOMDP with non-zero revelation probability, we can write the support of the observation function at any state $s$ as
    \begin{equation}
        \supp (O_i(s)) = \{z_i^s,\bot_i\},
    \end{equation}
    where $z_i^s$ denotes the unique observation of state $s$ in environment $i$, and $\bot_i$ denotes the no-revelation observation in environment $i$. 
    By the fixed-support condition, for any state $s$ and pair of environments $i,j$, $\{z_i^s,\bot_i\} = \{z_j^s,\bot_j\}$.
    Thus, we can deduce that
    \begin{equation}
        \{\bot_i\} = \bigcap_{s\in S} \{z_i^s,\bot_i\} = \bigcap_{s\in S} \{z_j^s,\bot_j\} = \{\bot_j\},
    \end{equation}
    i.e., the no-revelation observation is the same in all environments.
    The first and last equalities follow from the fact that $\bot_i$ is the only observation in environment $i$ that is common to all states, while the middle equality follows from the fixed-support condition.
    Since the blind observation is common to all environments, we can then deduce also that $z_i^s = z_j^s$, i.e., each state's unique observation is constant across the environment.

    We can now apply the results of \cite{DBLP:journals/tac/ChenL26} which describe how the value of a policy in an IOMDP changes as we increase the revelation rate, without shuffling observations.
    In particular, we first recall the following monotonicity result from \cite{DBLP:journals/tac/ChenL26} for IOMDPs.
    \begin{lemma}[Adapted from Corollary $1$ in \cite{DBLP:journals/tac/ChenL26}]
        \label{lem:IOMDP_mono}
        Fix an underlying MDP $\calM'$, and let $\calM_\lambda$ denote the IOMDP constructed from $\calM'$ with revelation probability $\lambda$.
        Additionally, let $\pi_\lambda$ be an optimal policy for the belief-MDP corresponding to $\calM_\lambda$. 
        Finally, for an arbitrary policy $\pi$ and revelation probability $\nu$, let $V^\pi_\nu$ be the value of policy $\pi$ in IOMDP $\calM_\nu$. 
        Then, for any $u \geq \lambda$, we have
        \begin{equation*}
            V^{\pi_\lambda}_\lambda \leq V^{\pi_\lambda}_u.
        \end{equation*}
    \end{lemma}
    In words, Lemma~\ref{lem:IOMDP_mono} shows that there exists an optimal policy $\pi_\lambda$ for $\calM_\lambda$ such that increasing the revelation probability above $\lambda$ will increase the value that $\pi_\lambda$ attains.
    While the result in \cite{DBLP:journals/tac/ChenL26} is given for infinite-horizon IOMDPs, one can easily apply it to finite-horizon IOMDPs as well.
    For example, one can encode the finite-horizon IOMDP as an infinite-horizon IOMDP where 
    \begin{itemize}
        \item States are augmented to include a time component,
        \item The reward corrects the discount factor,
        \item After the agent has taken $H$ actions, the infinite-horizon IOMDP transitions to a dead state with zero reward.
    \end{itemize}

    We next argue that we can decide the MO-IOMDP problem in non-deterministic polynomial time.

    We first upper bound the value of a given MO-IOMDP instance.
    Without loss of generality, we assume that the IOMDPs are ordered such that $\lambda_1 \leq \lambda_2 \leq \ldots \leq \lambda_n$ where $\lambda_i$ is the revelation probability for $\calM_i$.
    By Theorem $1$ in \cite{DBLP:journals/tac/ChenL26}, the values of these IOMDPs reflect their revelation probabilities. That is,
    \begin{equation}
        V_{\calM_1}^* \leq V_{\calM_2}^* \leq \ldots 
    \end{equation}
    We can then easily upper bound the MO-IOMDP value by $V_{\calM_1}$. Indeed, for any policy $\pi$, we have
    \begin{equation}
        \min_i V^{\pi}_{\calM_i} \leq \min_i V^{*}_{\calM_i} = V^*_{\calM_1}.
    \end{equation}

    We now observe that an optimal belief-based policy for $\calM_1$ achieves this value. 
    Indeed, let $\pi_1$ be an optimal belief-based policy in $\calM_1$.
    Then, by Lemma~\ref{lem:IOMDP_mono}, we have
    \begin{equation}
        \min_i V^{\pi_1}_{\calM_i} \geq \min_i V^{\pi_1}_{\calM_1} = V^*_{\calM_1}.
    \end{equation}

    Thus, a MO-IOMDP with fixed support has value at least $B$ if and only if there exists a policy for the environment with the lowest revelation probability with value at least $B$, and this fact suggests Algorithm~\ref{alg:simplified_MO_IOMDP} as a simple method for solving the value problem in MO-IOMDPs. 
    In short, we can guess and evaluate a policy in the class of optimal IOMDP policies, and by the proof of Proposition~\ref{prop:SOMDPs_base}, we can implement this procedure in non-deterministic polynomial time.
    \begin{algorithm}[t]
    \caption{MO-IOMDP}
    \label{alg:simplified_MO_IOMDP}
    \begin{algorithmic}[1]
    \Require IOMDPs $\calM_1,\ldots,\calM_n$ that differ in revelation probability $\lambda$, along with a value threshold $B$.

    \State Determine the index $i^*$ of the IOMDP with the lowest revelation rate.
    \State Non-deterministically guess a policy $\pi$ in the class of optimal IOMDP policies
    \If{$ \mathbb{E}_{\calM_i}^\pi[\sum_{t=0}^{H-1} R(s_t,a_t)] \geq B$}
        \Return $\mathsf{YES}$
    \EndIf
    \end{algorithmic}
    \end{algorithm}
\end{proof}

\clearpage\newpage
\section{Proofs for Section~\ref{sec:restricted_pol_classes}}
\label{sec:restricted_pol_full_proofs}

\subsection{Time-Dependent Markovian Policies in ME-POMDPs}

\propMEBlindPOMDPsMemoryless*

\begin{proof}
    Fix a tensor $(a_{ijk})_{i \in [m], j \in [n], k \in [n]}$. This tensor will encode a $2$-vs.-$1$ game, where the adversary chooses action $i \in [m]$ and two players choose actions $j$ and $k$ independently according to a pair of mixed strategies.
    
    We build a ME-BlindMDP $(\calM_{i})_{i\in [n]}$ such that 
    \begin{equation}
        \max_{\pi \in \polMR} \min_{i \in [m]} V^\pi_{\calM_i} = \max_{\mathbf{x} \in \dist{[n]}, \mathbf{y} \in \dist{[n]}} \min_{i \in [m]} \sum_{j \in [n], k \in [n]} a_{ijk} 
        \vecat{\mathbf{x}}{j} \vecat{\mathbf{y}}{j}
    \end{equation}
    In particular, we define a deterministic BlindMDP $\calM_i$ as
    \begin{equation}
        \calM_i = (\underbrace{\{s_0\} \cup [n]}_{S}, \underbrace{[n]}_{A}, \underbrace{\{\bot\}}_{Z}, T, O, R_i, \underbrace{\delta_{s_0}}_{b}, \underbrace{2}_{H}),
    \end{equation}
    where we set $T$ and $O$ to satisfy
    \begin{equation}
        T(s_0,j) = j \ \ \forall j \in [n], \quad O(s) = \bot \ \ \forall s \in \{s_0\} \cup [n],
    \end{equation}
    and we set $R_i$ to satisfy
    \begin{align}
        & R_i(s_0,j) = 0 \quad \forall i \in [m], j \in [n], \\
        & R_i(j,k) = a_{ijk} \quad \forall i \in [m], j \in [n], k \in [n].
    \end{align}

    A time-dependent Markovian policy in this ME-BlindMDP comprises a pair $\pi_0$ and $\pi_1$ of distributions in $\dist{[n]}$, and this policy will have the following value.
    \begin{equation}
        \min_{i \in [m]} \mathbb{E}_{j \sim \pi_0, k \sim \pi_1, s \sim T(s_0,j)}[R_i(s, k)]
        = \min_{i \in [m]} \mathbb{E}_{j \sim \pi_0, k \sim \pi_1}[R_i(j,k)] = \min_{i \in [m]} \sum_{j \in [n], k \in [n]} a_{ijk} 
        \vecat{\pi_0}{j} \vecat{\pi_1}{k}.
    \end{equation}

    Thus, for any fixed rational $B$, the ME-POMDP has value at least $B$ in time-dependent Markovian policies if and only if the $\mathsf{maxmin}$ value of the $2$-vs.$1$ team game that $a_{ijk}$ defines is greater than $B$. 
    The $\ETR$-hardness of the latter problem is then given by Lemma 24 in \cite{DBLP:conf/lics/BriceHT26}. 
    Alternatively, one can deduce the $\ETR$-hardness of the $2$-vs.-$1$ game problem by observing that deciding the existence of a solution $\mathbf{x}, \mathbf{y} \in \dist{[n]}$ to the system of equations
    \begin{equation}
        \sum_{j\in [n], k \in [n]} a_{ijk} 
        \vecat{\mathbf{x}}{j}
        \vecat{\mathbf{y}}{k} = 0 \quad \forall i \in [m],
    \end{equation}
    is $\ETR$-hard, as detailed in \cite{DBLP:journals/mst/Hansen19}.
\end{proof}

\propDetMEMDPsMemoryless*

\begin{proof}

    Fix a tensor $(a_{ijk})_{i \in [m], j \in [n], k \in [n]}$. This tensor will encode a $2$-vs.-$1$ game, where the adversary chooses action $i \in [m]$ and two players choose actions $j$ and $k$ independently according to a pair of mixed strategies.
    Additionally let $\calL = \{0,\ldots,n-1\}$ and $\calI = [m]$.
    
    We build a ME-MDP $(\calM_{(i,l)})_{(i,l) \in \calI \times \calL}$ such that 
    \begin{equation}
        \label{eq:MEMDP_MR_equality}
        \max_{\pi \in \polMR} \min_{(i,l) \in \calI \times \calL} V^\pi_{\calM_{(i,l)}} = \max_{\mathbf{x} \in \dist{[n]}, \mathbf{y} \in \dist{[n]}} \min_{i \in [m]} \sum_{j \in [n], k \in [n]} a_{ijk} 
        \vecat{\mathbf{x}}{j}
        \vecat{\mathbf{y}}{k}.
    \end{equation}
    In particular, we define a deterministic ME-MDP $(\calM_{(i,l)})_{(i,l) \in \calI \times \calL}$ with
    \begin{equation}
        \calM_{(i,l)} = (\underbrace{\{s_0\} \cup \calL}_{S}, \underbrace{[n]}_{A}, T_{l}, R_{(i,l)}, \underbrace{\delta_{s_0}}_{b},\underbrace{2}_{H}).
    \end{equation}
    We use the trivial observation set and function, where each state emits a unique observation deterministically.
    We set $T_{l}$ to satisfy
    \begin{equation}
        T_{l}(s_0, j) = ((j - 1 + l) \mod n) \quad \forall j \in [n],
    \end{equation}
    so that the state at $t = 1$ encodes a shuffled version of the agent's initial action.
    We then set the reward function to encode the tensor $(a_{ijk})_{i \in [m], j \in [n], k \in [n]}$.
    \begin{align*}
        & R_{(i,l)}(s_0,j) = 0 & \forall i \in [m], l \in \calL, j \in [n], \\
        & R_{(i,l)}(j',k) = a_{i (((j' - l) \mod n) + 1) k} &\forall i \in [m], j' \in \calL, l \in \calL, k \in [n].
    \end{align*}

    We remark that, due to the above transition and reward structure, the agent gets $a_{ijk}$ when it plays actions $j$ and $k$ in environment $i$.
    Indeed, for any $i \in [m], j,k \in [n]$, and $l \in \calL$, we have
    \begin{equation}
        \label{eq:reward_unshuffle}
        R_{(i,l)}((j - 1 + l) \mod n, k) = a_{i j k}
    \end{equation}
    as, by basic modular arithmetic rules, 
    \begin{equation*}
        (((j - 1 + l) \mod n) - l \mod n) + 1 = j.
    \end{equation*}

    We now prove \eqref{eq:MEMDP_MR_equality}, and we first fix notation.
    A time-dependent Markovian policy in this MEMDP comprises two mappings $\pi_0$ and $\pi_1$ such that
    \begin{equation}
        \pi_t : \{s_0\} \cup \calL \rightarrow \dist{[n]} \text{ for } t \in \{0,1\}.
    \end{equation}
    By an abuse of notation, we will identify $\pi_0$ with $\pi_0(s_0)$ since the first state is always $s_0$. We will then use $\pi^{(j')}_1$ to denote $\pi_1(j') \in \dist{[n]}$ for $j' \in \calL$.

    To show the $\leq$ direction of \eqref{eq:MEMDP_MR_equality}, we show that nature can uniformly randomize over the $\calL$ component of the environment to ensure that the agent can not use state information to improve their reward.
    Fix a time-dependent Markovian randomized policy $(\pi_0,\pi_1)$, and define $\tilde{\pi}_1 \in \dist{[n]}$ as follows.
    \begin{equation}
        \tilde{\pi}_1 = \sum_{j' \in \calL} \frac{1}{n} \pi_1^{(j')}.
    \end{equation}
    We next evaluate the policy $(\pi_0,\pi_1)$ when nature uniformly randomizes over $\calL$. 
    That is, we evaluate $\sum_{l \in \calL} \frac{1}{n} V_{(i,l)}^\pi$ as follows.
    \begin{multline}
        \label{eq:shuffle_bound}
        \sum_{l \in \calL} \frac{1}{n} \mathbb{E}_{j \sim \pi_0, j' \sim T_l(s_0,j), k \sim \pi_1^{(j')}} [R_{(i,l)}(j',k)]
        = \sum_{l \in \calL} \frac{1}{n} \mathbb{E}_{j \sim \pi_0, k \sim \pi_1^{((j-1 + l) \mod n)}} [R_{(i,l)}((j-1 + l) \mod n,k)]
        \\ = \sum_{l \in \calL} \frac{1}{n} \mathbb{E}_{j \sim \pi_0 k \sim \pi_1^{((j-1 + l) \mod n)}} [a_{ijk}]
        \\ =  \sum_{j \in [n]} \vecat{\pi_0}{j} \sum_{k \in [n]} \sum_{l \in \calL} \frac{1}{n} \vecat{\pi_1^{(j-1 + l) \mod n}}{k} \ a_{ijk}
        = \sum_{j \in [n]} \vecat{\pi_0}{j} \sum_{k \in [n]} 
        \vecat{\tilde{\pi}_1}{k}
        a_{ijk}.
    \end{multline}
    The first equality above uses the transition dynamics definition.
    The second equality uses \eqref{eq:reward_unshuffle}.
    The third equality makes the expectation explicit and reorders the sums.
    The final equality uses the definition of $\tilde{\pi}_1$.
    Obviously, nature can do better than uniformly randomizing over $\calL$ by choosing the worst value of $l$, and so we have
    \begin{equation}
        \label{eq:MEMDP_MR_equality_leq_single_pol}
        \min_{i\in [m]} \min_{l \in \calL} V_{(i,l)}^\pi \leq \min_{i\in [m]} \sum_{l \in \calL} \frac{1}{n} V_{(i,l)}^\pi = \min_{i \in [m]} \sum_{j \in [n],k\in [n]} \vecat{\pi_0}{j}\vecat{\tilde{\pi}_1}{k} a_{ijk}
         \leq  \max_{\mathbf{x} \in \dist{[n]}, \mathbf{y} \in \dist{[n]}} \min_{i \in [m]} \sum_{j \in [n], k \in [n]} a_{ijk} \vecat{\mathbf{x}}{j}\vecat{\mathbf{y}}{k}.
    \end{equation}
    The $\leq$ of \eqref{eq:MEMDP_MR_equality} follows by observing that 
    \eqref{eq:MEMDP_MR_equality_leq_single_pol} holds for any time-dependent Markovian policy $\pi$.
    
    We now deduce the $\geq$ direction of \eqref{eq:MEMDP_MR_equality}.
    Fix $\mathbf{x} \in \dist{[n]}, \mathbf{y} \in \dist{[n]}$. We define $\pi_0$ and $\pi_1$ such that
    \begin{equation}
        \vecat{\pi_0}{j} = \vecat{\mathbf{x}}{j} \ \ \forall j \in [n] \quad \quad 
        \vecat{\pi_1^{(j')}}{k}
        = \vecat{\mathbf{y}}{k} \ \ \forall k \in [n], \forall j' \in \calL.
    \end{equation}
    We then observe that, by following similar steps to the $(\leq)$ direction, that for any $i \in [m]$, $l\in \calL$, we have
    \begin{equation*}
        \mathbb{E}_{j \sim \pi_0, j' \sim T_l(s_0,j) k \sim \pi_1^{j'}} [R_{(i,l)}(j',k)]
        = \mathbb{E}_{j \sim \mathbf{x}, k \sim \mathbf{y}} [R_{(i,l)}((j-1 + l) \mod n,k)]
        = \mathbb{E}_{j \sim \mathbf{x}, k \sim \mathbf{y}} [a_{ijk}]
        = \sum_{j \in [n], k \in [n]} a_{ijk} \vecat{\mathbf{x}}{j} \vecat{\mathbf{y}}{k},
    \end{equation*}
    from which we can deduce the $\geq$ direction of \eqref{eq:MEMDP_MR_equality}.

\end{proof}

\propMOPOMDPsMemoryless*

\begin{proof}
    Fix a tensor $(a_{ijk})_{i \in  [m], j \in [n], k \in [n]}$. We build a MO-POMDP such that 
    \begin{equation}
        \label{eq:main_mopomdp_iff}
        \max_{\pi \in \polMR} \min_{i \in [m]} V_{\calM_j}^\pi \geq 0
        \iff
        \max_{\mathbf{x} \in \dist{[n]}, \mathbf{y} \in \dist{[n]}} \min_{i \in [m]} \sum_{j \in [n], k \in [n]} a_{ijk} \vecat{\mathbf{x}}{j}\vecat{\mathbf{y}}{k} \geq 0 
    \end{equation}
    As in the proofs of the preceding propositions, we build an environment where the agent must choose their first two actions independently.
    However, we must overcome the fact that the adversary can not directly edit the environment.

    To overcome this difficulty, we build an environment in which the adversary reveals information to the agent conditional on the agent taking a certain action.
    In particular, after the agent has chosen their first two actions, an environment-dependent observation is emitted, and the agent then chooses an action $a^*$ as a function of this observation.
    Following this stage, the agent plays a guessing game in which they must guess the value of a coin flip.
    If the agent chose the right value for $a^*$, the agent sees the true coin value, whereas if the agent chose the wrong action, they see nothing, and the agent suffers a large penalty. 
    As such, the adversarial nature can control the environment dynamics by incentivizing the agent to choose certain actions conditional on the environment.
    
    We now formally define the MO-POMDP
    \begin{equation*}
        \calM_i = (S,A,Z,T,O_i,R,b,H).
    \end{equation*}
    We set the state space to have the following structure.
    \begin{align*}
        S = & \{s_0\} 
        \\& \cup  \left([n] \right)
        \\& \cup \left([n] \times [n] \right)
        \\& \cup \left([n] \times [n] \times [m]  \right)
        \\& \cup \left([n] \times [n] \times [m] \times \{\bot,\top\} \right)
    \end{align*}
    We then set the action space to 
    \begin{equation}
        A = [n] \cup [m] \cup \{1, 2\}, 
    \end{equation}
    and the observation space to
    \begin{equation}
        Z = \{z_j | j \in [m]\} \cup \{z_\top, z_\bot\} \cup \{z_{\mathsf{blind}}\}.
    \end{equation}
    We set $b = \delta_{s_0}$ so that the environment starts in $s_0$, and we set $H = 5$. 

    We design the transition functions to track the actions taken in each time step, and additionally encode a coin flip at the last stage.
    \begin{align*}
        & T(s_0,j) = \begin{cases}
            j & \text{ if } j \in [n] \\
            1 & \text{ otherwise }
        \end{cases} \\
        & T(j,k) =  \begin{cases}
            (j,k) & \text{ if } k \in [n] \\
            (j,1) & \text{ otherwise }
        \end{cases} 
         \\
        & T((j,k),i) = \begin{cases}
            (j,k,i) & \text{ if } i \in [m] \\
            (j,k,1) & \text{ otherwise }
        \end{cases}  \\
        & T((j,k,i), \tilde{a}) = \left\{(j,k,i,\top) \mapsto \frac{1}{2}, (j,k,i,\bot) \mapsto \frac{1}{2}\right\} \quad \forall \tilde{a} \in A.
    \end{align*}
    The ``otherwise'' cases effectively ensure that the agent always chooses integers from the right set at each stage. 

    We then design the observation function for environment $i$ such that, at the third stage, the agent sees an observation $z_i$ that is unique to that environment. 
    The transition function definition also allows the state to encode whether or not the agent played action $i$ in environment $i$.
    As such, we can define the observation function to reveal the coin flip result conditional on whether the agent played action $i$ at the third stage.
    In particular, for any $i \in [m]$, we define the observation function $O_i$ in environment $i$ as follows.
    \begin{align*}
        & O_i(s_0) = z_{\mathsf{blind}} \\
        & O_i(j) = z_{\mathsf{blind}} & \forall j \in [n] \\
        & O_i((j,k)) = z_i & \forall j,k \in [n] \\
        & O_i((j,k,l)) = z_{\mathsf{blind}} & \forall j,k \in [n], l \in [m] \\
        & O_i((j,k,i,\top)) = z_\top  & \forall j,k \in [n] \\
        & O_i((j,k,i,\bot)) = z_\bot  & \forall j,k \in [n] \\
         & O_i((j,k,l,\top)) = \left\{z_\top \mapsto \frac{1}{2}, z_\bot \mapsto \frac{1}{2}\right\} & \forall j,k \in [n], l \in ([m] \setminus \{i\})\\
         & O_i((j,k,l,\bot)) = \left\{z_\top \mapsto \frac{1}{2}, z_\bot \mapsto \frac{1}{2}\right\} & \forall j,k \in [n], l \in ([m] \setminus \{i\}) 
    \end{align*}
    Note that, in all cases, the observation function still takes as input a \emph{state}, but, for these POMDPs, the state may be a tuple of indices, hence the double-parentheses.
    
    Finally, the rewards encode 
    \begin{enumerate}
        \item The game defined by $a_{ijk}$, and
        \item The rewards from the matching problem.
    \end{enumerate}
    In particular, let $M =  \max_{ijk} |a_{ijk}| $ represent the highest magnitude value in $(a_{ijk})_{i \in [m], j,k \in [n]}$.
    We set $R$ as follows
    \begin{align*}
        & R(s_0, \tilde{a}) = 0 & \forall \tilde{a} \in A\\
        & R(j, \tilde{a}) = 0 & \forall j \in [n], \tilde{a} \in A\\
        & R((j,k), \tilde{a}) = 0 & \forall j,k \in [n], \tilde{a} \in A\\
        & R((j,k,i), \tilde{a}) = a_{ijk} & \forall j,k \in [n], i\in [m], \tilde{a} \in A\\
        & R((j,k,i,\top),\tilde{a}) = \begin{cases}
            0 & \text{ if } \tilde{a} = 1 \\
            - 3M & \text{ otherwise }
        \end{cases} & \forall j,k \in [n], i \in [m] \\
        & R((j,k,i,\bot),\tilde{a}) = \begin{cases}
            0 & \text{ if } \tilde{a} = 2 \\
            - 3M & \text{ otherwise }
        \end{cases}
        & \forall j,k \in [n], i \in [m] 
    \end{align*}

    We now argue that \eqref{eq:main_mopomdp_iff} holds.

    ($\implies$)
    A memoryless policy in this MO-POMDP comprises a sequence $(\pi_t)_{t\in \{0,1,2,3,4\}}$ of five mappings. Ignoring their behavior on observations that occur with zero probability, these policies take the form
    \begin{align*}
        & \pi_0 : \{ z_{\mathsf{blind}}\} \rightarrow \dist{A}\\
        & \pi_1 : \{ z_{\mathsf{blind}}\} \rightarrow \dist{A}\\
        & \pi_2 : \{ z_i \}_{i\in [m]} \rightarrow \dist{A} \\
        & \pi_3 : \{ z_{\mathsf{blind}} \} \rightarrow \dist{A} \\
        & \pi_4 : \{ z_\bot, z_\top \} \rightarrow \dist{A}
    \end{align*}
    By an abuse of notation, we will identify $\pi_0(z_{\mathsf{blind}})$ with $\pi_0$ and use similar notation for $\pi_1$.
    
    We assume that under these policies the agent receives a non-negative reward in all environments, and we prove that there exist $\mathbf{x}$ and $\mathbf{y}$ to certify that the right-hand side of  \eqref{eq:main_mopomdp_iff} holds.
    In particular, we define the following vectors where $\mathbf{x}, \mathbf{y} \in \dist{[n]}$, $\mathbf{v}^{(i)} \in \dist{[m]}$ for each $i \in [m]$, and $\mathbf{w}_\top, \mathbf{w}_\bot \in \dist{\{1,2,3\}}$.
    \begin{equation*}
        \vecat{\mathbf{x}}{j} = \begin{cases}
            \vecat{\pi_0}{1} + \sum_{j' \in A \setminus [n]} \vecat{\pi_0}{j'} & \text{ if } j = 1 \\
            \vecat{\pi_0}{j} & \text{ otherwise.}
        \end{cases}
    \end{equation*}
    \begin{equation*}
        \vecat{\mathbf{y}}{k} = \begin{cases}
            \vecat{\pi_1}{1} + \sum_{k' \in A \setminus [n]} \vecat{\pi_1}{k'} & \text{ if } k = 1 \\
            \vecat{\pi_1}{k} & \text{ otherwise.}
        \end{cases}
    \end{equation*}
    \begin{equation*}
        \vecat{\mathbf{v}^{(i)}}{i'} = \begin{cases}
            \vecat{\pi_2(z_i)}{1} + \sum_{l \in A \setminus [m]} \pi_2(z_i)[l] & \text{ if } i' = 1 \\
            \vecat{\pi_2(z_i)}{i'}  & \text{ otherwise.}
        \end{cases}
    \end{equation*}
    \begin{align*}
        & \vecat{\mathbf{w}_\top}{1} = 
        \vecat{\pi_4(z_\top)}{1}
        & \vecat{\mathbf{w}_\bot}{1} = \vecat{\pi_4(z_\bot)}{1} \\
        & \vecat{\mathbf{w}_\top}{2} = \vecat{\pi_4(z_\top)}{2} 
        & \vecat{\mathbf{w}_\bot}{2} = \vecat{\pi_4(z_\bot)}{2} \\
        & \vecat{\mathbf{w}_\top}{3} = \sum_{l \in A \setminus \{1,2\}}\vecat{\pi_4(z_\top)}{l}
        & \vecat{\mathbf{w}_\bot}{3} = \sum_{l \in A \setminus \{1,2\}}\vecat{\pi_4(z_\bot)}{l}
    \end{align*}

    We now fix an index $i \in [m]$, and argue that
    \begin{equation}
        \label{eq:env_i_nonneg_mopomdp}
        \sum_{j \in [n], k \in [n]} a_{ijk} \vecat{\mathbf{x}}{j} \vecat{\mathbf{y}}{k} \geq 0.
    \end{equation}

    In particular, we evaluate the reward in each environment.
    When evaluating the reward, we can ignore the state at timesteps $0,1$, and $2$.
    The expected reward in environment $i$ at time $3$ is given by 
    \begin{equation}
        \mathbb{E}_{\calM_i}^\pi[R(s_3,a_3)] = \sum_{i' \in [m]} 
        \vecat{\mathbf{v}^{(i)}}{i'} \sum_{j \in [n], k \in [n]} a_{i'jk} \vecat{\mathbf{x}}{j} \vecat{\mathbf{y}}{k}. 
    \end{equation}
    Indeed, the probability of visiting state $(j,k,i')$ at $t = 3$ is $\vecat{\mathbf{x}}{j} \vecat{\mathbf{y}}{k} \vecat{\mathbf{v}^{(i)}}{i'}$ by definition of the transition dynamics.
    We can then upper bound this step's reward as follows.
    \begin{multline}
         \mathbb{E}_{\calM_i}^\pi[R(s_3,a_3)] 
         = 
         \vecat{\mathbf{v}^{(i)}}{i} \sum_{j \in [n], k \in [n]} a_{ijk} \vecat{\mathbf{x}}{j} \vecat{\mathbf{y}}{k}
         + \sum_{i' \in [m] \setminus \{i\}} \vecat{\mathbf{v}^{(i)}}{i'} \sum_{j \in [n], k \in [n]} a_{i'jk} \vecat{\mathbf{x}}{j} \vecat{\mathbf{y}}{k}
         \\\leq  \vecat{\mathbf{v}^{(i)}}{i} \sum_{j \in [n], k \in [n]} a_{ijk} \vecat{\mathbf{x}}{j} \vecat{\mathbf{y}}{k}
         + \sum_{i' \in [m] \setminus \{i\}} \vecat{\mathbf{v}^{(i)}}{i'} M
         \\ = \vecat{\mathbf{v}^{(i)}}{i} \sum_{j \in [n], k \in [n]} a_{ijk} \vecat{\mathbf{x}}{j} \vecat{\mathbf{y}}{k}
         + (1 - \vecat{\mathbf{v}^{(i)}}{i}) M
    \end{multline}
    In words, the above inequality chain upper bounds the benefit the agent can get in its reward at $t=3$ by not matching the environment-specific observation.

    The expected reward at time $4$ is given by 
    \begin{multline*}
        \mathbb{E}_{\calM_i}^\pi[R(s_4,a_4)] = \vecat{\mathbf{v}^{(i)}}{i}
        \left( \frac{1}{2} (1 - \vecat{\mathbf{w}_\top}{1})  + \frac{1}{2} (1 - \vecat{\mathbf{w}_\top}{2}) \right)(-3M)
        \\+ 
        (1 - \vecat{\mathbf{v}^{(i)}}{i}) 
        \left( \frac{1}{4} (1 - \vecat{\mathbf{w}_\top}{1}) + \frac{1}{4} (1 - \vecat{\mathbf{w}_\bot}{1}) + \frac{1}{4} (1 - \vecat{\mathbf{w}_\top}{2}) + \frac{1}{4} (1 - \vecat{\mathbf{w}_\bot}{2}) \right)(-3M).
    \end{multline*}
    Indeed, the first component of the sum handles the case where the agent chooses $i$ to match the environment, and the observation reveals the final state.
    Meanwhile, the second component handles the case where the observation provides no information about the final state.

    We can then upper bound the expected reward at time $4$ as follows.
    \begin{equation*}
        \mathbb{E}_{\calM_i}^\pi[R(s_4,a_4)] \leq (1 - \vecat{\mathbf{v}^{(i)}}{i}) 
         \left(\frac{1}{4} (1 - \vecat{\mathbf{w}_\top}{1}) + \frac{1}{4} (1 - \vecat{\mathbf{w}_\bot}{1}) +  \frac{1}{4} (1 - \vecat{\mathbf{w}_\top}{2}) + \frac{1}{4} (1 - \vecat{\mathbf{w}_\bot}{2}) \right)(-3M)
        \leq  -\frac{3M}{2}  (1 - \vecat{\mathbf{v}^{(i)}}{i})
    \end{equation*}
    Note that, in the second inequality, we use the fact that $1 - \vecat{\mathbf{w}_\top}{1} + 1- \vecat{\mathbf{w}_\top}{2} \geq 1$.
    
    We now upper bound the total reward.
    \begin{equation*}
        \mathbb{E}_{\calM_i}^\pi[R(s_3,a_3) + R(s_4,a_4)]
        \leq 
        \vecat{\mathbf{v}^{(i)}}{i} \sum_{j \in [n], k \in [n]} a_{ijk} \vecat{\mathbf{x}}{j} \vecat{\mathbf{y}}{k}
        -\frac{M}{2}  (1 - \vecat{\mathbf{v}^{(i)}}{i}).
    \end{equation*}

    We can combine this upper bound with the assumption that the reward is non-negative in environment $i$ to deduce that \eqref{eq:env_i_nonneg_mopomdp} holds.
    First, consider the case where $M = 0$. 
    If $M = 0$, then $a_{ijk} = 0$ for all $i,j,k$ and thus $\mathbf{x}$ and $\mathbf{y}$ trivially satisfy \eqref{eq:env_i_nonneg_mopomdp}.
    Now, assume that $M > 0$, and observe that $\vecat{\mathbf{v}^{(i)}}{i}$ must be positive, as otherwise the expected reward is negative.
    Using the fact that the expected reward is non-negative, we deduce that
    \begin{equation}
        \vecat{\mathbf{v}^{(i)}}{i} \sum_{j \in [n], k \in [n]} a_{ijk} \vecat{\mathbf{x}}{j} \vecat{\mathbf{y}}{k} 
        \geq 
        \frac{M}{2} (1 - \vecat{\mathbf{v}^{(i)}}{i}) \geq 0.
    \end{equation}
    Finally, using the fact that $\vecat{\mathbf{v}^{(i)}}{i}$ is positive, we conclude that $\sum_{j \in [n], k \in [n]} a_{ijk} \vecat{\mathbf{x}}{j} \vecat{\mathbf{y}}{k} \geq 0$.

    ($\impliedby$)
    We now assume that some $\mathbf{x}, \mathbf{y} \in \dist{[n]}$ exist satisfying the inequalities $\sum_{j\in [n], k \in [n]} a_{ijk} \mathbf{x}[j] \mathbf{y}[k] \geq 0$ for all $i$ in $[m]$. 
    We then define the following policy
    \begin{equation}
        \begin{aligned}
            & \pi_0(z_{\mathsf{blind}}) = \mathbf{x} \\
            & \pi_1(z_{\mathsf{blind}}) = \mathbf{y} \\
            & \pi_2(z_i) = \delta_i \\
            & \pi_4(z_\top) = \delta_1 \\
            & \pi_4(z_\bot) = \delta_2.
        \end{aligned}
    \end{equation}
    Note that we abuse notation as $\pi_0(z_{\mathsf{blind}}) \in \dist{A}$ while $\mathbf{x} \in \dist{[n]}$ and we may have $[n] \subsetneq A$. 
    However, we can easily extend $\pi_0$ as required.
    We leave $\pi_3$ undefined as it does not affect the transition dynamics.
    
    Fix an environment $i$. 
    The reward at $t = 3$ satisfies
    \begin{equation}
        \mathbb{E}_{\calM_i}^\pi[R(s_3,a_3)] 
         = 
          \sum_{j \in [n], k \in [n]} a_{ijk} \vecat{\mathbf{x}}{j} \vecat{\mathbf{y}}{k} \geq 0
    \end{equation}
    where the last inequality holds by assumption. 
    Meanwhile, the reward at $t=4$ is zero as $\vecat{\mathbf{v}^{(i)}}{i} = 1$, and $\vecat{\mathbf{w}_\top}{1} = \vecat{\mathbf{w}_\bot}{2} = 1$, where $\mathbf{w}$ and $\mathbf{v}^{(i)}$ are defined as above. 
\end{proof}

\lemMEPOMDPsMemoryless*

\begin{proof}

    A time-dependent Markovian policy that attains value at least $B$ exists if and only if there exists a solution $\mu_i,\nu_i,\pi$ to the following system of bilinear inequalities and equalities.
    \begin{equation}
        \begin{aligned}
                & \nu_i(s,0) = b_i(s) 
                & \forall s \in S, i \in [n],\\
                & \mu_i(s,a, t) = \nu_i(s,t) \sum_{z \in Z} O_i(s)(z) \pi_t(z,a) 
                & \forall i \in [n], s \in S, a \in A, t \in \{0,\ldots,H-1\},\\
                & \nu_i(s,t+1) = \sum_{s' \in S, a \in A} \mu_i(s',a,t) T_i(s',a)(s) 
                & \forall s \in S, t \in \{0,\ldots, H-2\}, i \in [n],\\
                & \sum_{a} \pi_t(z,a) = 1 
                & \forall z \in Z, t \in \{0,\ldots,H-1\},\\
                & \pi_t(z,a) \geq 0 
                & \forall z \in Z, t \in \{0,\ldots,H-1\}, a \in A,\\
                & \sum_{t \in \{0,\ldots, H-1\}} \sum_{s \in S} \sum_{a \in A} \mu_i(s,a,t) R_i(s,a) \geq B 
                & \forall i \in [n].\\
        \end{aligned}
    \end{equation}
    For each $i$, the tensor $\nu \in \mathbb{R}^{S \times \{0,\ldots, H-1\}}$ encodes the probability of visiting each state at each time, and $\mu \in \mathbb{R}^{S \times A \times \{0,\ldots, H-1\}}$ encodes the probability of visiting each state and taking each action at each time.
    Meanwhile, $\pi_t(z,a)$ encodes the probability of taking action $a$ upon seeing observation $z$ at time $t$.
    Since we can express the value problem as such a system, inclusion in $\ETR$ immediately follows. 
\end{proof}

\subsection{Finite-State Controllers in ME-BlindMDPs}
In this section, we prove the following theorem regarding the complexity of policy synthesis for FSCs.

\thmMEBlindPOMDPsFSCs*

We first give a series of Lemmas on the link between non-negative matrix factorization and ME-POMDPs. We then give the final proof of Theorem~\ref{thm:alpha_FSC_complexity} in Section~\ref{sec:final_proof_component}.

\subsubsection{Supporting results for Theorem~\ref{thm:alpha_FSC_complexity}}

\begin{definition}[Non-negative matrix factorization \cite{DBLP:journals/siamjo/Vavasis09}]
    A non-negative matrix $P \in \mathbb{R}^{m\times n}$ has a non-negative factorization with rank $k$ if $P$ can be written as the sum of $k$ non-negative rank-one matrices. 
\end{definition}

\begin{lemma}
    \label{lem:low_rank_FSC_corr}
    A joint action distribution on $(a_0,a_1)$ in a BlindMDP can be implemented by a $|Q|$-state time-dependent FSC if and only if the distribution has a non-negative factorization with fewer than $|Q|$ elements.
\end{lemma}
\begin{proof}

    In the case of a BlindMDP, the FSC update rule at time $t$ is equivalent to two mappings
    \begin{equation}
        \gamma_t : Q \rightarrow \dist{A} \quad \eta_t : Q \times A \rightarrow \dist Q
    \end{equation}

    ($\implies$) 
    Fix an FSC with $|Q|$ states.
    We can write the joint distribution of the actions $(a_0,a_1)$ at the first two timesteps as follows
    \begin{equation*}
        p(a_0,a_1) = \sum_{q_1 \in Q} \sum_{q_0 \in Q} p(q_0) p(a_0|q_0) p(q_1| a_0,q_0) p(a_1|q_1, q_0, a_0)
        = \sum_{q_1 \in Q} \sum_{q_0 \in Q} p(q_0) p(a_0|q_0) p(q_1| a_0,q_0) p(a_1|q_1).
    \end{equation*}
    Note that in the second equality, we use the fact that the distribution of the action $a_1$ only depends on the memory state $q_1$ at $t = 1$.
    Now, let $\vecat{P}{a_0,q_1}$ denote the sum $\sum_{q_0 \in Q} p(q_0) p(a_0|q_0) p(q_1| a_0,q_0)$, and let $\vecat{W}{a_1,q_1} = p(a_1| q_1)$.
    Note that $P$ and $W$ encode two matrices in $\mathbb{R}^{|A| \times |Q|}$ with non-negative entries.
    We can now write the joint distribution as follows
    \begin{equation}
        p(a_0,a_1) = \sum_{q_1 \in Q} 
        \vecat{P}{a_0, q_1}
        \vecat{W}{a_1, q_1} = 
        \vecat{(PW^\top)}{a_0, a_1},
    \end{equation}
    from which we can deduce that the joint distribution of $a_0$ and $a_1$ has a non-negative factorization with $|Q|$ components.

    ($\impliedby$) 
    Conversely, let $P \in \RR^{|A| \times |A|}$ be a non-negative matrix such that the entry-wise $1$-norm satisfies $||P||_{1,1} = 1$ and $P$ has a factorization
    \begin{equation}
        P = \sum_{l \in [k]} \mathbf{u}^l (\mathbf{v}^l)^\top,
    \end{equation}
    where $\mathbf{u}^l$ and $\mathbf{v}^l$ are non-negative vectors for each $l \in [k]$.
    We can assume \wlogg that each component matrix is non-zero so that $||\mathbf{u}^{l}||_1 > 0$ and $||\mathbf{v}^{l}||_1 > 0$.
    Next, observe that
    \begin{equation*}
        1 = ||P||_{1,1} = \sum_{i \in [|A|], j \in [|A|]} \left|\sum_{l \in [k]} 
        \vecat{\mathbf{u}^l}{i}\vecat{\mathbf{v}^l}{j}
        \right|
    = \sum_{i \in [|A|], j \in [|A|]} \sum_{l \in [k]} |\vecat{\mathbf{u}^l}{i}| |\vecat{\mathbf{v}^l}{j}|
        = \sum_{l \in [k]} ||\mathbf{u}^{l}||_1 ||\mathbf{v}^{l}||_1
    \end{equation*}
    The first, second, and last equalities follow by definition.
    The third equality uses the non-negativity of the elements of $\mathbf{u}^l$ and $\mathbf{v}^l$.
    We can then write $P$ as
    \begin{equation}
        P = \sum_{l \in [k]} ||\mathbf{u}^{l}||_1 ||\mathbf{v}^{l}||_1 \frac{\mathbf{u}^l}{ ||\mathbf{u}^{l}||_1} \frac{\left(\mathbf{v}^l\right)^\top}{||\mathbf{v}^{l}||_1}.
    \end{equation}
    That is, we can write $P$ as a convex combination of $k$ non-negative rank $1$ matrices $\mathbf{w} \mathbf{z}^\top$ where $\mathbf{w}$ and $\mathbf{z}$ are distributions.
    We can then implement this as the following FSC with $Q = [k]$.
    In particular, for each $l \in [k]$, we have
    \begin{align*}
        & \iota(l) = ||\mathbf{u}^{l}||_1 ||\mathbf{v}^{l}||_1, \\
        & \eta_0(l,a)(l') = 
        \begin{cases}
            1 & \ l = l' \\
            0 & \text{ otherwise}
        \end{cases} \quad \forall l' \in [k], a \in A,\\
        & \gamma_0(l)  = \frac{\mathbf{u}^l}{ ||\mathbf{u}^{l}||_1}, \\
        & \gamma_1(l) = \frac{\mathbf{v}^l}{ ||\mathbf{v}^{l}||_1}.
    \end{align*}
    Note that we abuse notation slightly by implicitly identifying distributions on the action set with distributions on $[|A|]$.
\end{proof}

We can prove the hardness of finding such a time-dependent FSC by appealing to non-negative factorization results.
In particular, we will construct a ME-POMDP where any policy that gets a non-negative reward must have a joint distribution on $a_0$ and $a_1$ that satisfies a set of constraints, and due to the FSC constraint, this joint distribution will need to have low non-negative rank.

A naive route to prove $\ETR$-hardness would be to apply the $\ETR$-hardness of non-negative matrix factorization \cite{shitov2016universality,DBLP:journals/corr/abs-2407-18006}. However, this approach is insufficient and does not immediately lead to the desired result.
A non-negative matrix factorization instance comprises a matrix $W$ and a rank parameter $r$, and asks whether the matrix $W$ has a non-negative factorization with $r$ components.
We can easily encode this decision problem as a ME-POMDP problem by designing a ME-POMDP instance such that the agent gets a non-negative reward if and only if each entry of the joint distribution on $a_0$ and $a_1$ matches an appropriate entry of $W/||W||_{1,1}$ where $||W||_{1,1}$ is the entry-wise $1$-norm.
Clearly, the matrix $W$ has an $r$-component non-negative factorization if and only if there exists an $r$-memory-state FSC in the ME-POMDP that attains reward at least zero.
However, this reduction does not provide information about the ratios between the number of memory-states $|Q|$ and the number of environments $n$ for which the FSC-synthesis problem is $\ETR$-hard.
In particular, we need to carefully examine the values of $W$ and $r$ for which the non-negative factorization problem is $\ETR$-hard to establish the ratios $|Q|/n $ for which the FSC-synthesis problem is $\ETR$-hard in ME-POMDPs.

Ultimately, adapting the $\ETR$-hardness proof in \cite{DBLP:journals/corr/abs-2603-19976} for the related problem of \emph{non-negative Gram feasibility} provides an easy route to carefully control the parameters of these $W$ and $r$ matrices. We now present this analysis.

We will reduce from the following ETR-complete problem.
\begin{definition}[ETR-INV \cite{DBLP:journals/corr/abs-1912-08674,DBLP:conf/stoc/AbrahamsenAM18}]
    An ETR-INV instance comprises a set of variables $x_1, \ldots ,x_n$, and a set of constraints $C_1, \ldots, C_m$, where each $C_j$ is a constraint of the form
    \begin{equation}
        x + y = z, \quad x\cdot y = 1, \quad x = 1
    \end{equation}
    for $x,y,z \in \{x_1,\ldots,x_n\}$.
    The ETR-INV problem asks if there exists a solution to these constraints with $x_i \in [1/2,2]$.
\end{definition}

We will also use the following ETR-complete satisfaction problem during the proof.
The ETR with addition, multiplication, and inequalities (ETR-AMI) problem appears in \cite{DBLP:journals/corr/abs-1912-08674} without the non-negativity constraints. However, we refer to the version used in \cite{DBLP:journals/corr/abs-2603-19976}, where all variables in the solution must be non-negative.
\begin{definition}[ETR-AMI]
     An ETR-AMI instance comprises a set of variables $x_1, \ldots ,x_n$, and a set of constraints $C_1, \ldots, C_m$, where each $C_j$ is a constraint of the following form
    \begin{equation}
        x + y = z, \quad x\cdot y = z, \quad x = 1
    \end{equation}
    for $x,y,z \in \{x_1,\ldots,x_n\}$.
    The ETR-AMI problem asks if there exists a solution to these constraints with $x_i \in [0,\infty)$.
\end{definition}

We first present the following lemma, which shows that we can convert an ETR-INV instance into an ETR-AMI instance with a solution in which the variables lie in the interval $[0,2]$.
This lemma is highly similar to those presented in \cite{DBLP:journals/corr/abs-1912-08674}, and we present it to smooth exposition.
\begin{lemma}[Adapted from \cite{DBLP:journals/corr/abs-1912-08674}]
    \label{lem:etrami_etrinv_conv}
    Given an ETR-INV instance with $n$ variables and $m$ constraints, there exists an ETR-AMI instance such that
    \begin{enumerate}
        \item The ETR-AMI instance has a solution if and only if the ETR-INV instance has a solution.
        \item The ETR-AMI instance has $3n + 3$ variables, and $3 + 2n + m$ constraints (ignoring the implicit non-negativity constraints).
        \item The solution of the ETR-AMI instance satisfies $x_i \in [0,2]$.
    \end{enumerate}
\end{lemma}

\begin{proof}
    Let $x_i$ be the variables of the original ETR-INV instance.
    For each $x_i$ we add two auxiliary variables $\bar{x}_i$ and $\underbar{x}_i$, along with variables $\llbracket1\rrbracket, \llbracket2\rrbracket$ and $\llbracket1/2\rrbracket$ to represent constants.
    We can then encode the original ETR-INV problem as the following ETR-AMI problem.
    \begin{equation}
    \begin{aligned}
        & \llbracket1\rrbracket = 1 \\
        & \llbracket1\rrbracket + \llbracket1\rrbracket = \llbracket2\rrbracket \\
        & \llbracket2\rrbracket \llbracket1/2\rrbracket = \llbracket1\rrbracket \\
        & x_i + \bar{x}_i = \llbracket2\rrbracket \quad \forall i \in [n]\\
        & x_i = \llbracket1/2\rrbracket + \underbar{x}_i \quad \forall i \in [n]\\
        & (C_j)_{j\in [m]} \\
        & \text{All variables non-negative.}
    \end{aligned}
    \end{equation}
    Note that we encode any constraint $xy = 1$ in the ETR-INV instance as $xy = \llbracket1\rrbracket$ in the ETR-AMI instance.
    
    Clearly, any solution to the original ETR-INV instance defines a feasible solution to this ETR-AMI instance. 
    The converse also holds, as any assignment to $x_i$ in the ETR-AMI problem must satisfy $x_i \in [1/2,2]$ and satisfy the constraints $(C_j)_{j\in[m]}$.
\end{proof}

Following \cite{DBLP:journals/corr/abs-2603-19976}, we show that ETR-AMI problems can be converted to constrained non-negative matrix factorization. 
We provide a slighlty different proof from \cite{DBLP:journals/corr/abs-2603-19976} as
\begin{enumerate}
    \item We don't constrain the factorization to be symmetric, and
    \item We define constraints differently so that we can tightly control the number of factorization constraints we add.
\end{enumerate}

\begin{lemma}[Adapted from \cite{DBLP:journals/corr/abs-2603-19976}]
    \label{lem:ETR_nonneg_fact_main}

    Fix an integer $r \geq 2$, and an ETR-AMI instance with $n$ variables and $m$ constraints.

    There exists a pair of integers $(f_A,f_B)$, and two sets of matrices $A^{(1)},\ldots, A^{(f_A)}$, $B^{(1)},\ldots, B^{(f_B)}$ in $\mathbb{R}^{n + r, n + r}$ such that the ETR-AMI instance has a solution if and only if there exists a solution $W \in \mathbb{R}^{n + r, n + r}$ to
    \begin{align}
        & \langle A^{(\iota)}, W \rangle = 0 \quad \forall \iota \in [f_A] \\
        & \langle B^{(\kappa)}, W \rangle = 1 \quad \forall \kappa \in [f_B]
    \end{align}
    where $W$ has a non-negative factorization with $r$ components.
    Additionally, $f_A + f_B \leq r + n + m + 1$ and every entry in the matrices $(A^{(\iota)}))_{\iota \in [f_A]}$ and $(B^{(\kappa)})_{\kappa \in [f_B]}$ is in $\{-1,0,1,2\}$.
\end{lemma}
\begin{proof}
    We first add the constraints
    \begin{equation}
        \label{eq:id_constraint}
        \vecat{W}{i,j} = \begin{cases}
            1 & \ i = j\\
            0 & \text{ otherwise}
        \end{cases} \quad \forall i\in [r], j \in [r].
    \end{equation}
    These constraints serve the same purpose as the anchor constraints in \cite{DBLP:journals/corr/abs-2603-19976}. 
    Indeed, suppose $W$ has the following non-negative factorization, where $U_{(1)}, V_{(1)} \in \RR^{r \times r}$, $U_{(2)} \in \RR^{n \times r}$, and $V_{(2)} \in \RR^{r \times n}$.
    \begin{equation}
         W = \begin{bmatrix}
             U_{(1)} \\
             U_{(2)}\\
            \end{bmatrix}
             \begin{bmatrix}
             V_{(1)} &
            V_{(2)}
            \end{bmatrix}.
    \end{equation}
    The above factorization implies that $U_{(1)} V_{(1)} = I$, where $I$ is the $r \times r$ identity matrix. 
    We can then use the following fact, which generalizes Lemma~5.5 in \cite{DBLP:journals/corr/abs-2603-19976}.
    \begin{lemma}[\cite{ding2014matrix}]
        If $A$ is non-negative and has a non-negative inverse, then $A = D P$, where $D$ is a diagonal matrix with positive diagonal entries, and $P$ is a permutation matrix. 
    \end{lemma}
    We can thus write $U_{(1)}$ and $V_{(1)}$ in the following forms, where $(j_1,\ldots,j_n)$ is a permutation of $[n]$ that encodes the permutation matrix $P$, and $a_1,\ldots,a_n$ is a sequence of positive real numbers.
    \begin{equation}
        U_{(1)} = 
        \begin{bmatrix}
            a_1\unit{j_1}^\top \\
            \vdots \\
            a_n\unit{j_n}^\top 
        \end{bmatrix}
        \quad 
        V_{(1)} = 
        \begin{bmatrix}
            \frac{1}{a_1}\unit{j_1} &
            \hdots &
            \frac{1}{a_n}\unit{j_n}
        \end{bmatrix}.
    \end{equation}  

    We next add constraints to define the variables. In particular, we add
    \begin{equation}
    \label{eq:var_correct_constraints}
    \begin{aligned}
        &\vecat{W}{r + i, j} = 0 \quad \forall j \in \{2,\ldots,r\}, \forall i \in [n] \\
        & \vecat{W}{j, r + i} = 0 \quad \forall j \in \{2,\ldots,r\}, \forall i \in [n] \\
        & \vecat{W}{r + i, 1} = \vecat{W}{1, r+ i} \quad \forall i \in [n]. \\
    \end{aligned}
    \end{equation}
    The first constraint enforces that only the $j_1^{\text{th}}$ column of $U_{(2)}$ be non-zero.
    Similarly, the second constraint enforces that only the $j_1^{\text{th}}$ row of $V_{(2)}$ be non-zero.
    That, is $U_{(2)}$ and $V_{(2)}$ have the following forms in any valid factorization, where $(q_i)_{i\in [n]}$ and $(y_i)_{i \in [n]}$ are non-negative real numbers.
    \begin{equation}
        U_{(2)} = 
        \begin{bmatrix}
            q_1\unit{j_1}^\top \\
            q_2\unit{j_1}^\top \\
            \vdots \\
            q_n\unit{j_1}^\top 
        \end{bmatrix}
        \quad 
        V_{(2)} = 
        \begin{bmatrix}
            y_1\unit{j_1} &
            y_2\unit{j_1} &
            \hdots &
            y_n\unit{j_1} 
        \end{bmatrix}
    \end{equation}
    The final constraint in \eqref{eq:var_correct_constraints} enforces that, for each $i \in [n]$, we have
    \begin{equation}
        a_1 y_i = \left\langle a_1 \unit{j_1}, y_i\unit{j_1} \right\rangle = \vecat{W}{1, r+i} = \vecat{W}{r+i, 1} = \left\langle \frac{1}{a_1} \unit{j_1}, q_i\unit{j_1} \right\rangle = \frac{1}{a_1} q_i.
    \end{equation}
    With these constraints, we can then observe that, for any $i,j \in [n]$, we have
    \begin{equation}
        \vecat{W}{r + i, r + j} = \langle q_i \unit{j_1}, y_j\unit{j_1} \rangle = q_i y_j = \frac{1}{a_1} q_i a_1 y_j = a_1 y_i a_1 y_j.
    \end{equation}
    Ultimately, the value $a_1 y_i = \frac{1}{a_1} q_i$ will represent variable $x_i$. 
    The constraints in \eqref{eq:var_correct_constraints} take the same role as the constraints used in \cite{DBLP:journals/corr/abs-2603-19976} to encode variables.

    We can finally add the constraints to encode the ETR-AMI problem. In particular, as in \cite{DBLP:journals/corr/abs-2603-19976}, for each constraint of the form $x_i + x_j = x_k$, we add the constraint
    \begin{equation}
        \label{eq:etr_ami_add_constraint}
        \vecat{W}{r + i,1} + \vecat{W}{r + j,1} = \vecat{W}{r + k,1}.
    \end{equation}
    For any constraint of the form $x_i x_j = x_k$ we add 
    \begin{equation}
        \label{eq:etr_ami_mult_constraint}
        \vecat{W}{r + i, r + j} = \vecat{W}{r + k,1}.
    \end{equation}
    For a constraint of the form $x_i = 1$, we add
    \begin{equation}
        \label{eq:etr_ami_unity_constraint}
        \vecat{W}{r+i, 1} = 1.
    \end{equation}
    We have now specified all constraints.

    We now count the number of constraints, omitting the entry-wise zero constraints.
    In particular, for the constraints in \eqref{eq:id_constraint}, we have $r$ non-zero equality constraints for the diagonal entries.
    For the constraints in \eqref{eq:var_correct_constraints}, we have $n$ non-zero-entry equality constraints, one for each variable.
    Finally, for each of the $m$ constraints in the ETR-AMI instance, we add a linear constraint on $W$.
    Overall, the following table summarises these constraints.
    \begin{center}
        \begin{tabular}{ll}
        \hline
        Source & Number of constraints \\
        \hline
        \eqref{eq:id_constraint}  & $r$  \\
        \eqref{eq:var_correct_constraints}  & $n$  \\
        \eqref{eq:etr_ami_add_constraint},\eqref{eq:etr_ami_mult_constraint},\eqref{eq:etr_ami_unity_constraint}  & $m$  \\
        \hline
        \end{tabular}
    \end{center}
    Each of the constraints counted above is linear, and we can represent it in one of the following forms.
    \begin{equation}
        \langle A^{(\iota)}, W \rangle = 0 \quad \langle B^{(\kappa)}, W \rangle = 1.
    \end{equation}
    By inspecting the constraints, we can see that each entry in these matrices is in $\{-1,0,1,2\}$.
    Additionally, let $S \subset [r + n] \times [r + n]$ denote the indices of matrix entries which we constrain to be zero in $W$, and let $D^{S}$ be the indicator of these entries, i.e.,  
    \begin{equation}
        D^S = \sum_{(i,j) \in S} \unit{i}\unit{j}^\top.
    \end{equation}
    Since $W$ is non-negative, we have $\langle D^S,W \rangle = 0$ if and only if $\vecat{W}{i,j} = 0$ for all $(i,j) \in S$. 
    Indeed, 
    \begin{equation}
        \langle D^S,W \rangle = 0 \iff \sum_{(i,j) \in S} \vecat{W}{i,j} = 0 \iff \vecat{W}{i,j} = 0 \ \ \forall (i,j) \in S.
    \end{equation}
    The second $\iff$ holds as $W$ must be non-negative.
    Note also that $D^S$ has binary entries.
    Thus, we can encode the above constraints with matrices $A^{(1)},\ldots,A^{(f_A)}$, $B^{(1)},\ldots, B^{(f_B)}$ so that $f_A + f_B \leq r + n + m + 1$.

    We now argue that the ETR-AMI instance has a solution if and only if there exists a matrix $W$ with an $r$-non-negative rank factorization that satisfies the specified constraints. 

    ($\impliedby$)
    Suppose the constraints have a solution $W$ with non-negative factorization $UV$ such that $U \in \RR^{n + r\times r}$ and $V \in \RR^{r \times n + r}$.
    Now, set $x_i$ to $\vecat{W}{r+i,1} = a_1 y_i$ for each $i$, noting that, by assumption $a_1 y_i$ is non-negative. 
    We argued earlier in the proof that $\vecat{W}{r+i,r+j} = a_1y_ia_1 y_j$, so we can also deduce that $\vecat{W}{r+i,r+j} = x_i x_j$. 
    We can now conclude that, since we imposed the constraints in \eqref{eq:etr_ami_add_constraint}, \eqref{eq:etr_ami_mult_constraint}, and \eqref{eq:etr_ami_unity_constraint}, the values $x_1,\ldots,x_n$ must satisfy the constraints of the ETR-AMI instance.
    Thus, we can conclude that a solution to the ETR-AMI problem exists.
    
    ($\implies$) 
    Suppose the ETR-AMI instance has a solution $x_1,\ldots, x_n$, and let $\mathbf{x}$ denote the vector where $\mathbf{x}[i] = x_i$ for each $i \in [n]$.
    Then, the following matrix $W$ is a solution to the factorization problem.
    \begin{equation*}
        W = \begin{bmatrix}
            1& \mathbf{0}_{r-1}^\top\\
            \mathbf{0}_{r-1} & I_{r-1\times r-1} \\
            \mathbf{x} & \mathbf{0}_{n \times r-1}
        \end{bmatrix}
        \begin{bmatrix}
            1 & \mathbf{0}_{r-1}^\top & \mathbf{x}^\top \\
            \mathbf{0}_{r-1} & I_{r-1 \times r-1} & \mathbf{0}_{r-1 \times n} 
        \end{bmatrix}
        = \begin{bmatrix}
                 1& \mathbf{0}_{r-1}^\top  & \mathbf{x}^\top \\
                \mathbf{0}_{r-1} & I_{r-1\times r-1} & \mathbf{0}_{r-1 \times n} \\ 
                \mathbf{x} & \mathbf{0}_{n \times r-1} & \mathbf{x} \mathbf{x}^\top
             \end{bmatrix}
    \end{equation*}
    Indeed, $W$ has a non-negative factorization with $r$ components, and it trivially satisfies the constraints in \eqref{eq:id_constraint} and \eqref{eq:var_correct_constraints}.
    Furthermore, because $x_1,\ldots,x_n$ represents a solution to the ETR-AMI problem, the constraints \eqref{eq:etr_ami_add_constraint}, \eqref{eq:etr_ami_mult_constraint} and \eqref{eq:etr_ami_unity_constraint} are satisfied.
\end{proof}

We now connect Lemmas~\ref{lem:etrami_etrinv_conv} and \ref{lem:ETR_nonneg_fact_main} to demonstrate that we can reduce ETR-INV to a non-negative factorization problem with an additional simplex constraint on the factorized matrix. 
\begin{lemma}
    \label{lem:etrinv_to_ineq}
    Given an ETR-INV instance with $n$ variables and $m$ constraints, and a fixed integer $r \geq 2$, 
    there exists a sequence of matrices $C^{1},\ldots,C^{(f)}$ in $\RR^{r + n' \times r + n'}$
    such that the ETR-INV instance is feasible if and only if there exists a matrix $X \in \RR^{r + n' \times r + n'}$ such that
    \begin{itemize}
        \item $X$ has non-negative entries,
        \item $||X||_{1,1} = 1$, i.e., the entries of $X$ sum to $1$,
        \item $\langle C^{(\iota)}, X\rangle \geq 0$ for all $\iota \in [f]$,
        \item $X$ has a non-negative factorization with rank $r$,
    \end{itemize}
    where { $f \leq 5n + m + r + 10$ and $n' = 3n + 3$.}
\end{lemma}

\begin{proof}
    Fix an ETR-INV instance with $n$ variables and $m$ constraints. 
    Next, build an ETR-AMI instance with $n'$ variables and $m'$ constraints according to Lemma~\ref{lem:etrami_etrinv_conv}, and build the associated matrices  $A^{(1)},\ldots,A^{(f_A)}$, $B^{(1)},\ldots, B^{(f_B)}$ according to Lemma~\ref{lem:ETR_nonneg_fact_main}.
    The ETR-AMI instance will have $n' := 3n + 3$ variables, and $m' := 3 + 2n + m$ constraints, while $f_A$ and $f_B$ satisfy $f_A + f_B \leq n' + m' + r + 1 = 3n + 3 + 3 + 2n + m + r + 1 = 5n + m + r + 7$.
    
    Additionally, assume w.l.o.g. that $B^{(1)}$ encodes the constraint $\vecat{W}{1,1} = 1$. 

    We build the following linear inequality system from the above matrices.
    \begin{equation}
    \label{eq:simplex_AB_version}
    \begin{aligned}
        & \langle A^{(\iota)} - A^{(\iota+1)}, X \rangle \geq 0 & \forall \iota \in \{1,\ldots, f_A - 1\} \\
        & \langle A^{(f_A)} - A^{(1)}, X \rangle \geq 0\\
        & \langle A^{(1)}, X \rangle \geq 0\\
        & \langle -A^{(1)}, X \rangle \geq 0\\
        & \langle B^{(\kappa)} - B^{(\kappa+1)}, X \rangle \geq 0 & \forall \kappa \in \{1,\ldots, f_B - 1\}\\
        & \langle B^{(f_B)} - B^{(1)}, X \rangle \geq 0 \\
        & \left\langle B^{(1)} -  \frac{1}{2({ r + 4n + 4n^2})} \mathbf{1} \mathbf{1}^\top, X \right\rangle \geq 0
    \end{aligned}
    \end{equation}
    Note that this system involves
    $f_A + f_B + 3$ inequalities of the form $\langle C, X \rangle \geq 0$.

    We claim that \eqref{eq:simplex_AB_version} has a solution with $X$ such that
    \begin{itemize}
        \item $X \geq 0$,
        \item $||X||_{1,1} = 1$, and
        \item $X$ has a non-negative factorization with $r$ components.
    \end{itemize}
    if and only if the original ETR-INV problem has a solution.

    \newcommand{\smallposentry}{\nu}
    ($\implies$)
    Suppose $X$ is a solution to \eqref{eq:simplex_AB_version} with the desired properties. We immediately deduce that it must satisfy
    \begin{equation}
        \label{eq:AB_equations_delta}
        \begin{aligned}
            & \langle A^{(\iota)}, X \rangle = 0 & \forall \iota \in \{1,\ldots, f_A\} \\
            & \langle B^{(\kappa)}, X \rangle = \smallposentry & \forall \kappa \in \{1,\ldots, f_B\} \\
        \end{aligned}
    \end{equation}
    for some $\smallposentry \in \mathbb{R}$.
    Additionally, the final constraint implies, along with the fact that $||X||_{1,1} = 1$, implies that
    \begin{equation}
        \langle B^{(1)}, X \rangle \geq \left\langle\frac{1}{2({ r + 4n + 4n^2})} \mathbf{1} \mathbf{1}^\top, X \right\rangle = \frac{1}{2({ r + 4n + 4n^2})} > 0,
    \end{equation}
    which in turn implies that, in fact, $\smallposentry > 0$.
    Now, define $\tilde{X} = X/\smallposentry$. $\tilde{X}$ is now a matrix that satisfies \eqref{eq:AB_equations_delta} with $\smallposentry = 1$, and $\tilde{X}$ also has an $r$-component non-negative factorization. However, by Lemma~\ref{lem:ETR_nonneg_fact_main} and Lemma~\ref{lem:etrami_etrinv_conv}, we can now deduce that the original ETR-INV instance is feasible.

    ($\impliedby$)
    Suppose the ETR-INV instance has a solution. 
    As a result, the constructed ETR-AMI instance has a solution $x_1,\ldots,x_n$ where $x_i \in [0,2]$ for every $i$ by Lemma~\ref{lem:etrami_etrinv_conv}. 
    As in the proof of Lemma~\ref{lem:ETR_nonneg_fact_main}, we can now deduce that the following matrix $W$ is a solution to \eqref{eq:AB_equations_delta} with $\smallposentry = 1$, where we use $\mathbf{x}$ to denote the vector with the ETR-AMI solution as its elements.
    \begin{equation*}
        W = \begin{bmatrix}
            1& \mathbf{0}_{r-1}^\top\\
            \mathbf{0}_{r-1} & I_{r-1\times r-1} \\
            \mathbf{x} & \mathbf{0}_{n \times r-1}
        \end{bmatrix}
        \begin{bmatrix}
            1 & \mathbf{0}_{r-1}^\top & \mathbf{x}^\top \\
            \mathbf{0}_{r-1} & I_{r-1 \times r-1} & \mathbf{0}_{r-1 \times n} 
        \end{bmatrix}
        = \begin{bmatrix}
                 1&  \mathbf{0}_{r-1}^\top & \mathbf{x}^\top \\
                \mathbf{0}_{r-1} & I_{r-1\times r-1} & \mathbf{0}_{r-1 \times n} \\ 
                \mathbf{x} & \mathbf{0}_{n \times r-1} & \mathbf{x} \mathbf{x}^\top
             \end{bmatrix}
    \end{equation*}
    We need to construct a scaling of $W$ that lies in the simplex and satisfies \eqref{eq:simplex_AB_version}.
    By counting the elements of $W$ and using the bound on the value of the elements of $\mathbf{x}$, we have
    \begin{equation}
        ||W||_{1,1} \leq ||I_{r\times r}||_{1,1} + 2||\mathbf{x}||_1 + ||\mathbf{x}\mathbf{x}^\top||_{1,1} \leq r + 4n + 4 n^2.
    \end{equation}
    Also, note that $||W||_{1,1} > 0$ as $W$ is non-zero.
    Thus, we can define $X = W/||W||_{1,1}$ and observe that
    \begin{equation*}
        \langle B^{(1)} ,X\rangle = \vecat{X}
        {1,1} = 1/||W||_{1,1} \geq 1/ (r + 4n + 4 n^2) \geq 1/ 2(r + 4n + 4 n^2).
    \end{equation*}
    Furthermore, by linearity, $X$ satisfies \eqref{eq:AB_equations_delta} for some $\smallposentry$, from which it follows that $X$ is a matrix that satisfies \eqref{eq:simplex_AB_version}.
    Also, because $W$ has an $r$-component non-negative factorization, so does $X$, and trivially, $X$ satisfies $||X||_{1,1} = 1$.
\end{proof}

\subsubsection{Proof of Theorem~\ref{thm:alpha_FSC_complexity}}
\label{sec:final_proof_component}

    Fix a rational $\alpha \in (0,1)$. 

    Fix an ETR-INV instance with $n$ variables and $m$ constraints, and construct an integer $r$ such that  
    \begin{equation}
        r = \left\lceil \frac{\alpha}{1 - \alpha} ({ 5n + m + 10}) \right\rceil,
    \end{equation}
    where, for a real number $u$, $\lceil u \rceil$ is the smallest integer greater than or equal to $u$. 
    Note that, since $\alpha$ is fixed, $r$ is bounded by $\mathcal{O}(m + n)$, and the time to compute $r$ is bounded by a polynomial in the representation size of $\alpha$, which is fixed, and the size of the original ETR-INV instance.

    Next, build the system of constraints $C^{(1)} \ldots C^{(f)}$ specified in Lemma~\ref{lem:etrinv_to_ineq}, where $f \leq 5n + m + r + 10$.
    We will build a ME-POMDP with an environment for each $\iota \in [f]$. 
    
    Note that we have
    \begin{equation}
        \label{eq:rfalpha_conversion_FSC}
        r \geq \frac{\alpha}{ 1 - \alpha} (5n + m + 10)
        \implies (1 - \alpha) r \geq \alpha (5n + m  + 10)
        \implies r \geq \alpha (5n + m + r + 10) \geq \alpha f.
    \end{equation}

    From this stage, we follow the same construction as in the proof of Proposition~\ref{prop:mepomdp_blind_timedep}. The only modifications we make are that
    \begin{itemize}
        \item The set of environments is $[f]$. 
        \item In environment $\iota$, the reward at state $(j,k)$ is $R_\iota(j,k) = \vecat{C^{(\iota)}}{j,k}$. 
    \end{itemize}
    
    Note that, in environment $\iota$, the reward is
    \begin{equation}
        \mathbb{E}_{(j,k) \sim P^\pi}[R_\iota(j,k)] = \langle P^\pi, C^{(\iota)}\rangle,
    \end{equation}
    where $P^\pi$ is the joint distribution on $a_0$ and $a_1$ induced by the time-dependent finite-state controller $\pi$. 

    We now claim that the value of this ME-POMDP is non-negative if and only if the ETR-INV instance has a solution.

    ($\implies$)
    Suppose a time-dependent FSC policy $\pi$, with $r$ memory states exists such that
    \begin{equation}
         \langle P^\pi, C^{(\iota)}\rangle \geq 0 \quad \forall \iota \in [f]. 
    \end{equation}
    By Lemma~\ref{lem:low_rank_FSC_corr}, $P^\pi$ is a joint-distribution matrix with an $r$-component non-negative factorization.
    Then, by Lemma~\ref{lem:etrinv_to_ineq}, the original ETR-INV instance has a solution.

    ($\impliedby$)
    If the ETR-INV instance has a solution, then by Lemma~\ref{lem:etrinv_to_ineq}, there exists a matrix $X$ with an $r$-component non-negative factorization such that $X$ has non-negative entries, $||X||_{1,1} = 1$, and
    \begin{equation}
         \langle X, C^{(\iota)}\rangle \geq 0 \quad \forall \iota \in [f]. 
    \end{equation}
    By Lemma~\ref{lem:low_rank_FSC_corr}, there exists an $r$-state time-dependent FSC policy such that $P^\pi = X$, and thus there exists a policy $\pi$ such that
    \begin{equation}
        \mathbb{E}_{(j,k) \sim P^\pi}[R_\iota(j,k)] \geq 0 \quad \forall \iota \in [f].
    \end{equation}
    
    Finally, we claim that the ME-POMDP 
    \begin{enumerate}
        \item Lies in the specified $\alpha$-decision class, and
        \item Can be constructed in polynomial time from the given ETR-INV instance.
    \end{enumerate}
    The constructed POMDP has $f$ environments, and we look for an FSC with $r$ states. 
    It then follows from \eqref{eq:rfalpha_conversion_FSC} that the ratio of memory-states to environments is greater than $\alpha$.

    For the second claim, we first note that the ETR-INV to ETR-AMI instance conversion can obviously be done in polynomial time.
    We next note that $r \leq 1 + \frac{\alpha}{1 - \alpha} ({ 5n + m + 10})$, which means the  $A^{(\iota)}$ and $B^{(\kappa)}$ matrices have polynomial size in the original ETR-INV instance size.
    We can clearly translate the ETR-AMI instance into these $A^{(\iota)}$ and $B^{(\kappa)}$ in polynomial time, as writing each matrix requires at most looking up data from the ETR-AMI instance and translating it to these matrices as described in the proof of Lemma~\ref{lem:ETR_nonneg_fact_main}.
    Building all but one of the $C^{(\iota)}$ matrices described in Lemma~\ref{lem:etrinv_to_ineq} then simply requires basic addition and subtraction between some $A^{(\iota)}$ and $B^{(\kappa)}$ matrices, all of which have values in $\{-1,0,1,2\}$. 
    The remaining $C^{(\iota)}$ matrix is of the form
    \begin{equation}
        \unit{1}\unit{1}^\top - \frac{1}{2({ r + 4n + 4n^2})} \mathbf{1}\mathbf{1}^\top,
    \end{equation}
    which can clearly be built in polynomial time.
    Finally, once we have constructed the $C^{(\iota)}$ matrices, we can build the POMDP in polynomial time.
    Indeed, the POMDP has size defined by the following inequalities
    \begin{itemize}
        \item $|S| = 1 + r + n' \leq 1 + (3n + 3) + 1 + \frac{\alpha}{1 - \alpha} ({ 5n + m + 10})$.
        \item $|A| = |S| - 1$.
        \item $|Z| = 1$.
        \item The number of environments is $f \leq 5n + m + r + 10 \leq 5n + m + 11 + \frac{\alpha}{1 - \alpha} ({ 5n + m + 10})$.
        \item The number of memory states satisfies $r \leq f$.
    \end{itemize}
    We can clearly write the transition and observation functions in polynomial time, as all transitions/observation are deterministic.
    To encode the reward functions, we simply need to copy the $C^{(\iota)}$ matrices into the reward function as appropriate.
$\square$

\subsection{Finite-State Controllers in General ME-POMDPs}

\propMEPOMDPsFSCsAlpha*

\begin{proof}
    Fix a tensor $(a_{ijk})_{i\in [m], j \in [n], k\in [n]}$. 
    Furthermore, assume that matrices come in pairs with opposite signs. 
    That is,
    \begin{equation}
        \label{eq:sign_flip_assump}
        \forall i \in [m] \ \exists i' \in [m]: \quad
        a_{ijk} = - a_{i'jk} \quad \forall j\in [n], k \in [n].
    \end{equation}
    Indeed, this class of tensors can encode the problem of deciding if there exists a solution $\mathbf{x},\mathbf{y} \in \dist{[n]}$ such that
    \begin{equation}
        \sum_{j\in [n], k \in [n]} b_{ijk} \vecat{\mathbf{x}}{j}\vecat{\mathbf{y}}{k} = 0 \quad \forall i \in [m]
    \end{equation}
    for some tensor $b_{ijk}$, and this problem is $\ETR$-complete \cite{DBLP:journals/mst/Hansen19}.

    We build a ME-POMDP $(\calM_i)_{i \in [m]}$ such that there exists an FSC-based policy with non-negative reward in all environments if and only if there exists $\mathbf{x},\mathbf{y} \in \dist{[n]}$ such that
    \begin{equation}
        \label{eq:n_state_FSC_ETR_tensor_ineq}
        \sum_{j,k \in [n]} a_{ijk} \vecat{\mathbf{x}}{j}\vecat{\mathbf{y}}{k} \geq 0 \quad \forall i \in [m]. 
    \end{equation}

    In particular, we define POMDP $\calM_i$, for $i \in [m]$, according to
    \begin{equation}
        \calM_{i} = (S, \underbrace{[n] \cup [m]}_{A}, \underbrace{[m] \cup \{\bot\}}_{Z} , T, O, R_i, b, \underbrace{4}_{H}).
    \end{equation}
    We define the state space as
    \begin{equation}
        S = \{s_0\} \cup ([n] \times [m] \times \{1,2,3\}).
    \end{equation}
    We define the transition function in a similar manner to the proof of Proposition~\ref{prop:memoryless_MEMDP}, but we additionally add a randomly drawn component of the environment.
    In particular, we define $T$ as follows.
    \begin{equation*}
        \begin{aligned}
            &T(s_0,j) = \left\{(j,1,1) \mapsto \frac{1}{m}, \ldots, (j,m,1) \mapsto \frac{1}{m} \right\} & \forall j \in [n] \\
            &T(s_0,j) = \left\{(1,1,1) \mapsto \frac{1}{m}, \ldots, (1,m,1) \mapsto \frac{1}{m} \right\} & \forall j \notin [n] \\
            &T((j,l,1), a) = (j,l,2) & \forall j\in [n], l \in [m], a \in A \\
            &T((j,l,2), a) = (j,l,3) & \forall j\in [n], l \in [m], a \in A 
        \end{aligned}
    \end{equation*}
    We then define the observation function such that the random component of the environment is perfectly observed at time $1$, i.e., 
    \begin{equation*}
        \begin{aligned}
            & O(s_0) = \bot\\
            & O((j, l, 1)) = l & \forall j\in [n], l \in [m] ,\\
            & O((j, l, 2)) = \bot & \forall j\in [n], l \in [m],\\
            & O((j, l, 3)) = \bot & \forall j\in [n], l \in [m].
        \end{aligned}
    \end{equation*}
    We define the reward function at states in $[n] \times [m] \times 2$ so that the agent must guess the random component of the state at $t = 2$, 
    \begin{equation}
        R_i((j,l,2), a) = \begin{cases}
            0 & \text{ if } a = l, \\
            -1 & \text{ otherwise,}
        \end{cases}
    \end{equation}
    Meanwhile, the reward function at states in $[n] \times [m] \times 3$ encodes the tensor. That is
    \begin{equation}
        R_i((j,l,3), k) = \begin{cases}
            a_{ijk} & \text{ if } k \in [n], \\
            a_{ij1} & \text{ otherwise .}
        \end{cases}
    \end{equation}
    We set rewards for all other states to be zero for all actions.

    We now argue that this ME-POMDP has a time-dependent-FSC-based policy with $m$ states if and only if there exists $\mathbf{x},\mathbf{y} \in \dist{[n]}$ such that \eqref{eq:n_state_FSC_ETR_tensor_ineq} holds.

    ($\implies$)
    Fix a policy $\pi$. 
    Since all environments share transition and observation dynamics, the joint distribution on the memory state $q_2$ at time $2$, and the observation at time $1$ is constant across the environments. 

    We define the matrix $P \in \mathbb{R}^{m \times m}$ to encode the distribution of $q_2$ conditioned on each possible observation at time $1$. That is,
    \begin{equation}
        \vecat{P}{k,j} := \mathbb{P}^\pi_{\calM_i}[q_2 = k| z_1 = j] \quad \forall k,j \in [m],
    \end{equation}
    where the definition of $P$ does not depend on the choice of environment $i$.
    Note that, by definition of the observations, the observation $z_1$ at time $1$ is $l$ if and only if the state at time $1$ lies in the set $S_{l} \coloneq \{(j, l, 1)| j \in [n]\}$.
    
    At time $2$, the state lies in the set $[n] \times [m] \times \{2\}$, and so the observation is $\bot$. As such, the relationship between the action and memory-state at time $2$ on any feasible path is determined by a matrix $F \in \mathbb{R}^{|A| \times m}$ where $\vecat{F}{k,l}$ is the probability of taking action $k$ conditioned on the memory state being $l$ at time $2$.  
    We use $F_{\mathsf{trunc}}$ to denote the matrix formed by restricting $F$ to its first $m$ rows, so that it becomes a square matrix.

    With this notation fixed, we can then evaluate reward at time $2$ in environment $i$. In particular, we have
    \begin{multline}
        \label{eq:fsc_mepomdp_r2}
        \mathbb{E}_{\calM_i}^\pi[R_i(s_2,a_2)] 
        = 
        \sum_{j \in [n], l \in [m]} \sum_{k\in A} R_i((j,l,2),k) \mathbb{P}^\pi_{\calM_i}[s_2 = (j,l,2), a_2 = k]
        \\ = -\sum_{l \in [m]} \sum_{k \in A \setminus \{l\}}  \mathbb{P}^\pi_{\calM_i}[s_2 \in \{(j,l,2) | j \in [n]\}, a_2 = k]
        \\ = -\sum_{l \in [m]} \sum_{k \in A \setminus \{l\}}  \mathbb{P}^\pi_{\calM_i}[z_1 = l, a_2 = k]
        = -\sum_{l \in [m]} \sum_{k \in A \setminus \{l\}}  \mathbb{P}^\pi_{\calM_i}[a_2 = k| z_1 = l] \mathbb{P}_{\calM_i}^\pi[z_1 = l]
    \end{multline}
    The first equality just involves expanding the expectation.
    The second equality follows by definition of the reward.
    The third equality applies the definition of the POMDP, and in particular, it applies the fact that the state is in $\{(j,l,2) | j \in [n]\}$ at time $2$ if and only if the first observation was $l$.
    The fourth equality just applies conditional probability rules.

    We next argue that $\mathbb{P}^\pi_{\calM_i}[a_2 = l| z_1 = l] = 1$ in all environments for all $l \in [m]$.
    We first note that
    the reward at time $2$ is the same across environments, since the expression derived in \eqref{eq:fsc_mepomdp_r2} depends only on the probabilities of states and actions along the path.
    However, by the assumption encoded in \eqref{eq:sign_flip_assump}, for any policy $\pi$, there exists an environment such that the reward at $t = 3$ is non-positive.
    Thus, for the policy $\pi$ to have non-negative value, it must hold that $\mathbb{E}^\pi[R_i(s_2,a_2)] \geq 0$, where we now drop the $\calM_i$ subscript to emphasize that values are constant across all environments for any policy $\pi$.
    We also note that for every $l \in [m]$, 
    $\mathbb{P}[z_1 = l] = \frac{1}{m}$.
    Thus we can deduce from \eqref{eq:fsc_mepomdp_r2} that
    \begin{equation}
        \sum_{k \in A \setminus \{l\}}  \mathbb{P}^\pi[a_2 = k| z_1 = l] = 0 \quad \forall l \in [m].
    \end{equation}

    The requirement that $a_2 = z_1$ then allows us to argue about the memory structure. Indeed, we have, for all $k,l \in [m]$,
    \begin{equation}
        \mathbb{P}[a_2 = k| z_1 = l] = 
        \sum_{j \in [m]} 
        \vecat{F}{k,j}
        \vecat{P}{j,l} = \begin{cases}
            1 & \ k = l \\
            0 & \text{ otherwise}
        \end{cases},
    \end{equation}
    which in turn implies that 
    \begin{equation}
        F_{\mathsf{trunc}} P = I_{m\times m}.
    \end{equation}
    However, since $F_{\mathsf{trunc}}$ and $P$ are non-negative, it follows from \cite{ding2014matrix} that $P$ is a permutation matrix.
    That is, the memory state at time $2$ must be equal to the observation at time $1$, up to some relabeling.

    Since the agent observes only $\bot$ at times $2$ and $3$, the action $a_3$ at time $3$ is independent of $a_0$, conditioned on the memory state $q_2$ at time $2$. Thus, we have
    \begin{equation}
        \mathbb{P}^\pi[a_3 = k, a_0 = j] = \sum_{l \in [m]}\mathbb{P}^\pi[a_3 = k| q_2 = l] \mathbb{P}^\pi[q_2 = l, a_0 = j], 
    \end{equation}
    where we continue to drop the $\calM_i$ from the subscript of the probability measure to emphasize when properties are independent of the environment.
    However, we argued above that $q_2$ must be a permutation of the observation at time $1$, which is independent of the action $a_0$ and has a uniform distribution.
    As such, we can rewrite the joint distribution of $a_3$ and $a_0$ as
    \begin{equation}
        \mathbb{P}^\pi[a_3 = k, a_0 = j] = \underbrace{\sum_{l \in [m]}\mathbb{P}^\pi[a_3 = k| q_2 = l] \frac{1}{m}}_{=: \vecat{g}{k}} \mathbb{P}^\pi[a_0 = j].
    \end{equation}
    In a similar manner to previous proofs, we can then define two vectors $\mathbf{x},\mathbf{y} \in \dist{[n]}$
    \begin{equation}
        \vecat{\mathbf{x}}{j} = \begin{cases}
            \mathbb{P}^\pi[a_0 = 1] + \sum_{j' \in A \setminus [n]} \mathbb{P}[a_0 = j'] & \text{ if } j = 1, \\
            \mathbb{P}^\pi[a_0 = j] & \text{ otherwise.}
        \end{cases}
    \end{equation}
    \begin{equation}
        \vecat{\mathbf{y}}{k} = \begin{cases}
            \vecat{g}{1} + \sum_{k' \in A \setminus [n]} \vecat{g}{k'} & \text{ if } k = 1, \\
            \vecat{g}{k} & \text{ otherwise.}
        \end{cases}
    \end{equation}
    
    By definition of the ME-POMDP, the reward in environment $i$ when $t=3$ is equal to 
    \begin{equation}
        \sum_{j\in [n], k \in [n]} a_{ijk} \vecat{\mathbf{x}}{j} \vecat{\mathbf{y}}{k}. 
    \end{equation}
    As we established earlier, the reward at $t=2$ is zero. As such, for the ME-POMDP to have non-negative reward, the reward at $t=3$ must be non-negative in all environments. However, we can then deduce the existence of the requisite $\mathbf{x}$ and $\mathbf{y}$. 

    ($\impliedby$)
    Suppose there exists $\mathbf{x}$ and $\mathbf{y}$ such that
    \begin{equation}
        \sum_{j\in [n], k \in [n]} a_{ijk} \vecat{\mathbf{x}}{j} \vecat{\mathbf{y}}{k} \geq 0 \quad \forall i \in [m]. 
    \end{equation}
    Then, we can define a FSC-based policy $\pi$ where
    \begin{itemize}
        \item The action at time $0$ is drawn from $[n]$ according to $\mathbf{x}$ for any memory state.
        \item At time $1$, the memory is updated to store the observation $z_1$.
        \item At time $2$, the policy plays the action corresponding to the stored observation.
        \item The action at time $3$ is drawn from $[n]$ according to $\mathbf{y}$ for any memory state.
    \end{itemize}
    At time $2$ we get zero reward in all POMDPs as we have $a_2 = z_1$.
    Meanwhile, at time $3$ we have non-negative expected reward in all POMDPs as the joint distribution of $a_0$ and $a_3$ satisfies
    \begin{equation}
        \mathbb{P}[a_0 = j,a_3 = k] = \vecat{\mathbf{x}}{j}\vecat{\mathbf{y}}{k}. 
    \end{equation}
\end{proof}

\end{document}